\documentclass[10pt]{article} % For LaTeX2e
\usepackage{multirow}
\usepackage{tabularx}
\usepackage[accepted]{tmlr}
\usepackage{amsmath,amsfonts,bm}

\def\eqref#1{equation~\ref{#1}}
\def\1{\bm{1}}

\DeclareMathAlphabet{\mathsfit}{\encodingdefault}{\sfdefault}{m}{sl}
\SetMathAlphabet{\mathsfit}{bold}{\encodingdefault}{\sfdefault}{bx}{n}

\usepackage[utf8]{inputenc} % allow utf-8 input
\usepackage[T1]{fontenc}    % use 8-bit T1 fonts
\usepackage{hyperref}
\usepackage{booktabs}       % professional-quality tables
\usepackage{amsfonts}       % blackboard math symbols
\usepackage{nicefrac}       % compact symbols for 1/2, etc.
\usepackage{microtype}      % microtypography
\usepackage{xcolor}         % colors
\usepackage{wrapfig}

\usepackage{microtype}
\usepackage{graphicx}
\usepackage{subfigure}
\usepackage{booktabs} 

\usepackage{amsmath}
\usepackage{amssymb}
\usepackage{mathtools}
\usepackage{amsthm}
\usepackage{mathrsfs}
\usepackage{graphicx}

\usepackage[normalem]{ulem}
\useunder{\uline}{\ul}{}
\usepackage{colortbl}
\usepackage{capt-of}
\usepackage[]{graphicx}
\usepackage{pifont}% http://ctan.org/pkg/pifont
\newcommand{\cmark}{\ding{51}}%
\newcommand{\xmark}{\ding{55}}%
\usepackage[capitalize,noabbrev]{cleveref}
\usepackage{gensymb}

\newcommand{\iclr}[1]{{#1}}

\definecolor{customblue}{rgb}{0.1, 0.3, 0.8}
\newcommand{\tmlr}[1]{{#1}}
\newcommand{\method}{\texttt{PhASER}}

\usepackage[utf8]{inputenc} % allow utf-8 input
\usepackage[T1]{fontenc}    % use 8-bit T1 fonts
\usepackage{url}            % simple URL typesetting
\usepackage{booktabs}       % professional-quality tables
\usepackage{amsfonts}       % blackboard math symbols
\usepackage{nicefrac}       % compact symbols for 1/2, etc.
\usepackage{microtype}      % microtypography
\usepackage{xcolor}         % colors
\definecolor{blue}{rgb}{0,0,1}
\definecolor{darkgreen}{rgb}{0,0.40,0}
\definecolor{firebrick}{rgb}{0.698,0.133,0.133}

\hypersetup{
    colorlinks=true,
    linkcolor=firebrick,
    citecolor=darkgreen,
    urlcolor=firebrick,
    linktocpage=true,
    unicode=true,
    pdfborder={0 0 0} % Remove boxes
}
\theoremstyle{plain}
\newtheorem{theorem}{Theorem}[section]
\newtheorem{proposition}[theorem]{Proposition}
\newtheorem{lemma}[theorem]{Lemma}

\theoremstyle{definition}
\newtheorem{definition}[theorem]{Definition}

\theoremstyle{remark}

\usepackage{hyperref}
\usepackage{url}

\title{Phase-driven Generalizable Representation Learning for\\ Nonstationary Time Series Classification}
\author{\name Payal Mohapatra\thanks{These authors contributed equally.} \email payal.mohapatra@northwestern.edu \\
  \addr Northwestern University, Evanston, Illinois, USA
  \AND
  \name Lixu Wang\footnotemark[1]  \email lixu.wang@northwestern.edu \\
  \addr Northwestern University, Evanston, Illinois, USA
  \AND
  \name Qi Zhu \email qi.zhu@northwestern.edu \\
  \addr Northwestern University, Evanston, Illinois, USA
}

\def\month{10}  % Insert correct month for camera-ready version
\def\year{2025} % Insert correct year for camera-ready version
\def\openreview{\url{https://openreview.net/forum?id=cb3nwoqLdd}} % Insert correct link to OpenReview for camera-ready version

\begin{document}

\maketitle

\begin{abstract}

Pattern recognition is a fundamental task in continuous sensing applications, but real-world scenarios often experience distribution shifts that necessitate learning generalizable representations for such tasks. This challenge is exacerbated with time-series data, which also exhibit inherent \emph{nonstationarity}—variations in statistical and spectral properties over time. In this work, we offer a fresh perspective on learning generalizable representations for time-series classification by considering the phase information of a signal as an approximate proxy for nonstationarity and propose a phase-driven generalizable representation learning framework for time-series classification, \method{}. It consists of three key elements: 1) \emph{Hilbert transform-based augmentation}, which diversifies nonstationarity while preserving task-specific discriminatory semantics, 2) \emph{separate magnitude-phase encoding}, viewing time-varying magnitude and phase as independent modalities, and 3) \emph{phase-residual feature broadcasting}, integrating 2D phase features with a residual connection to the 1D signal representation, providing inherent regularization to improve distribution-invariant learning. Extensive evaluations on five datasets from sleep-stage classification, human activity recognition, and gesture recognition against 13 state-of-the-art baseline methods demonstrate that \method{} consistently outperforms the best baselines by an average of 5\% and up to 11\% in some cases. Additionally, the principles of \method{} can be broadly applied to enhance the generalizability of existing time-series representation learning models.

\end{abstract}

\section{Introduction}

% (1)What is your problem? Why is it important?

Time-series data play a ubiquitous and crucial role in numerous real-world applications, such as continuous monitoring for human activity recognition~\citep{li2020survey}, gesture identification~\citep{ozdemir2020emg}, sleep tracking~\citep{kemp2000analysis}, fatigue monitoring~\citep{mohapatra2024wearable} and more. Continuous time series often exhibit \textit{non-stationarity}, i.e., the statistical and spectral properties of the data evolve over time. \tmlr{Another practical challenge is domain shift, where the data-generating process systematically evolves over time due to factors such as changes in sensor type, sub-population shifts, or environmental variations, thereby degrading model performance on previously unseen distributions.} Thus, developing methods for more generalizable pattern recognition in nonstationary time series classification is crucial.

% Another practical challenge is the distribution shift due to the underlying sensing properties or subject-specific attributes, commonly referred to as \textit{domain shift}, which directly degrades the performance of time-series models in real-world applications.

% Another inherent challenge is the distribution shift due to the underlying sensing properties and monitoring object preference,
% commonly referred to as \textit{domain shift}.
% One of the most representative examples of such non-stationarity is the distribution shift due to the underlying sensing properties, commonly referred to as \textit{domain shift}. 
% For instance, domain shift in HAR~\citep{qin2022domain} may appear when a wearer switches devices or when a new wearer adopts the device. 
%Such domain shifts directly impact the performance of time-series models, 

% (2) State briefly what the other solutions are to the problem, and why they aren't satisfactory.

% \lixu{Too many citations in this paragraph, can we just pick some representative ones?}
Most existing methods~\citep{ragab2023adatime, ragab2023source, he2023domain} tackle distribution shifts in time-series applications via domain adaptation, assuming accessible target domain samples. Yet, obtaining data from unseen distributions in advance is not always feasible.
%but are not suitable for many practical applications.
% There have been studies~\citep{ragab2023adatime, ragab2023source, he2023domain} addressing distribution shifts in time series applications with a more relaxed assumption that the target domain data is accessible and it becomes possible to conduct domain adaptation. 
% However, in the real world, it is often impractical to access data from unseen distributions in advance. 
To overcome this challenge, a few works~\citep{gagnon2022woods, xu2022globem} applied standard domain generalization (DG) algorithms~\citep{volpi2018generalizing, sagawa2019distributionally, parascandolo2020learning} to temporally-varying time-series data, but reported a significant performance gap when compared with visual data. Recent research on DG tailored for time series explores latent-domain characterization~\citep{lu2023outofdistribution,du2021adarnn}, augmentation strategies~\citep{iwana2021empirical, li2021simple}, preservation of non-stationarity dictionary~\citep{non_transformer_1, kim2021reversible}, and utilization of spectral characteristics of time series~\citep{he2023domain, yang2022unsupervised, kim2021broadcasted}. While successful in some cases, these methods have their limitations. Latent-domain characterization heavily relies on the hypotheses of latent domains, limiting its broader applicability. Augmentation strategies (shift, jittering, masking, etc.) for time series may not be universally applicable and can impair the task~\citep{iwana2021empirical}. For instance, in physiological signal analysis, morphological alterations from augmentations are harmful, and time-slicing is unsuitable for periodic signals. Advanced augmentation techniques like spectral perturbations (time-frequency warping, decomposition techniques, etc.) are usually heavily parametric~\citep{time_augmentation} and application-specific. Other approaches specific to preserving non-stationarity are constrained by maintaining the same input-output space, making them unsuitable for multivariate time-series classification tasks. While some works~\citep{he2023domain, yang2022unsupervised} focus on frequency domain representations for robustness to feature shifts, they overlook cases with time-varying spectral responses. Another significant issue is that many of these studies rely on domain identity, which in practice is expensive and intrusive to obtain, especially in healthcare and finance~\citep{yan2024prompt, bai2022temporal}. \tmlr{Thus, achieving generalizable time-series classification without access to unseen distributions and the labels of available distributions remains a challenging yet crucial pursuit.}

\smallskip \textbf{Our Approach and Contributions.}
We propose a novel \uline{Ph}ase-\uline{A}ugmented \uline{S}eparate \uline{E}ncoding and \uline{R}esidual (\method{}) framework to achieve generalizable representations for classification of \emph{nonstationary} real-world time series. Figure~\ref{fig:overall_system_diag} illustrates an overview of \method{}, which includes three key modules. First, we diversify the source domain data through an intra-instance phase shift by leveraging the generality and non-parametric nature of the Hilbert Transform (HT)~\citep{king2009hilbert} to handle nonstationary signals and introduce phase-shift-based augmentation. Next, we encode the time-varying magnitude and phase responses separately for enhanced integration of the time-frequency information. Finally, we design an effective broadcasting mechanism with a non-linear residual connection between the phase-encoded embedding and the backbone representation to learn generalizable~\citep{regularize_resnet,marion2023implicit} task-specific features~\citep{he2016deep}. We experiment with 13 baselines on 5 datasets to quantitatively demonstrate \method{}'s superiority in learning generalizable representations, even in challenging scenarios such as transferring from one domain to multiple domains. Additionally, we provide detailed design insights through ablation analysis, explore \method{}'s applicability to other architectures, examine other augmentation schemes with \method{}, and present qualitative visualizations of its learned representations.

% It allows \method{} to benefit from the inherent regularization~\citep{regularize_resnet,marion2023implicit} to learn domain-invariant features while offsetting any degradation to the desirable underlying representation~\citep{he2016deep}. 

% \textcolor{blue}{lixu: need a sentence to mention the experiments}

% An overview of \method{} is illustrated in Figure~\ref{fig:overall_system_diag}. 

% The details of \method{} are presented in Section~\ref{sec:approach}, where we also theoretically substantiate our design philosophy by demonstrating the importance of addressing non-stationarity in optimizing classification risks of unseen distributions and the role of HT in changing nonstationary statistics of time series. In Section~\ref{sec:experiments}, we conduct extensive experiments on 5 benchmark datasets across 3 popular application scenarios and demonstrate the superior performance of \method{} over a comprehensive set of state-of-the-art approaches. We also validate the effectiveness of each component in \method{} and its general applicability through systematic analysis. Section~\ref{sec:related_works} discusses additional related works beyond those mentioned in the introduction. Section~\ref{sec:conclusion} concludes the paper.

\begin{figure*}[t]
\centering
% https://app.diagrams.net/#G1Yw_pX5EzhDu-MAPOuXmTw9yUP3ehms0a#%7B%22pageId%22%3A%22cDESBTEy9yq7dKrJ2j5t%22%7D
% \includegraphics[width=\linewidth]{figures/NIPsMainTSDGRevamp.png}
\includegraphics[width=\linewidth]{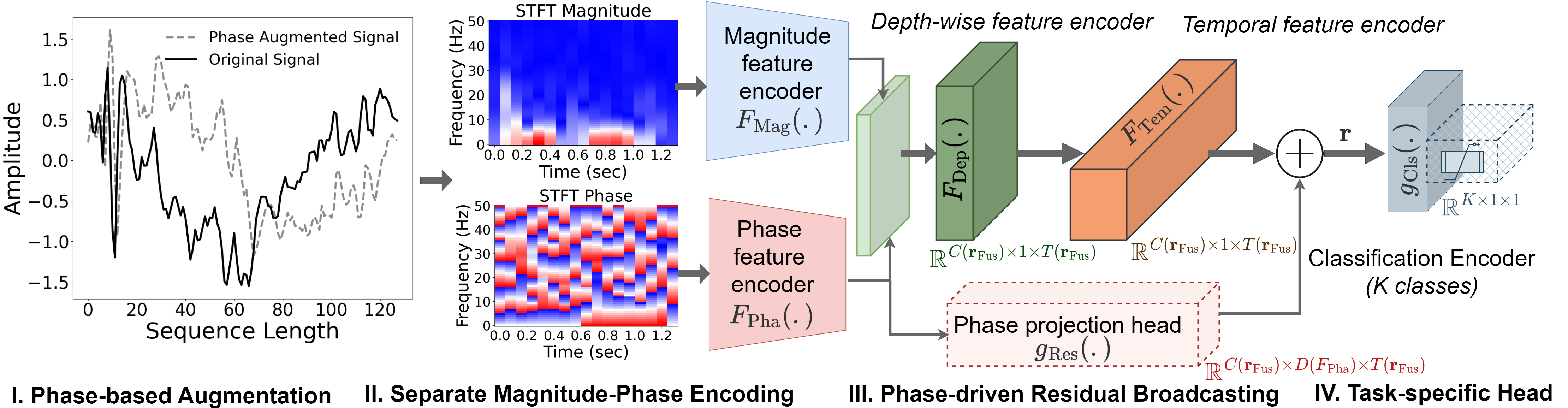}
% \includegraphics[width=0.95\linewidth]{figures/MainTSDG_cropped.png}
% \vspace{-25pt}
\caption{\small \method{}'s components: I. Hilbert transform-based phase augmentation. II. Separate feature encoding of time-varying phase and magnitude derived from Short-Term Fourier Transform (STFT) using $F_\mathrm{Mag}$ and $F_\mathrm{Pha}$. III. Key elements of the phase-residual broadcasting network, demonstrating design of depth-wise feature encoder ($F_\mathrm{Dep}$), temporal encoder ($F_\mathrm{Tem}$), and incorporation of phase-projection head's output ($g_\mathrm{Res}$) for broadcasting (annotated dimensions of intermediate feature maps). IV. Task-specific classification encoder ($g_\mathrm{Cls}$).}
% This encoding and fusion approach outperforms mere concatenation of magnitude and phase features. 
% \caption{Overview of PhASER's components. I. illustrates Hilbert transform-based phase augmentation using phasor representation (top left) for a time-series signal's negative and positive frequency components. The augmentation translates a signal $\mathbf{x}(t)$, to its $\pi/2$ phase-shifted version, $\widehat{\mathbf{x}}(t)$. II. demonstrates separate feature encoding of time-varying phase and magnitude derived from Short-Term Fourier Transform (STFT) using the magnitude encoder $F_\mathrm{Mag}$ and the phase encoder $F_\mathrm{Pha}$ with sub-feature normalization. We find that such a separate encoding followed by fusion provides superior performance over mere concatenation of magnitude and phase features for further processing. III. shows the key elements of the phase-residual broadcasting network. The dimensions of the intermediate feature maps are annotated to demonstrate the design of the depth-wise feature encoder ($F_\mathrm{Dep}$) followed by the temporal encoder ($F_\mathrm{Tem}$) and incorporation of the phase-projection head's ($g_\mathrm{Res}$) output for broadcasting. IV. represents task-specific classification encoder ($g_\mathrm{Cls}$) to optimize a categorical objective. (Best viewed in color.)}
\label{fig:overall_system_diag}
\vspace{-15pt}
\end{figure*}

\section{Approach}\label{sec:approach}

% Thoughts :: Although we cannot capture continuous time samples and ultimately use discrete series, its ok to refer to our time series as a continuou funciton here and in HT as well since we mention its discretization in STFT anyway.
\subsection{Problem Formulation}
\begin{definition}[\textbf{Nonstationary Time Series}] 
\label{definition_nonstationary}
% Following the definition of mixed decomposition in~\citept{timeseriesdecomposition}, we assume that a continuous time-series sample drawn from a nonstationary domain $\mathbf{x} = \{x_0, ..., x_t, ...\} \sim \mathcal{D}_\mathbf{x}$ can be decomposed into components with mean $\mu_t$ and variance $\sigma_t$ (both $\mu_t$ and $\sigma_t$ are not always zero) as:
% \todo{Here we want to say that the time-series is nonstationary and not that the domain is nonstationary. Lixu please verify for consistency.}
Following the definition of mixed decomposition-based nonstationary signals in~\citet{timeseriesdecomposition}, we assume that a nonstationary time-series sample $\mathbf{x} = \{x_0, ..., x_t, ...\}$ drawn from a domain $\mathcal{D}_\mathbf{x}$ can be decomposed into components with mean $\mu_t$ and variance $\sigma_t$ (both $\mu_t$ and $\sigma_t$ are not always zero) as:
%\vspace{-4pt}
\begin{equation}
\begin{aligned}
\mathrm{Pr}_{\mathbf{x} \sim \mathcal{D}_\mathbf{x}}(\mathbf{x})(t) = \mu_t + \sigma_t \times z,\,\text{where}\,\forall L \geq 1, \exists t, \left[\mu_t \neq \mu_{t+L}\right] \vee \left[\sigma_t \neq \sigma_{t+L}\right],
\end{aligned}
%\vspace{-4pt}
\end{equation}
% \vspace{-5pt}
where $z$ is a stationary stochastic component with a zero mean and a unit variance. 
\end{definition}

\smallskip
\begin{definition}[\textbf{\tmlr{Time-Series Generalization to Unseen Domains}}]
Suppose there is a dataset $\mathbf{S} \!=\! \{(\mathbf{x}_i, y_i)\}_{i=1}^{M}$ with $M$ nonstationary time-series samples drawn from a set of $N_S$ source domains $S \!=\! \{\mathcal{S}_i\}_{i=1}^{N_S}$. The joint distribution of $\mathbf{S}$ is $\mathrm{Pr}(\mathcal{X}_\mathbf{S}, \mathcal{Y}_\mathbf{S})$, i.e., $\mathbf{x}_i \!\sim\! \mathcal{X}_\mathbf{S}, y_i \!\sim\! \mathcal{Y}_\mathbf{S}$ and $\mathbf{x}_i \in \mathbb{R}^{V \times T}$, where $V$ is the number of time-series feature dimensions and $T$ is the sequence length. $y_i \in \mathbb{R}^{1 \times 1}$ is the categorical label. \iclr{Note that the joint distributions of different source domains are similar (with shared underlying patterns) but domain-specific distinctions:}
\begin{equation}\label{eqn:dg}
    \mathrm{Pr}(\mathcal{X}_{\mathcal{S}_i}, \mathcal{Y}_{\mathcal{S}_i}) \neq \mathrm{Pr}(\mathcal{X}_{\mathcal{S}_j}, \mathcal{Y}_{\mathcal{S}_j}), 1 < i \neq j \leq N_S.
    % \vspace{-15pt}
\end{equation}

For any potential unseen target domain $\mathcal{D}_\mathrm{U}$, its joint distribution remains distinct like Eq.~(\ref{eqn:dg}). In our problem, although the source dataset is assumed to contain multiple domains, the annotations that specify the domain identity are unavailable. Our goal is to train a model consisting of a feature extractor $F$ and a classifier $g$ using the given source dataset ($F \circ g:\mathcal{X}_\mathbf{S} \xrightarrow{} \mathcal{Y}_\mathbf{S}$), such that
% \vspace{-5pt}
\begin{equation}
    \min \mathop{{}\mathbb{E}}_{(\mathbf{x}, y) \sim  \mathcal{D}_\mathrm{U}} [\mathcal{L}(g(F(\mathbf{x})), y)],
%    \vspace{-3pt}
\end{equation}
where $\mathcal{L}(\cdot)$ is a certain cost that measures the errors between model predictions and the ground truth.
%\vspace{-1em}
\end{definition}

\begin{figure}[!htbp]
% \vspace{-15pt}
\centering
\includegraphics[width=0.75\linewidth]{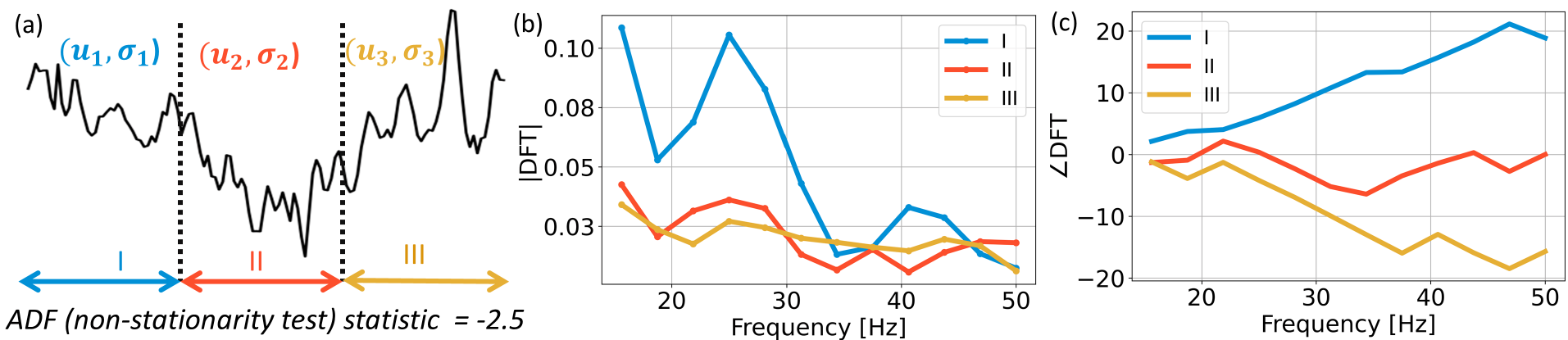}
\vspace{-12pt}
\caption{\small (\textbf{Real-world time-series are nonstationary.}) Example of non-stationarity using a sample from a human activity recognition dataset (HHAR) where (a) shows the temporal non-stationarity denoted by varying mean $\mu$ and variance $\sigma$ within a domain for three regions color-coded and denoted as I, II, and III. (b) shows that the magnitude response ($|\text{DFT}|$) of the Discrete Fourier Transform (DFT) for each region is distinct. There is a clear difference in the dominant frequency for each region. (c) shows the phase responses ($\angle(\text{DFT})$) for each region. The $\angle(\text{DFT})$ of each region is also distinct.}
% d) illustrates the overall time domain response of the original signal and the phase-shifted version. We achieve non-stationarity diversification by shifting the phase response without altering the magnitude response of the signal thus preserving task-relevant semantics.
% Path source :: /home/payal/TSDG_2023/toy_example_hhar.ipynb
\label{fig:motivation}
% \vspace{-12pt}
\end{figure}
% \vspace*{-\baselineskip}

% \begin{figure*}[t]
%   \centering
%   \begin{minipage}{0.4\textwidth}
%     \centering
%     \includegraphics[width=\textwidth]{figures/NS_motivation_nips.png} 
%     \vspace{-12pt}
%     \caption{Performance comparison between \method{} (Ours) and BCResNet with increasingly nonstationary HHAR dataset.}
%     % \caption{Illustration of the efficacy of \method{} with increasing non-stationarity constructed on HHAR and the comparison with BCResNet.}
%     \label{fig:motivation_2}
%   \end{minipage}\quad
%   \hspace{0.05\textwidth}
%   \begin{minipage}{0.4\textwidth}
%     \centering
%     \small
%     \vspace{-2pt}
%     \resizebox{1.\textwidth}{!}{
%     \setlength{\tabcolsep}{.8mm}{
%     \begin{tabular}{lc}
%     \toprule
%     Input Modality    & Accuracy \\ \midrule
%     Only Magnitude (Mag)          & 0.81 $\pm$ 0.03   \\
%     Only Phase (Pha)          & 0.62 $\pm$ 0.03    \\
%     Mag-Pha Concatenate  & 0.73 $\pm$ 0.03 \\
%     Mag-Pha Separate  & 0.85 $\pm$ 0.01 \\ \bottomrule
%     \end{tabular}}}
%     \vspace{3pt}
%     \captionof{table}{Performance comparison of various input configurations for time-frequency information.}
%     \label{tab:table_pilot}
%     \end{minipage}
%   \vspace{-10pt}
% \end{figure*}

\smallskip \textbf{Motivation.} Consider a human activity recognition (HAR) application, where non-stationarity is unavoidable due to changes in sensor characteristics~\citep{bangaru2020data}. We illustrate an instance of non-stationarity in Figure~\ref{fig:motivation}(a), which visualizes a univariate accelerometer data sample from a dataset called HHAR~\citep{stisen2015smart} in the time domain. By segmenting this sample into sequential windows and conducting a Discrete Fourier Transform (DFT) to obtain its magnitude and phase responses, as shown in Figures~\ref{fig:motivation}(b) and (c), we observe shifts in the spectral domain. The Augmented Dickey-Fuller (ADF) statistic~\citep{said1984testing} also supports the signal's non-stationarity. Thus, there is non-stationarity in terms of signal statistics and spectral properties. Most real-world time-series datasets for activity recognition, sleep stage classification, and gesture recognition applications are nonstationary, as indicated by their ADF statistics reported in Table~\ref{tab:app_dataset} in the Appendix. \tmlr{Furthermore, more structured nonstationarities (in time domain statistics, spectral properties etc.) arise from sensor and subpopulation shifts, which manifest as distribution shifts and highlight the need to learn generalizable representations for unseen distributions.}

\begin{wrapfigure}{r}{0.3\textwidth}
    \begin{center}
        % \vspace{-8pt}
        \includegraphics[width=0.26\textwidth]{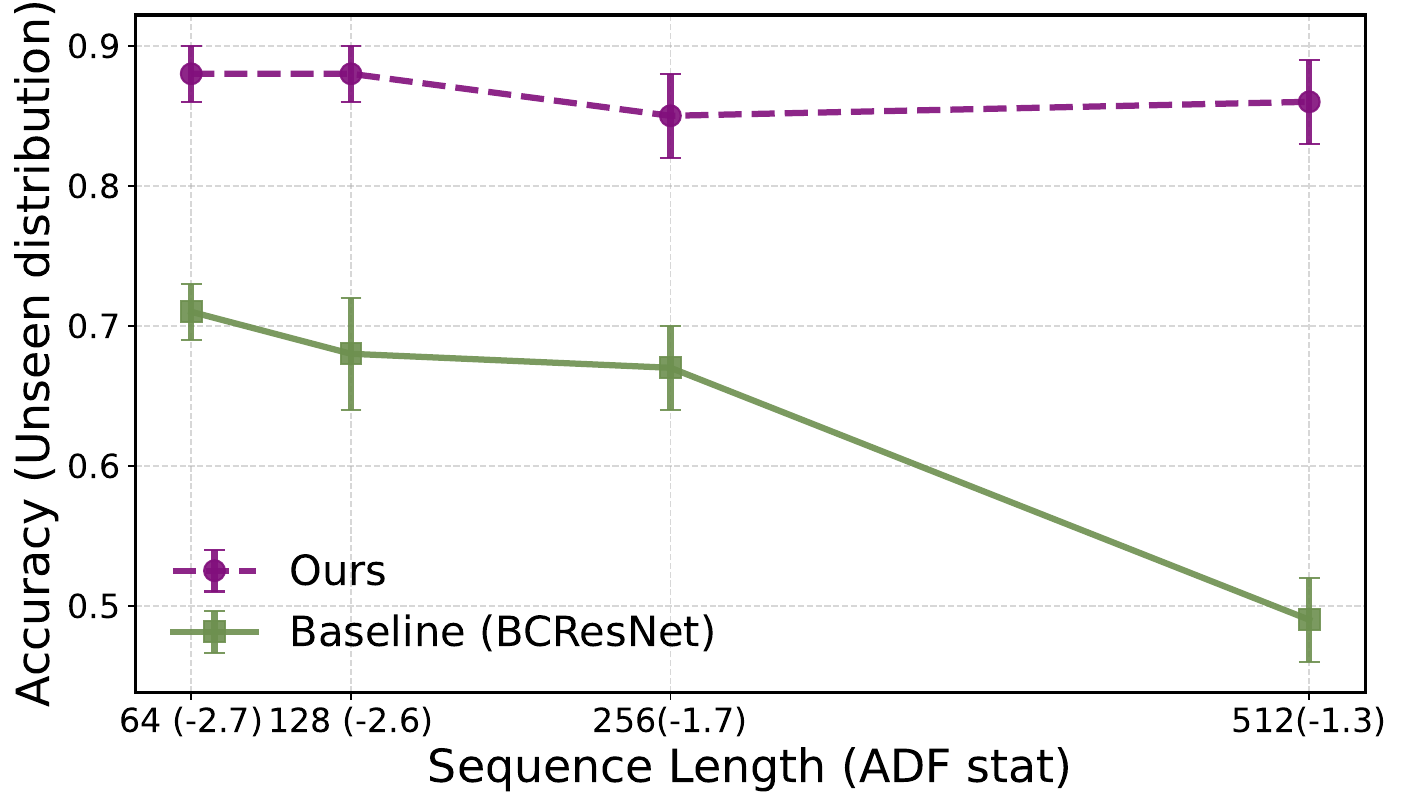}
        \vspace{-13pt}
        \caption{\small (\textbf{Nonstationarity impacts generalization}) Comparison between \method{} and BCResNet with increasing non-stationarity in HHAR dataset.}
        \vspace{-15pt}
        \label{fig:motivation_2}      
    \end{center}
\end{wrapfigure}
Now we pose the question: \textit{What is the impact of non-stationarity of time series on a model's generalization ability \tmlr{to unseen distributions}?} We conduct a simple empirical study on the HHAR dataset by varying the sequence length to synthesize increasing non-stationarity, measured by the ADF statistic (a higher ADF value indicates greater non-stationarity). More details of the ADF test are provided in Section~\ref{app:nw_design} of the Appendix. We adopt a DG model BCResNet from~\citet{kim2021broadcasted} for time-series classification to explore the relationship between the degree of non-stationarity and the model's generalization ability to unseen distributions. \tmlr{Figure~\ref{fig:motivation_2} shows an evident drop in the accuracy of BCResNet as non-stationarity increases, highlighting that non-stationarity has a detrimental impact on generalization performance, which is also observed in prior work~\citep{10.1109/TPAMI.2024.3355212}, possibly due to model overfitting to the source domains' non-task-specific nonstationary attributes.}

\tmlr{We are motivated to leverage the phase information of a signal as a proxy for its non-stationarity, as the phase response captures underlying local time shifts and time-localized frequency variations characteristic of nonstationary signals~\citep{klein2001non, oppenheim1999discrete}. In this work, we propose \method{}, whose components are anchored in phase to achieve generalizable performance for nonstationary time series classification tasks, as demonstrated in Figure~\ref{fig:motivation_2}.}

\textbf{Overview of \method{}.} Our proposed \method{} framework, shown in Figure~\ref{fig:overall_system_diag}, begins with an augmentation module that utilizes the Hilbert Transform to generate out-of-phase augmentations for time series, which diversify non-stationarity while preserving the category-discriminatory semantics for classification tasks. We use the short-term Fourier Transform (STFT) to obtain temporal magnitude and phase responses. Two separate encoders process the magnitude and phase as distinct input modalities. Next, \method{} establishes a phase-feature broadcasting mechanism as a residual connection to emphasize learning task-relevant information. \tmlr{Finally, the classifier extracts robust task-discriminatory representations.} In the following Sections~\ref{subsec:ht} to~\ref{subsec:broadcasting}, we introduce the details of these three components of \method{} and discuss the theoretical insights that inspire our design.

% By fully leveraging the phase-related information, the \method{} framework implicitly regularizes the representations against non-stationarity and offsets any degradation to the desirable features. 
% non-stationarity 

% can be significantly characterized in the final representation

% \lixu{need to check again}, subsequently allowing the learning of a generalizable representation by the classifier in a straightforward way.

\subsection{Hilbert Transform based Phase Augmentation}
\label{subsec:ht}

\begin{wrapfigure}{r}{0.6\textwidth}
    \vspace{-20pt}
    \begin{center}
        \includegraphics[width=0.6\textwidth]{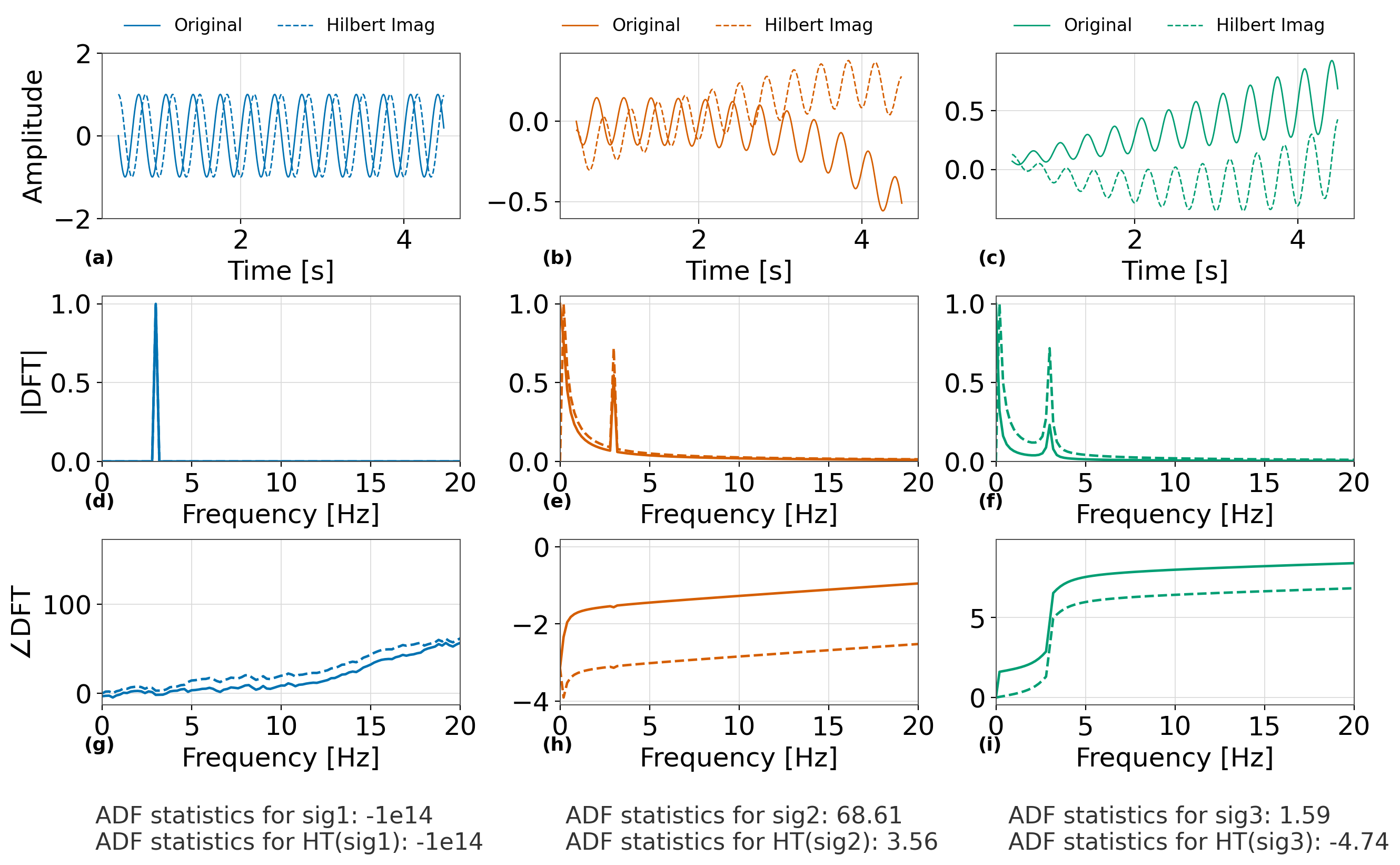}
        \vspace{-15pt}
        \caption{\small 
    \tmlr{\textbf{(Phase as a proxy for non-stationarity.)} 
    (a–c) Time-domain signals: (a) Stationary sinusoid; (b, c) Non-stationary sinusoids with the same base frequency. 
    (d–f) Magnitude spectra (DFT) of the corresponding signals, showing similar frequency content. 
    (g–i) Phase spectra, displaying distinct responses that reflect differences in underlying dynamics. 
    ADF statistics below each column summarize the stationarity of the respective signals and show how it changes with the proposed Hilbert Transform (HT)-based augmentation.}
}
        \label{fig:motivation_ns_phase}
    \end{center}
    \vspace{-15pt}
\end{wrapfigure}Our motivating study in Figure~\ref{fig:motivation_2}, demonstrates the adverse effects of increasing non-task-specific non-stationarity on a model's generalization ability, and Figure~\ref{fig:motivation_ns_phase} shows that the phase response of a signal encodes its non-stationarity. Inspired by these observations, we propose to employ a data augmentation technique that diversifies the non-stationarity in the training data while preserving the original data's discriminatory properties to ensure semantic differentiability.

% A promising approach to address this challenge, without explicitly characterizing the non-stationarity, is to employ data augmentation techniques that diversify the non-stationarity in the training data to. These augmentations must preserve the original data's discriminatory properties to ensure semantic differentiability.

Different from most existing time-series augmentation techniques, we introduce a phase shift to a signal while preserving the magnitude response, thereby offering an augmented view. This intra-sample phase-augmentation technique is less studied in the context of time-series classification for domain generalization (although some recent works, like~\citet{demirel2024finding}, explore phase-mixup for contrastive learning). We intuitively justify our design choice by exploring a question: \textit{Does shifting the phase response of a time series change its non-stationarity?} \tmlr{In Figure~\ref{fig:motivation_ns_phase}(a--c), we illustrate a stationary sinusoid and two non-stationary sinusoids all sharing the same frequency, as evidenced by their magnitude responses shown in Figure~\ref{fig:motivation_ns_phase}(d--f). However, the distinct phase responses of each signal reveal changes in the underlying dynamics. These phase variations occur as time-local oscillometric fluctuations arise, motivating the use of phase information as a proxy for capturing the underlying non-stationarity~\citep{Wu2009}.} 
% Figure~\ref{fig:overall_system_diag}.I shows the effect of precisely shifting the phase of a non-stationary signal from the HHAR dataset without changing its magnitude response in the time domain.

We propose a simple but effective data augmentation technique based on the Hilbert Transform (HT) to diversify the non-stationarity and preserve discriminatory features. Specifically, for each time-series sample $\mathbf{x}$ in the source dataset $\mathbf{S}$, we can assume it is a real-valued signal $\mathbf{x}=\{x_0, ..., x_t, ...\} \in \mathbb{R}$ that is characterized by a deterministic function $x_t = \mathbf{x}(t)$. 
Then, $\mathrm{HT}(\mathbf{x}(t)) = \widehat{\mathbf{x}}(t) = \int_{-\infty}^{\infty}\mathrm{x}(\tau)\frac{1}{\pi(t-\tau)}d\tau$. 
HT can be easily interpreted in the frequency domain via Fourier analysis: 
% \vspace{-10pt}
$$
\begin{aligned}
    &f_\mathbf{x}(\xi) = \mathcal{F}\{\mathbf{x}(t)\} = \int^{\infty}_{-\infty} \mathbf{x}(t)e^{i 2 \pi \xi t} dt, -\infty < \xi < \infty, \\
    &\mathbf{x}(t) = \mathcal{F}^{-1} \{f_\mathbf{x}(\xi)\} = \int^{\infty}_{-\infty} f_\mathbf{x}(\xi) e^{i 2 \pi \xi t} d\xi, -\infty < t < \infty,
\end{aligned}
% \vspace{-10pt}
$$
where $\mathcal{F}, \mathcal{F}^{-1}$ denote the Fourier transform and inverse, and $\xi$ is the frequency variable. To interpret $\widehat{\mathbf{x}}$ in the frequency domain, the negative frequency spectrum of $f_\mathbf{x}(\xi)$ is multiplied with the imaginary unit $i$, while the positive spectrum is multipled with $-i$. Then we have:
%\vspace{-10pt}
\begin{equation}
    \mathrm{HT}(\mathbf{x}(t)) = \widehat{\mathbf{x}}(t) = \mathcal{F}^{-1}\{-i \cdot \mathrm{sgn}(\xi)f_\mathbf{x}(\xi)\},
\label{eq:ht}
% \vspace{-5pt}
\end{equation} 
where $\mathrm{sgn}(\cdot)$ is a sign function. Applying HT on a signal results in a phase shift of $-\pi/2$, yielding a new out-of-phase signal. After obtaining the transformed $\widehat{\mathbf{x}}$ for across all feature dimensions, we merge the augmented dataset $\widehat{\mathbf{S}}$ and the original $\mathbf{S}$ to form a new larger dataset $\mathbf{S}^\prime = \widehat{\mathbf{S}} \cup \mathbf{S}$. For the rest of the design, there is no distinction among the samples in $\mathbf{S}^\prime$, whether they belong to $\widehat{\mathbf{S}}$ or $\mathbf{S}$.

\textbf{Theoretical Motivation: Phase-Based Hilbert Transform Diversifies Non-Stationarity.} Suppose there are $M_\mathcal{D}$ samples (observations) available for a nonstationary time-series domain $\mathcal{D}_\mathbf{x}$, and each sample $\mathbf{x}_i=\{x_{i, 0}, ..., x_{i, t}, ...\}$ is characterized by, $\mathbf{x}_i(t) = x_{i, t} = \mathrm{x}_i(t)$, $i\in[1, M_\mathcal{D}]$. If we apply Hilbert Transformation $\mathrm{HT}(\mathbf{x}(t)) = \widehat{\mathbf{x}}(t) = \int_{-\infty}^{\infty}\mathrm{x}(\tau)\frac{1}{\pi(t-\tau)}d\tau$ to augment these time-series samples, the nonstationary statistics of augmented samples are different from the original ones,
$
\mathrm{Pr}_{\mathbf{x} \sim \widehat{\mathcal{D}}_\mathbf{x}}(\mathbf{x})(t) \neq \mathrm{Pr}_{\mathbf{x} \sim \mathcal{D}_\mathbf{x}}(\mathbf{x})(t).
$

\begin{wraptable}{r}{0.52\textwidth}
\footnotesize
\vspace{-20pt}
\centering
\begin{minipage}[t]{0.22\textwidth}
\centering
\caption{\small Domain discrepancy accuracy for $\mathbf{S}$ vs $\widehat{\mathbf{S}}$.}
\label{tab:domain_discrepancy}
\begin{tabularx}{\linewidth}{@{}lc@{}}
\toprule
Dataset & Accuracy \\
\midrule
WISDM  & $0.99_{\pm 0.01}$ \\
UCIHAR & $0.99_{\pm 0.01}$ \\
HHAR   & $1.00_{\pm 0.02}$ \\
SSC    & $0.98_{\pm 0.03}$ \\
GR     & $0.97_{\pm 0.02}$ \\
\bottomrule
\end{tabularx}
\end{minipage}%
\hfill
\begin{minipage}[t]{0.28\textwidth}
\centering
\caption{\small Generalization accuracy across domains (train $\rightarrow$ test).}
\label{tab:gen_results}
\begin{tabularx}{\linewidth}{@{}lcc@{}}
\toprule
Dataset & $\mathbf{S} \rightarrow \mathbf{S}$ & $\mathbf{S} \rightarrow \widehat{\mathbf{S}}$ \\
\midrule
WISDM  & $0.88_{\pm0.02}$ & $0.88_{\pm0.01}$ \\
UCIHAR & $0.86_{\pm0.02}$ & $0.83_{\pm0.04}$ \\
HHAR   & $0.76_{\pm0.01}$ & $0.74_{\pm0.06}$ \\
SSC    & $0.76_{\pm0.03}$ & $0.77_{\pm0.01}$ \\
GR     & $0.78_{\pm0.03}$ & $0.75_{\pm0.01}$ \\
\bottomrule
\end{tabularx}
\end{minipage}
\end{wraptable} \tmlr{\textbf{Implication.}
This result establishes that phase manipulation via the Hilbert transform can meaningfully diversify the non-stationarity of time series. Empirically (see Figure~\ref{fig:motivation_ns_phase}), we observe that phase-augmented signals exhibit distinct non-stationarity statistics, as measured by ADF tests. To demonstrate that the proposed augmentation produces meaningful diversification in real-world datasets, we conduct a domain discrepancy test~\citep{saito2018maximum}, where we train a classifier to distinguish between samples drawn from $\widehat{\mathbf{S}}$ or $\mathbf{S}$. For all the datasets used in this work (more details on datasets are given in Section~\ref{sec:experiments}), we observe much higher-than-chance accuracy in Table~\ref{tab:domain_discrepancy}, supporting that $\widehat{\mathbf{S}}$ and $\mathbf{S}$ are indeed diverse.

To verify that the proposed diversification does not compromise task-specific discriminative features, we conduct a controlled experiment: a model is trained on $\mathbf{S}$ and evaluated on held-out samples from both $\mathbf{S}$ and their augmented counterparts $\widehat{\mathbf{S}}$. As shown in Table~\ref{tab:gen_results}, the model exhibits minimal performance degradation across domains, indicating that the augmentation preserves task-relevant semantics. This aligns with our design motivation to preserve the magnitude response, which, as shown in the following section (Table~\ref{tab:table_pilot}, contains more prominent task-relevant information.}

% % ### TMLR :: Update
% \begin{theorem}[\textbf{Non-stationarity Change of Hilbert Transform}]
% \label{theorem_hilbert}
% Suppose there are $M_\mathcal{D}$ samples (observations) available for a nonstationary time-series domain $\mathcal{D}_\mathbf{x}$, and each sample $\mathbf{x}_i=\{x_{i, 0}, ..., x_{i, t}, ...\}$ is characterized by its deterministic function, i.e., $\mathbf{x}_i(t) = x_{i, t} = \mathrm{x}_i(t)$, $i\in[1, M_\mathcal{D}]$. If we apply Hilbert Transformation $\mathrm{HT}(\mathbf{x}(t)) = \widehat{\mathbf{x}}(t) = \int_{-\infty}^{\infty}\mathrm{x}(\tau)\frac{1}{\pi(t-\tau)}d\tau$ to augment these time-series samples, the nonstationary statistics of augmented samples are different from the original ones,
% $
% \mathrm{Pr}_{\mathbf{x} \sim \widehat{\mathcal{D}}_\mathbf{x}}(\mathbf{x})(t) \neq \mathrm{Pr}_{\mathbf{x} \sim \mathcal{D}_\mathbf{x}}(\mathbf{x})(t).
% $
% \end{theorem}

% \textbf{Insights.} This theorem illustrates that HT does change the nonstationary statistics of time series, proving that our phase augmentation can diversify the non-stationarity of time series. \tmlr{Need to reference the new figure with changed ADF.}

\subsection{Magnitude-Phase Separate Encoding} \label{subsec:stft}
% \lixu{separate may be better than separable, the section title also needs discussion, since section 3.1 is a name of an operation, while section 3.2 and 3.3 are network names}
% \textcolor{red}{Reorganize this section following the order: 1. Why we view mag and phase two separate modalities (literature tells us each single response contains sufficient information to reconstruct the original signals, and our pilot study tells us) 2. the reason why we choose STFT rather than DFT. 3. The detailed STFT process, including how to sample x(t) into x[m], and then how to compute mag and phase responses (with mathematical language: notations and formulas). 4. The reason why we need sub-feature normalization, and How to do sub-feature normalization, also in mathematical language.}

After augmenting the source domain with phase-shift using HT, next, we identify optimal ways to encode time series for generalization. While employing spectral transformation is a common approach, our perspective diverges from most existing methods which typically focus on separating time and frequency information. Rather, we unify the time and frequency context and instead consider the \emph{magnitude} and \emph{phase} information as distinct modalities of the original signals.
% We conduct a time-frequency unification using Short Term Fourier Transform (STFT) based feature extraction to preserve the semantics of temporal shifts in the frequency characteristics. 

\textbf{Intuition of treating phase and magnitude as separate modalities.} Building on insights from prior studies~\citep{he2023domain, kim2021broadcasted} highlighting the importance of spectral input in generalizable learning, we conduct a small-scale empirical study on the WISDM HAR dataset~\citep{kwapisz2011activity} to explore optimal time-frequency input methods. Specifically, we compare four approaches: magnitude-only, phase-only, concatenated magnitude and phase, and separate encoders for magnitude and phase. Results (see Table~\ref{tab:table_pilot}) demonstrate that using only phase input yields inferior performance compared to magnitude-only input, suggesting the latter contains more discriminative information for classification tasks. \begin{wraptable}{r}{0.4\textwidth}
\centering
\small
\vspace{-20pt}
    \caption{Comparison of various time-frequency input configurations.}
    \begin{tabular}{lc}
        \toprule
            Input Modality    & Accuracy \\ \midrule
            Only Magnitude (Mag)          & 0.81 $\pm$ 0.03   \\
            Only Phase (Pha)              & 0.62 $\pm$ 0.03    \\
            Mag-Pha Concatenate           & 0.73 $\pm$ 0.03 \\
            Mag-Pha Separate              & 0.85 $\pm$ 0.01 \\ \bottomrule
    \end{tabular}
    % \vspace{-15pt}
\label{tab:table_pilot}
\end{wraptable}Here the phase-only features achieve an accuracy of 0.62 in a six-class classification task -- significantly higher than chance accuracy (0.17) -- supporting the presence of task-discriminating but time-varying attributes in the phase response; motivating us to use it as an approximate proxy for signal's nonstationarity in \method{}. Also, concatenating magnitude and phase does not improve performance, whereas separate encoding followed by late fusion proves superior in this case. This may be attributed to 1) the independent selection of high-level features from the magnitude and phase for the task of classification, and 2) the learning about non-stationarity from the phase information. Additionally, these magnitude and phase encoders for separate encoding can be deployed in parallel after training since they have no data dependencies, potentially optimizing computational efficiency during inference.

% However, more and more studies demonstrate that either magnitude or phase individually contains sufficient information to characterize the original signal. Inspired by this, we adopt to view the magnitude and phase as two separate modalities of the original signals. Such adoption has been preliminarily validated by the pilot study in SectionX. 

% Before encoding magnitude and phase as two separate modalities, the very first step is to transform the time series from the time to the frequency domain.

We adopt STFT instead of DFT because the DFT is typically suited for stationary, periodic signals, whereas time-varying signals require a method that accounts for changes over time. The STFT addresses this by applying the DFT sequentially within a moving window, capturing both the frequency and time information across the entire time series. Specifically, for each training sample $\mathbf{x} \in \mathbf{S}^\prime$ with a continuous time function $\mathbf{x}(t)$, sampling it at a fixed rate generates a discrete time series denoted as $\mathbf{x}[n]$ with a sequence length $N$, we have:
% \lixu{You use $N$ here, but why don't modify $\widetilde{N}$ in the following} \textcolor{red}{N is for length of $x[n]$ and $\widetilde{N}$ is the time dimension of $f_\mathbf{x}[n, k]$}, we have 
% \begin{equation}
%     f_\mathbf{x}[m, n] = \sum_{n=-\infty}^\infty w[n-m] \mathbf{x}[n] e^{-i \xi n},
% \end{equation}\label{eq:stft}
\vspace{-10pt}
\begin{equation}
    % f_\mathbf{x}[n, k] = \sum_{n=-\infty}^\infty w[n-m] \mathbf{x}[n] e^{-i \xi n},
    f_\mathbf{x}[n, k] = \sum_{m=n-(W-1)}^n w[n-m] \mathbf{x}[m] e^{i\xi_k m}.
    \label{eq:stft}
% \vspace{-5pt}
\end{equation}
The STFT of $\mathbf{x}[n]$, $f_\mathbf{x}[n, k]$, is a function of both discrete time $n$ and frequency bin indices $k$ with lengths $\widetilde{N}$ and $\Xi$, respectively. $\xi_k$ is a digital frequency variable given by $\xi_k = \frac{2 \pi k}{\Xi}$ and $w[\cdot]$ is a window function. Without losing generality, we adopt the Hanning window with window length $W$, i.e., $w[n] = 0.5(1-cos\frac{2\pi n}{W-1})$ where $0\leq n \leq W-1$. Note that the length and shape of the window determine the time-frequency resolution. A larger $W$ provides better frequency resolution and a smaller $W$ gives a better temporal scale. We set $W$ to be randomly sampled powers of 2 for each time-series feature, i.e., $W_i = 2^{p_i} \leq \Xi, p_i \sim \mathcal{U}\in\mathbb{Z}_0^+, i\in[1, V]$, where $\mathcal{U}$ denotes a uniform distribution for integers. After obtaining $f_\mathbf{x}[n, k]$, we can compute its magnitude and phase as:
% \vspace{-10pt}
\begin{equation}
\begin{aligned}
    \mathrm{Mag}(\mathbf{x}) = \sqrt{\mathrm{Re}(f_\mathbf{x}[n, k])^2 + \mathrm{Im}(f_\mathbf{x}[n, k])^2},
    \mathrm{Pha}(\mathbf{x}) = \arctan2\left(\mathrm{Im}(f_\mathbf{x}[n]), \mathrm{Re}(f_\mathbf{x}[n, k])\right),
\end{aligned}
% \vspace{-5pt}
\end{equation}
where $\mathrm{Im}(\cdot)$ and $\mathrm{Re}(\cdot)$ indicate imaginary and real parts of a complex number, and $\arctan2(\cdot)$ is the two-argument form of arctan. Then we take $\mathrm{Mag}(\mathbf{x}), \mathrm{Pha}(\mathbf{x}) \in \mathbb{R}^{V \times \Xi \times \widetilde{N}}$ as inputs of two separate encoders $F_{\mathrm{Mag}}$ and $F_{\mathrm{Pha}}$. This approach is motivated by the viability of reconstructing a time-series signal using phase and magnitude responses~\citep{hayes1980signal, jacques2020keep}, which is supported by our study below.

Before fusing the extracted embeddings of $F_\mathrm{Mag}$ and $F_\mathrm{Pha}$, we incorporate sub-feature normalization proposed by~\citet{chang2021subspectral}. Specifically, the embeddings of $F_\mathrm{Mag}$ and $F_\mathrm{Pha}$ are divided into $B$ sub-feature spaces. We apply normalization in each sub-feature space for each time-series variate, $
    F_\mathrm{Mag}(\mathbf{x}) \!=\! \left\{F_\mathrm{Mag}(\mathbf{x})_b \!:=\! \frac{F_\mathrm{Mag}(\mathbf{x})_b - \overline{F_\mathrm{Mag}(\mathbf{x})_b}}{\sigma\left({F_\mathrm{Mag}(\mathbf{x})_b}\right)}\right\}_{b=1}^B,
\label{eq:ssn}
$
% \vspace{-5pt}
% \begin{equation}
% $$
%     F_\mathrm{Mag}(\mathbf{x}) \!=\! \left\{F_\mathrm{Mag}(\mathbf{x})_b \!:=\! \frac{F_\mathrm{Mag}(\mathbf{x})_b - \overline{F_\mathrm{Mag}(\mathbf{x})_b}}{\sigma\left({F_\mathrm{Mag}(\mathbf{x})_b}\right)}\right\}_{b=1}^B,
% \label{eq:ssn}
% \vspace{-5pt}
% $$
% \end{equation}
% $i\in(1,v)$ as $\hat{x}^i_b = \frac{x^i_b - \mu^i_b}{\sigma^i_b}$ 
% To ensure the generalisability of this schema, if $w$ happens to be a prime number we simply conduct batch-normalization spatially (across all features, $b=w$).
where $\overline{(\cdot)}$ and $\sigma(\cdot)$ denote the computation of the mean and variance of the given input. The same sub-feature normalization is also conducted on $F_\mathrm{Pha}(\mathbf{x})$. Then, both $F_\mathrm{Mag}(\mathbf{x})$ and $F_\mathrm{Pha}(\mathbf{x})$ are fused along the variate axis by multiplying with 2D convolution kernels denoted as a fusing encoder $F_\mathrm{Fus}$. The fused embeddings $\mathbf{r}_\mathrm{Fus} = F_\mathrm{Fus}(F_\mathrm{Mag}(\mathbf{x}), F_\mathrm{Pha}(\mathbf{x}))$ are then fed into the following modules. 
\subsection {Phase-Residual Feature Broadcasting}
\label{subsec:broadcasting}

Lastly, we outline our phase-based broadcasting approach to achieve domain generalizable representation learning. It starts with a depthwise feature encoder, $F_\mathrm{Dep}$, which transforms the fused embeddings, $\mathbf{r}_\mathrm{Fus}$, into 1D feature maps, $\mathbf{r}_\mathrm{Dep}$, along the temporal dimension, given as:
% \textcolor{red}{Is there are reason to mix $\mathbf{r}_\mathrm{Fus}$ and $\mathbf{F}_\mathrm{Fus}$?} \lixu{we couldn't say a neural network layer has a temporal dimension, only an embedding has the temporal dimension. Using $\mathbf{r}$ to denote representations can prevent usage of multibracket, since the representation experiences too many input and output in our model}
$$
% \vspace{-0.5em}
    \mathbb{R}^{C(\mathbf{r}_\mathrm{Fus}) \times D(\mathbf{r}_\mathrm{Fus}) \times T(\mathbf{r}_\mathrm{Fus})} \rightarrow \mathbb{R}^{C(\mathbf{r}_\mathrm{Fus}) \times 1 \times T(\mathbf{r}_\mathrm{Fus})},
% \vspace{-0.5em}
$$
where $C(\cdot)$, $D(\cdot)$, and $T(\cdot)$ represent the channel number, the feature dimensions, and the temporal dimensions of an embedding. $F_\mathrm{Dep}$ is implemented as several convolution layers followed by an average pooling operation to unify all features at each temporal index. Once the 1D feature map is obtained, we attach a sequence-to-sequence (the dimension format of the feature map remains intact) temporal encoder, $F_\mathrm{Tem}$, to characterize its temporal dependency and semantics. The choice of backbone for $F_\mathrm{Tem}$ is not central to our design and a suitable sequence-to-sequence encoder can be chosen. Here we leverage convolution layers to form $F_\mathrm{Tem}$, and we have also tested other architectures (please refer to Section~\ref{app:F_temp} in the Appendix for details). We adopt this feature consolidation approach to enable specialized learning of spectral attributes by $F_\mathrm{Dep}$ and global temporal dependencies using $F_\mathrm{Tem}$, resulting in a more valuable overall semantic characterization.

We now introduce a non-linear projection of $F_\mathrm{Pha}(\mathbf{x})$ as a shortcut through $F_\mathrm{Dep}$ to $F_\mathrm{Tem}$. To suitably broadcast with the output dimensions of $F_\mathrm{Tem}$, we use a projection head, $g_\mathrm{Res}$ for the transformation:
% \textcolor{red}{Here too the temporal dimensions can arise from same inputs as C and D. Is there a reason you denote it differently?}
% \begin{equation}
% \begin{aligned}
%     &\mathbb{R}^{C(F_\mathrm{Pha}(\mathbf{x})) \times D(F_\mathrm{Pha}(\mathbf{x})) \times T(F_\mathrm{Pha}(\mathbf{x}))} \\ &\rightarrow \mathbb{R}^{C(\mathbf{r}_\mathrm{Fus}) \times D(F_\mathrm{Pha}(\mathbf{x})) \times T(\mathbf{r}_\mathrm{Fus})}.
% \end{aligned}
% \end{equation}
% \vspace{-2pt}
$$
\begin{aligned}
    &\mathbb{R}^{C(F_\mathrm{Pha}(\mathbf{x})) \times D(F_\mathrm{Pha}(\mathbf{x})) \times T(F_\mathrm{Pha}(\mathbf{x}))} \rightarrow \mathbb{R}^{C(\mathbf{r}_\mathrm{Fus}) \times D(F_\mathrm{Pha}(\mathbf{x})) \times T(\mathbf{r}_\mathrm{Fus})}.
\end{aligned}
$$
\begin{wrapfigure}{r}{0.3\textwidth}
    \vspace{-20pt}
    \begin{center}
        \includegraphics[width=0.27\textwidth]{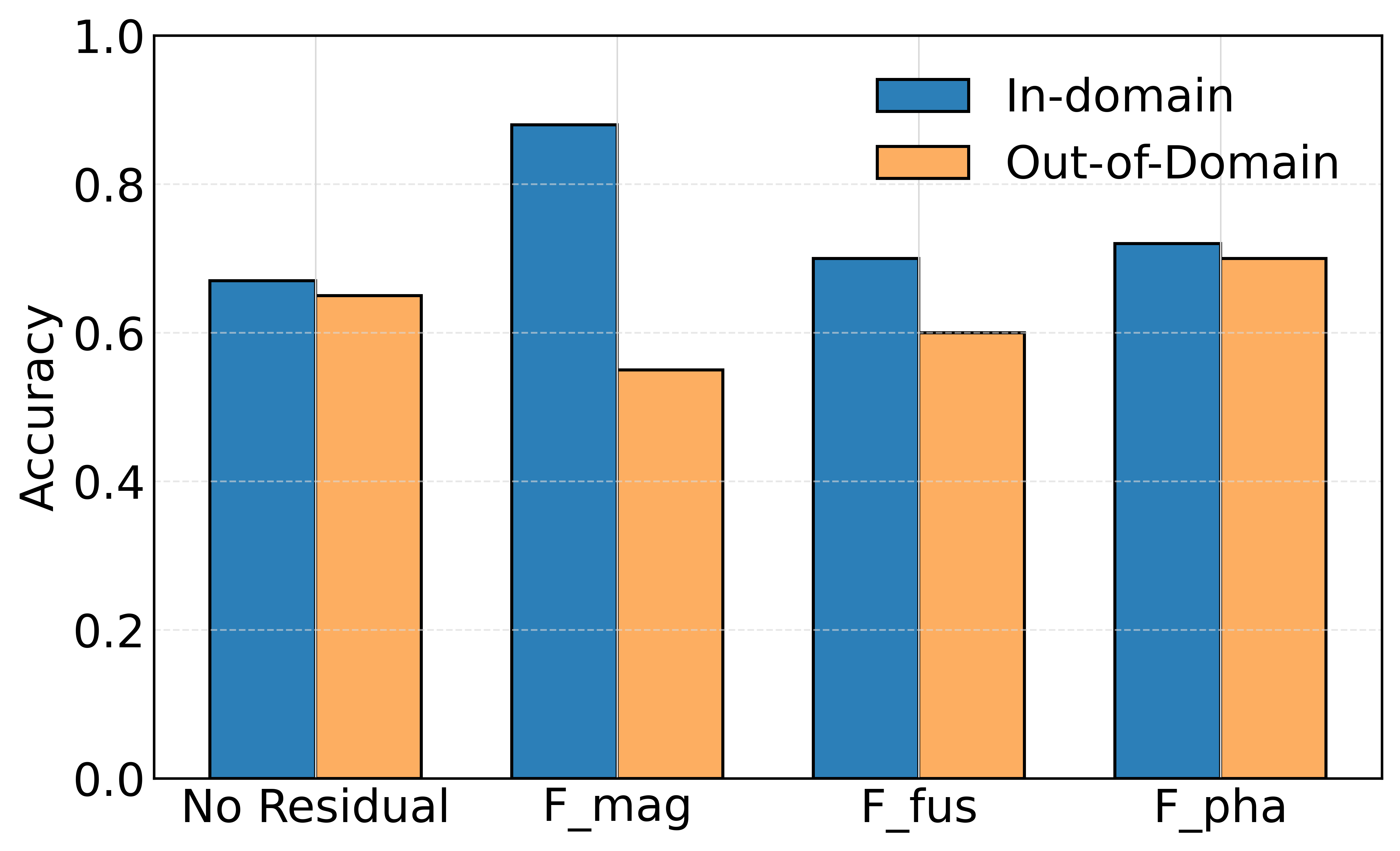}
        \vspace{-10pt}
        \caption{\small \tmlr{Comparison of generalization performance of different residual broadcasting features.}}
        \vspace{-15pt}
        \label{fig:phase_res}      
    \end{center}
\end{wrapfigure}
\tmlr{We conduct a controlled experiment using different residual feature broadcasting—no residual connection, using magnitude ($F_\mathrm{Mag}$), using phase ($F_\mathrm{Pha}$), and using the fused magnitude and phase ($F_\mathrm{Fus}$)—and evaluate this on held-out samples from the source domain, denoted as in-domain accuracy, and on the target domain, denoted as out-of-domain accuracy. We present the results in Figure~\ref{fig:phase_res} on a Gesture Recognition (GR) dataset. It is not surprising that the magnitude residual performs very well for in-domain evaluation; however, the drop in OOD accuracy can be indicative that the model has overfit to the in-domain non-task specific non-stationarity. Using a phase residual especially since the phase component of the dataset is diversified thorugh the porosed augmentation helps implicitly regularize the model's learnt against non-task specific nonstationarities.}

After the projection, we can broadcast the output of $F_\mathrm{Tem}$ to form the final representation $\mathbf{r}$ that is intended to learn discriminatory characteristics despite non-stationarity: 
%Overall, the broadcasting can be expressed as:
%\vspace{-5pt}
\begin{equation}
    \mathbf{r} = F_\mathrm{Tem}(\mathbf{r}_\mathrm{Dep}) + g_\mathrm{Res}(F_{\mathrm{Pha}}(\mathrm{Pha}(\mathbf{x}))).
% \vspace{-5pt}
\end{equation}
\vspace{-6pt}
After these efforts to preserve and enhance the discriminatory characteristics amid input's non-stationarity, we now optimize for semantic distinction. This optimization is achieved with a Cross-Entropy Loss applied to a classification head $g_\mathrm{Cls}$, which is attached to $F_\mathrm{Tem}$ as $
    \mathcal{L}_\mathrm{CE} = \frac{1}{N_B}\sum_{i=1}^{N_B} \mathbf{y}_i \log{g_\mathrm{Cls}(\mathbf{r})},
$
% \vspace{-5pt}
% \begin{equation}
%     \mathcal{L}_\mathrm{CE} = \frac{1}{N_B}\sum_{i=1}^{N_B} \mathbf{y}_i \log{g_\mathrm{Cls}(\mathbf{r})},
% \vspace{-5pt}
% \end{equation}
where $N_B$ is the size of a batch in the mini-batch training, and $\mathbf{y}_i$ is the one-hot form of the label $y_i$. 
\begingroup
\renewcommand\thefootnote{}\footnotetext{Code is available at \url{https://github.com/payalmohapatra/PhASER}}
\addtocounter{footnote}{-1}
\endgroup

\subsection{\tmlr{Theoretical Motivation}}
\label{subsec:theory}

Here we provide some theoretical insights to demonstrate that our method design is rigorously motivated. Detailed definitions and proofs are provided in Section~\ref{sec_proof} of the Appendix.
% of the following theorems and lemmas
% Using infimum and supremum notations to show theoretical guarantee
\begin{definition}[\textbf{$\beta$-Divergence}] Suppose two data domains $\mathcal{D}_1, \, \mathcal{D}_2$ are built on input variable $\mathbf{x}$ and label variable $y$. Let $q > 0$ be a constant. The $\beta$-Divergence between $\mathcal{D}_1$ and $\mathcal{D}_2$ is defined as:
% \vspace{-5pt}
\begin{equation}
\beta_q(\mathcal{D}_1 \| \mathcal{D}_2) = \left[\mathbb{E}_{(\mathbf{x}, y) \sim \mathcal{D}_2}\left(\frac{\mathcal{D}_1(\mathbf{x}, y)}{\mathcal{D}_2(\mathbf{x}, y)} \right)^q \right]^{\frac{1}{q}}.
% \vspace{-2pt}
\end{equation}
Per the definition in~\citep{divergence}, $\beta$-Divergence can be linked to the R\'{e}nyi Divergence~\citep{van2014renyi} $\mathrm{RD}_q(\cdot)$ as:
% \vspace{-5pt}
\begin{equation}
\beta_q(\mathcal{D}_1 \| \mathcal{D}_2) = 2^{\frac{q-1}{q}\mathrm{RD}_q(\mathcal{D}_1 \| \mathcal{D}_2)}.
% \vspace{-5pt}
\end{equation}
\end{definition}

\begin{lemma}[\textbf{Bounding $\beta$-Divergence in A Convex Hull}]
\label{lemma_bounding_beta}
Let $S$ be a set of source domains, denoted as $S = \{\mathcal{S}_i\}_{i=1}^{N_S}$. A convex hull $\Lambda_S$ considered here consists of a mixture distributions $\Lambda_S = \{\Bar{\mathcal{S}}: \Bar{\mathcal{S}}(\cdot) = \sum_{i=1}^{N_S} \pi_i \mathcal{S}_i (\cdot), \pi_i \in \Delta_{N_S - 1}\}$, where $\Delta_{N_S - 1}$ is the ($N_S \!-\! 1$)-th dimensional simplex. Let $\beta_q(\mathcal{S}_i \| \mathcal{S}_j) \leq \epsilon$ for $\forall i, j \in [N_S]$, and then we have the following relation for the $\beta$-Divergence between any pair of two domains $\mathcal{D}^\prime, \, \mathcal{D}^{\prime \prime} \in \Lambda_S$ in the convex hull:
\vspace{-2pt}
\begin{equation}
\beta_q(\mathcal{D}^\prime \| \mathcal{D}^{\prime \prime}) \leq \epsilon.
\vspace{-2pt}
\end{equation}
\end{lemma}

% \begin{proposition}[\textbf{\tmlr{Informal Risk of An Unseen Time-Series Domain}}]
\begin{proposition}[\textbf{\tmlr{Heuristic Risk Bound for an Unseen Domain}}]
\label{theorem_risk}
Let $\mathcal{H}$ be a hypothesis space built from a set of source time-series domains, denoted as $S = \{\mathcal{S}_i\}_{i=1}^{N_S}$ with the same value range (i.e., the supports of these source domains are the same). Suppose $q>0$ is a constant. For any unseen time-series domain $\mathcal{D}_\mathrm{U}$ from the convex hull $\Lambda_S$, we have its closest element $\mathcal{D}_{\Bar{\mathrm{U}}}$ in $\Lambda_S$, i.e., $\mathcal{D}_{\Bar{\mathrm{U}}} = \arg \min\limits_{\pi_1, ..., \pi_{N_S}} \beta_q(\mathcal{D}_{\Bar{\mathrm{U}}} \| \sum_{i=1}^{N_S}\pi_i\mathcal{S}_i)$. Then the risk of $\mathcal{D}_\mathrm{U}$ on any $\rho$ in $\mathcal{H}$ is:
\vspace{-6pt}
\begin{equation}
R_{\mathcal{D}_\mathrm{U}}[\rho] \leq \frac{1}{2}\mathrm{d}_{\mathcal{D}_\mathrm{U}}(\rho) + \epsilon \cdot \left[\mathrm{e}_{\mathcal{D}_{\Bar{\mathrm{U}}}}(\rho) \right]^{1-\frac{1}{q}},
% \vspace{-3pt}
\label{eq_upperbound}
\end{equation}
where $\mathrm{d}_{\mathcal{D}}(\rho)$ and $\mathrm{e}_{\mathcal{D}}(\rho)$ are an expected disagreement and an expected joint error of a domain $\mathcal{D}$, respectively.
% and they are defined as follows,
% \begin{align}
% \mathrm{d}_\mathcal{D}(\rho) &= \mathbb{E}_{\mathbf{x}\sim\mathcal{D}_\mathbf{x}} \mathbb{E}_{h \sim \rho} \mathbb{E}_{h^\prime \sim \rho} \mathrm{I}[h(\mathbf{x}) \neq h^\prime(\mathbf{x})], \\
% \mathrm{e}_\mathcal{D}(\rho) &= \mathbb{E}_{(\mathbf{x}, y)\sim\mathcal{D}} \mathbb{E}_{h \sim \rho} \mathbb{E}_{h^\prime \sim \rho} \mathrm{I}[h(\mathbf{x}) \neq y]\mathrm{I}[h^\prime(\mathbf{x}) \neq y],
% \end{align}
% where $\mathrm{I}[\cdot]$ is an indicator function with $\mathrm{I}[\mathrm{True}]=1$ and $\mathrm{I}[\mathrm{False}]=0$. 
The $\epsilon$ is a value larger than the maximum $\beta$-Divergence in $\Lambda_S$:
% \vspace{-2pt}
\begin{equation}
\epsilon \geq \max\limits_{i, j\in[N_S], i \neq j, t\in [0, +\infty)} 2^{\frac{q-1}{q}\mathrm{RD}_q(\mathcal{S}_i(t) \| \mathcal{S}_j(t))},
\vspace{-5pt}
\end{equation}
\vspace{-6pt}
\begin{equation}
\begin{aligned}
\text{where}\,\,\mathrm{RD}_q(\mathcal{S}_i(t) \| \mathcal{S}_j(t)) =&\frac{q(\mu_{j, t}-\mu_{i, t})^2}{2(1-q)\sigma_{i, t}^2 + 2\sigma_{j, t}^2} + \frac{\ln{\frac{\sqrt{(1-q)\sigma_{i, t}^2 + \sigma_{j, t}^2}}{\sigma_{i, t}^{1-q}\sigma_{j, t}^q}}}{1-q}.
\label{eq_rd_details}
\end{aligned}
% \vspace{-5pt}
\end{equation}
\end{proposition}

\textbf{Insights.} Theorem~\ref{theorem_risk} indicates potential efforts to reduce the generalization risk of an unseen target domain. According to Eq.~(\ref{eq_upperbound}), the risk is bounded by two terms. The first term $\mathrm{d}_{\mathcal{D}_\mathrm{U}}(\rho)$ is the expected disagreement of $\mathcal{D}_\mathrm{U}$ and we are unable to conduct any approximation without accessing the data from $\mathcal{D}_\mathrm{U}$. Regarding the second term, the coefficient $\epsilon$ can be viewed as the maximum $\beta$-Divergence of source domains, and according to Eq.~(\ref{eq_rd_details}), the nonstationary statistics of time series are arguments of the $\beta$-Divergence. We regard the $\beta$-Divergence as a proxy for non-stationarity. \tmlr{Our architectural design leverages an implicit regularization effect, achieved by reintroducing the phase dictionary through phase-residual broadcasting at deeper network layers. In this mechanism, the $F_{\mathrm{Pha}}$ features, which approximately encapsulate diverse (due to augmentation) non-stationarity (shown earlier in Figure~\ref{fig:motivation_ns_phase}), are used to modulate the task-specific information learned so far using $F_{\mathrm{Temp}}$. This forces the network to repeatedly disentangle and emphasize task-relevant representations deeper in the layers as well~\citep{noh2017regularizing}, resulting in implicit regularization that contributes to a tighter generalization risk bound according to Theorem~\ref{theorem_risk}. Empirically, this translates to enhanced robustness and improved transferability to unseen target domains, as shown in Figure~\ref{fig:phase_res}.}

\section{Experiments}
\label{sec:experiments}
We extensively evaluate our proposed \method{} framework against \iclr{13} state-of-the-art approaches (including a large foundation time-series model), on 5 datasets across three time-series applications. Our evaluation metric is per-segment accuracy. More implementation-specific details are provided in Section~\ref{app:implementation} of the Appendix. Our source codes are provided in the Supplementary Materials.

\begin{table*}[!htbp]
\vspace{-15pt}
\small
\caption{\small Classification accuracy of Target 1$\sim$4 scenarios for cross-person generalization in Human Activity Recognition on WISDM, HHAR, and UCIHAR (\textbf{Best} in bold, {\ul second-best} underlined).}
\resizebox{1\textwidth}{!}{
\setlength{\tabcolsep}{1.6mm}{
\begin{tabular}{l|ccccc|ccccc|ccccc|c}
% \hline
\toprule
\multicolumn{1}{c|}{Dataset}                  & \multicolumn{5}{c|}{WISDM}                                                                                                             & \multicolumn{5}{c|}{HHAR}                                                                                                              & \multicolumn{5}{c|}{UCIHAR}                                                                                                            & HAR           \\ \cmidrule{1-16}
\multicolumn{1}{c|}{Target}                   & \multicolumn{1}{c}{1} & \multicolumn{1}{c}{2} & \multicolumn{1}{c}{3} & \multicolumn{1}{c|}{4}             & \multicolumn{1}{c|}{Avg.} & \multicolumn{1}{c}{1} & \multicolumn{1}{c}{2} & \multicolumn{1}{c}{3} & \multicolumn{1}{c|}{4}             & \multicolumn{1}{c|}{Avg.} & \multicolumn{1}{c}{1} & \multicolumn{1}{c}{2} & \multicolumn{1}{c}{3} & \multicolumn{1}{c|}{4}             & \multicolumn{1}{l|}{Avg.} & Avg.          \\ \midrule
ERM                      & \multicolumn{1}{c}{0.57}  & \multicolumn{1}{c}{0.50}  & \multicolumn{1}{c}{0.51}  & \multicolumn{1}{c|}{0.55}              & \multicolumn{1}{c|}{0.53}     & \multicolumn{1}{c}{0.49}  & \multicolumn{1}{c}{0.46}  & \multicolumn{1}{c}{0.45}  & \multicolumn{1}{c|}{0.47}              & \multicolumn{1}{c|}{0.47}     & \multicolumn{1}{c}{0.72}  & \multicolumn{1}{c}{0.64}  & \multicolumn{1}{c}{0.70}  & \multicolumn{1}{c|}{0.72}              & \multicolumn{1}{c|}{0.70}     &  0.57             \\
GroupDRO                 & 0.71                  & 0.67                  & 0.60                  & \multicolumn{1}{c|}{0.67}          & 0.66                      & 0.60                  & 0.53                  & 0.59                  & \multicolumn{1}{c|}{0.64}          & 0.59                      & {\ul 0.91}            & {\ul 0.84}            & 0.89                  & \multicolumn{1}{c|}{0.85}          & 0.87                      & 0.71          \\
DANN                     & 0.71                  & 0.65                  & 0.65                  & \multicolumn{1}{c|}{0.70}          & 0.68                      & \multicolumn{1}{c}{0.66}  & \multicolumn{1}{c}{0.71}  & \multicolumn{1}{c}{0.67}  & \multicolumn{1}{c|}{0.69}              & \multicolumn{1}{c|}{0.68}     & 0.84                  & 0.79                  & 0.81                  & \multicolumn{1}{c|}{0.86}          & 0.83                      &  0.73             \\
RSC                      & 0.69                  & 0.71                  & 0.64                  & \multicolumn{1}{c|}{0.61}          & 0.66                      & 0.52                  & 0.49                  & 0.44                  & \multicolumn{1}{c|}{0.47}          & 0.48                      & 0.82                  & 0.73                  & 0.74                  & \multicolumn{1}{c|}{0.81}          & 0.78                      & 0.64          \\
ANDMask                  & 0.74                  & 0.73                  & 0.69                  & \multicolumn{1}{c|}{0.69}          & 0.71                      & 0.63                  & 0.64                  & 0.66                  & \multicolumn{1}{c|}{0.69}          & 0.66                      & 0.86                  & 0.80                  & 0.76                  & \multicolumn{1}{c|}{0.78}          & 0.80                      & 0.72          \\
\iclr{InceptionTime}    & \iclr{\ul 0.83}       & \iclr{\ul 0.82}       & \iclr{0.80}           & \multicolumn{1}{c|}{\iclr{0.77}}    & \iclr{0.81}               & \iclr{0.77}           & \iclr{\ul 0.80}       & \iclr{\ul 0.82}       & \multicolumn{1}{c|}{\iclr{\ul 0.83}} & \iclr{\ul 0.80}           & \iclr{{\ul 0.91}}           & \iclr{0.82}           & \iclr{0.88}           & \multicolumn{1}{c|}{\iclr{0.91}}    & \iclr{0.88}               & \iclr{0.82}   \\
BCResNet                 & {\ul 0.83}            & 0.79                  & 0.75                  & \multicolumn{1}{c|}{0.78}          & 0.79                      & 0.66                  & 0.70                  & 0.75                  & \multicolumn{1}{c|}{0.68}          & 0.70                      & 0.81                  & 0.77                  & 0.78                  & \multicolumn{1}{c|}{0.83}          & 0.80                      & 0.76          \\
NSTrans                      & 0.43                  & 0.40                  & 0.37                  & \multicolumn{1}{c|}{0.37}          & 0.40                      & 0.21                  & 0.22                  & 0.27                  & \multicolumn{1}{c|}{0.28}          & 0.24                      & 0.35                  & 0.35                  & 0.51                  & \multicolumn{1}{c|}{0.47}          & 0.42                      & 0.35          \\
Koopa                      &0.63 &0.61 &0.72 &\multicolumn{1}{c|}{0.57} &0.63 &0.72 & 0.63 & 0.72 &\multicolumn{1}{c|}{0.69} &0.69 & 0.81 & 0.72 & 0.81 &\multicolumn{1}{c|}{0.77} & 0.78 & 0.70                    \\
Ours+RevIN & 0.86         & 0.85         & 0.84         & \multicolumn{1}{c|}{0.84} & 0.85             & 0.82         & 0.82         & 0.92         & \multicolumn{1}{c|}{0.85} & 0.85             & 0.96         & 0.90         & 0.93         & \multicolumn{1}{c|}{0.97} & 0.94             & 0.88 \\ 
MAPU                     & 0.75                  & 0.69                  & 0.79                  & \multicolumn{1}{c|}{0.79}          & 0.75                      & 0.73                  & 0.72                  & 0.81                  & \multicolumn{1}{c|}{0.78}    & 0.76                      & 0.85                  & 0.80                  & 0.85                  & \multicolumn{1}{c|}{0.82}          & 0.83                      & 0.78          \\
Diversify                & 0.82                  &  0.82           & {\ul 0.84}            & \multicolumn{1}{c|}{{\ul 0.81}}    & {\ul 0.82}                & {\ul 0.82}            &  0.76            & 0.82            & \multicolumn{1}{c|}{0.68}          & 0.77                & 0.89                  & {\ul 0.84}            & {\ul 0.93}            & \multicolumn{1}{c|}{{\ul 0.90}}    & {\ul 0.89}                & {\ul 0.83}    \\ 
Chronos                & 0.71                  & 0.66            & 0.65            & \multicolumn{1}{c|}{{0.62}}    & {0.66}                & {0.66}            & {0.73}            & {0.75}            & \multicolumn{1}{c|}{0.66}          & {0.72}                & 0.56                  & {0.57}            & {0.50}            & \multicolumn{1}{c|}{{0.82}}    & {0.61}                & {0.67}    \\ \midrule
% Ours+RevIN* & \textit{0.86}         & \textit{0.85}         & \textit{0.84}         & \multicolumn{1}{r|}{\textit{0.84}} & \textit{0.85}             & \textit{0.82}         & \textit{0.82}         & \textit{0.92}         & \multicolumn{1}{r|}{\textit{0.85}} & \textit{0.85}             & \textit{0.96}         & \textit{0.90}         & \textit{0.93}         & \multicolumn{1}{r|}{\textit{0.97}} & \textit{0.94}             & \textit{0.88} \\ 
Ours                     & \textbf{0.86}         & \textbf{0.85}         & \textbf{0.85}         & \multicolumn{1}{c|}{\textbf{0.82}} & \textbf{0.85}             & \textbf{0.83}         & \textbf{0.83}         & \textbf{0.94}         & \multicolumn{1}{c|}{\textbf{0.88}} & \textbf{0.87}             & \textbf{0.96}         & \textbf{0.91}         & \textbf{0.95}         & \multicolumn{1}{c|}{\textbf{0.97}} & \textbf{0.95}             & \textbf{0.89} \\
\bottomrule
\end{tabular}
}}
\label{tab:har_result}
\vspace{-10pt}
\end{table*} 
% In this section we present the performance of \method{} framework against other state-of-the-art approaches and offer insights into our design elements \lixu{As regular ML papers, we don't need this sentence, instead, we need the following one. Since this sentence cannot offer useful information and does not mention the code is provided in SM}. 
% The experiment setups include datasets, implementation details, and comparison baselines, more details can be found in the Appendix. Our implementation codes are provided in the supplementary materials.

\noindent \textbf{Datasets.} 
We conduct experiments on three common time-series applications -- Human Activity Recognition (HAR), Sleep-Stage Classification (SSC), and Gesture Recognition (GR). For HAR, we use 3 benchmark datasets: \textit{\textbf{WISDM}}~\citep{kwapisz2011activity} collected from 36 different users with 3 univariate dimensions, \textit{\textbf{UCIHAR}}~\citep{bulbul2018human} collected from 30 people with 9 variates, and \textit{\textbf{HHAR}}~\citep{stisen2015smart} collected from 9 users with 3 feature dimensions, comprising 6 distinct activities with a sequence length of 128. For SSC, the dataset~\citep{goldberger2000physiobank} consists of single-channel EEG data from 20 healthy individuals with a sequence length of 3000. For GR, the dataset~\citep{lobov2018latent} is 8-channel EMG data for 6 different gestures, with a sequence length of 200, prepared similarly as in~\citep{lu2022out}. We follow the setup of ADATime~\citep{ragab2023adatime} for HAR and SSC. More data-specific details are provided in Table~\ref{tab:app_dataset} of the Appendix. Specifically, the class distributions of the considered datasets in Figure~\ref{app:data_dist}, as well as the trends of two performance metrics, segment-wise Area-under-the-curve (AUC) and accuracy, for the WISDM dataset in Figure~\ref{app:auc_roc}, are provided in the Appendix to justify the choice of performance metrics in accordance with previous works~\citep{ragab2023adatime, lu2022out}.

% We carry out experiments on three commonly used time-series applications -- Human Activity Recognition (HAR), Sleep-Stage Classification (SSC), and Gesture Recognition (GR). For HAR, we use three benchmark datasets: 1) \textit{\textbf{WISDM}}~\citep{kwapisz2011activity} collected from 36 different users with 3 univariate dimensions, 2) \textit{\textbf{UCIHAR}}~\citep{bulbul2018human} collected from 30 people with 9 variates, and 3) \textit{\textbf{HHAR}}~\citep{stisen2015smart} collected from 9 users with 3 feature dimensions. All HAR datasets consist of 6 distinct activities and their sequence length is 128.  SSC~\citep{goldberger2000physiobank} is built on continuously recorded single channel \textit{\textbf{EEG}} from 20 healthy individuals segmented to a sequence length of 3000. As for GR~\citep{lobov2018latent}, we utilize 8-channel \textit{\textbf{EMG}} signals to classify 6 different gestures with a sequence length of 200 and follow preprocessing steps similar to \citept{lu2022out}. And we follow the setup of ADATime~\citep{ragab2023adatime} for HAR and SSC. More data-specific details are provided in Table~\ref{tab:app_dataset} of the Appendix.

% We further provide evidence of non-stationarity in Table~\ref{tab:app_dataset} by reporting the mean ADF statistics throughout each dataset. More details on the categorical variables of each dataset are given in Table~\ref{tab:app_dataset}. 

\noindent \textbf{Experimental Setup.} Each dataset is divided into four distinct non-overlapping cross-domain scenarios, following the approach in~\citep{lu2023outofdistribution}. Details are provided in Section~\ref{app:subsubdataset} of the Appendix. 20\% of the training data is reserved for validation. Mean results from three trials are reported in the main text, with full statistics in Section~\ref{app:add_results} of the Appendix.

% Our setting includes domain labels to compare against domain-label-dependent baselines, however, we aim to generalize without accessing domain labels which is a more challenging and practical scenario for real-world applications.
% \lixu{most content here needs to be put in appendix, here we only need to list names of baselines we have compared} 
\noindent \textbf{Comparison Baselines.}
We conduct comparison with a range of state-of-the-art approaches including domain generalization algorithms -- ERM, DANN~\citep{ganin2016domain}, GroupDRO~\citep{sagawa2019distributionally}, RSC~\citep{huang2020self} and ANDMask~\citep{parascandolo2020learning} implemented based on the DomainBed benchmarking suite~\citep{gulrajani2020search}; an audio domain generalization method BCResNet~\citep{kim2021domain}; a time-series representation learning method MAPU~\citep{ragab2023source}; a strong deep-learning time-series classification model (top ranked by~\citet{middlehurst2024bake}), InceptionTime~\citep{ismail2020inceptiontime}, a time-series domain generalizable learning method Diversify~\citep{lu2022out}; and a large time-series foundation model Chronos~\citep{ansari2024chronos}. We also adapt the time-series forecasting models Nonstationary Transformer (NSTrans)~\citep{non_transformer_1} and Koopa~\citep{liu2024koopa}, and integrate a network-agnostic statistical technique RevIN~\citep{kim2021reversible} with our method (denoted as Ours+RevIN). We follow the default setups of these works and only conduct necessary modifications for our problem setting. Details are in Sections~\ref{app:baseline} and~\ref{app:comp_analyses} of the Appendix.

%%%%%%%%%%%%%%%%%%%%%%%%%%%%%%%%%%%%%%%%%%%%%%%%%%%%%%%%%%%%%%%%%%%%%%%%%%%%%%%%%%%%%%%%%%%%%%%%%%%%%%%%%%%%%%%%%%%%%%%%%%%%%%%%%
% \vspace{-10pt}
\begin{table}[!htbp]
    \centering
    \begin{minipage}[t]{0.475\textwidth}
        \centering
        \small
        \vspace{-12pt}
        \caption{ \small Classification accuracy with Source 0$\sim$8 person for one-person-to-another generalization on the HHAR dataset (\textbf{Best} in bold, {\ul second-best} underlined).}
        \vspace{10pt}
        \resizebox{\textwidth}{!}{
        \setlength{\tabcolsep}{.4mm}{
        \begin{tabular}{l|ccccccccc|c}
        % \hline
        \toprule
        \multicolumn{1}{c|}{Source}  & \multicolumn{1}{c}{0}    & \multicolumn{1}{c}{1}    & \multicolumn{1}{c}{2}    & \multicolumn{1}{c}{3}    & \multicolumn{1}{c}{4}    & \multicolumn{1}{c}{5}    & \multicolumn{1}{c}{6}    & \multicolumn{1}{c}{7}    & \multicolumn{1}{c}{8}    & \multicolumn{1}{|c}{Avg.}          \\ \midrule
        ERM             &0.27                          &0.40                          &0.41                          &0.44                          &0.42                          &0.44                          &0.45                          &0.44                          &0.48                          &0.42               \\
        GroupDRO        &0.33                          &0.53                          &0.38                          &0.48                          &0.47                          &0.51                          &0.47                          &0.48                          &0.49                          &0.46               \\
        DANN            &0.32                          &0.44                          & 0.42                          &0.45                          &0.42                          &0.48                          &0.49                          &0.45                          &0.51                          &0.44               \\
        RSC             & \multicolumn{1}{c}{0.27} & \multicolumn{1}{c}{0.45} & \multicolumn{1}{c}{0.38} & \multicolumn{1}{c}{0.45} & \multicolumn{1}{c}{0.40} &  \multicolumn{1}{c}{0.47}& \multicolumn{1}{c}{0.50} & \multicolumn{1}{c}{0.44} & \multicolumn{1}{c|}{0.53} &  0.43        \\
        ANDMask         & \multicolumn{1}{c}{0.34} & \multicolumn{1}{c}{0.50} & \multicolumn{1}{c}{0.37} & \multicolumn{1}{c}{0.43} & \multicolumn{1}{c}{0.46} &  \multicolumn{1}{c}{0.51}& \multicolumn{1}{c}{0.46} & \multicolumn{1}{c}{0.47} & \multicolumn{1}{c|}{0.52} &  0.45         \\
        {\iclr{InceptionTime}}        & \multicolumn{1}{c}{{\iclr{{\ul 0.52}}}} & \multicolumn{1}{c}{{\iclr{{\ul 0.62}}}} & \multicolumn{1}{c}{{\iclr{{\ul 0.44}}}} & \multicolumn{1}{c}{{\iclr{\textbf{0.69}}}} & \multicolumn{1}{c}{{\iclr{{\ul 0.60}}}} & \multicolumn{1}{c}{{\iclr{0.57}}} & \multicolumn{1}{c}{{\iclr{{\ul 0.66}}}} & \multicolumn{1}{c}{{\iclr{{\ul 0.64}}}} & \multicolumn{1}{c|}{{\iclr{{\ul 0.61}}}} & {\iclr{{\ul 0.59}}} \\
       
        BCResNet        & \multicolumn{1}{c}{0.28} & \multicolumn{1}{c}{0.48} & \multicolumn{1}{c}{0.32} & \multicolumn{1}{c}{0.47} & \multicolumn{1}{c}{0.42} & \multicolumn{1}{c}{0.52} & \multicolumn{1}{c}{0.44} & \multicolumn{1}{c}{0.45} & \multicolumn{1}{c|}{0.49} & 0.43          \\
        NSTrans         & \multicolumn{1}{c}{0.20} & \multicolumn{1}{c}{0.22} & \multicolumn{1}{c}{0.17} & \multicolumn{1}{c}{0.20} & \multicolumn{1}{c}{0.21} & \multicolumn{1}{c}{0.22} & \multicolumn{1}{c}{0.26} & \multicolumn{1}{c}{0.17} & \multicolumn{1}{c|}{0.20}  & 0.21               \\
        Koopa &0.32 &0.42 &0.37 &0.40 &0.42 &0.45 &0.35 &0.43  &0.48  &0.40\\
        \tmlr{Ours+RevIN}            & \multicolumn{1}{c}{0.48} & \multicolumn{1}{c}{0.66} & \multicolumn{1}{c}{0.57} & \multicolumn{1}{c}{0.65} & \multicolumn{1}{c}{0.61} & \multicolumn{1}{c}{0.64} & \multicolumn{1}{c}{0.65} & \multicolumn{1}{c}{0.64} & \multicolumn{1}{c|}{0.63} & 0.62 \\
        MAPU            &0.39                          &0.57                          &0.35                          &0.52                          &0.49                          &0.54                          &0.49                          &0.50                          &0.52                           &0.49               \\
        Diversify       & \multicolumn{1}{c}{0.42} & \multicolumn{1}{c}{{\ul 0.62}} & \multicolumn{1}{c}{0.32} & \multicolumn{1}{c}{0.62} & \multicolumn{1}{c}{0.56} & \multicolumn{1}{c}{0.61} & \multicolumn{1}{c}{0.53} & \multicolumn{1}{c}{0.52} & \multicolumn{1}{c|}{{\ul 0.61}} & 0.53    \\ 
        Chronos       & \multicolumn{1}{c}{{0.32}} & \multicolumn{1}{c}{0.23} & \multicolumn{1}{c}{0.26} & \multicolumn{1}{c}{0.25} & \multicolumn{1}{c}{{0.27}} & \multicolumn{1}{c}{0.23} & \multicolumn{1}{c}{0.21} & \multicolumn{1}{c}{0.24} & \multicolumn{1}{c|}{0.25} & {0.25}    \\ \midrule
        Ours           & \multicolumn{1}{c}{\textbf{0.53}} & \multicolumn{1}{c}{\textbf{0.70}} & \multicolumn{1}{c}{\textbf{0.63}} & \multicolumn{1}{c}{{\ul{0.66}}} & \multicolumn{1}{c}{\textbf{0.64}} & \multicolumn{1}{c}{\textbf{0.67}} & \multicolumn{1}{c}{\textbf{0.65}} & \multicolumn{1}{c}{{\textbf{0.67}}} & \multicolumn{1}{c|}{\textbf{0.62}} & \textbf{0.64} \\
        \bottomrule
        \end{tabular}}}
        \label{tab:one_to_x}
    \end{minipage}
    \hspace{0.01\textwidth}
    \begin{minipage}[t]{0.475\textwidth}
        \centering
        \small
        \vspace{-12pt}
        % \caption{Classification accuracy for Target 1$\sim$4 scenarios for cross-person generalization Sleep-Stage Classification and Gesture Recognition using EEG and EMG, with \textbf{best} and {\ul second-best} results highlighted.}
        \caption{\small Classification accuracy for cross-person generalization (Target 1$\sim$4) Sleep-Stage Classification (EEG) and Gesture Recognition (EMG) (\textbf{Best} in bold, {\ul second-best} underlined).}
        % \vspace{-6pt}
        \resizebox{\textwidth}{!}{
        \setlength{\tabcolsep}{.4mm}{
        \begin{tabular}{l|ccccc|ccccc}
        \toprule
        \multicolumn{1}{c|}{Application}   & \multicolumn{5}{c|}{Sleep-Stage   Classification}                                                                                      & \multicolumn{5}{c}{Gesture Recognition}                                                                                               \\ \midrule
        \multicolumn{1}{c|}{Target}    & \multicolumn{1}{c}{1} & \multicolumn{1}{c}{2} & \multicolumn{1}{c}{3} & \multicolumn{1}{c|}{4}             & \multicolumn{1}{c|}{Avg.} & \multicolumn{1}{c}{1} & \multicolumn{1}{c}{2} & \multicolumn{1}{c}{3} & \multicolumn{1}{c|}{4}             & \multicolumn{1}{c}{Avg.} \\ \midrule
        ERM       & \multicolumn{1}{c}{0.50}  & \multicolumn{1}{c}{0.46}  & \multicolumn{1}{c}{0.49}  & \multicolumn{1}{c|}{0.45}              & \multicolumn{1}{c|}{0.47}     & \multicolumn{1}{c}{0.45}  & \multicolumn{1}{c}{0.58}  & \multicolumn{1}{c}{0.57}  & \multicolumn{1}{c|}{0.54}              & \multicolumn{1}{c}{0.54}     \\
        GroupDRO  & 0.57                  & 0.56                  & 0.55                  & \multicolumn{1}{c|}{0.59}          & 0.57                      & 0.53                  & 0.36                  & 0.59                  & \multicolumn{1}{c|}{0.45}          & 0.48                     \\
        DANN      & \multicolumn{1}{c}{0.64}  & \multicolumn{1}{c}{0.63}  & \multicolumn{1}{c}{0.69}  & \multicolumn{1}{c|}{0.63}              & \multicolumn{1}{c|}{0.65}     & \multicolumn{1}{c}{0.60}  & \multicolumn{1}{c}{0.66}  & \multicolumn{1}{c}{0.65}  & \multicolumn{1}{c|}{0.64}              & \multicolumn{1}{c}{0.64}     \\
        RSC       & 0.50                  & 0.48                  & 0.52                  & \multicolumn{1}{c|}{0.46}          & 0.49                      & 0.50                  & 0.66                  & 0.64                  & \multicolumn{1}{c|}{0.56}          & 0.59                     \\
        ANDMask   & 0.55                  & 0.50                  & 0.54                  & \multicolumn{1}{c|}{0.57}          & 0.54                      & 0.41                  & 0.54                  & 0.45                  & \multicolumn{1}{c|}{0.39}          & 0.45                     \\
        \iclr{InceptionTime} & \iclr{0.74} & \iclr{0.78} & \iclr{0.72} & \multicolumn{1}{c|}{\iclr{0.80}} & \iclr{0.76} & \iclr{\ul 0.68} & \iclr{0.70} & \iclr{0.72} & \multicolumn{1}{c|}{\iclr{0.69}} & \iclr{0.70} \\

        BCResNet  & {\ul 0.79}            & {\ul 0.82}            & {\ul 0.79}            & \multicolumn{1}{c|}{{\ul 0.81}}    & {\ul 0.80}                & 0.62                  & 0.67                  & 0.65                  & \multicolumn{1}{c|}{0.61}          & 0.64                     \\
        NSTrans   & 0.43                  & 0.37                  & 0.42                  & \multicolumn{1}{c|}{0.35}          & 0.39                      & 0.31                  & 0.34                  & 0.34                  & \multicolumn{1}{c|}{0.32}          & 0.33                     \\
        Koopa &0.58 &0.62 &0.53 &\multicolumn{1}{c|}{0.49} &0.56 &0.47 &0.54 &0.60 &\multicolumn{1}{c|}{0.70} &0.58 \\
        \tmlr{Ours+RevIN}      & 0.82         & 0.79         & 0.78         & \multicolumn{1}{c|}{0.81} & 0.80             & 0.68         & 0.81         & 0.77         & \multicolumn{1}{c|}{0.76}    & 0.76            \\
        MAPU      & 0.69                  & 0.68                  & 0.65                  & \multicolumn{1}{c|}{0.69}          & 0.68                      & 0.64                  & 0.69                  & 0.71                  & \multicolumn{1}{c|}{0.68}          & 0.68                     \\
        Diversify & 0.73                  & 0.76                  & 0.68                  & \multicolumn{1}{c|}{0.77}          & 0.73                      & {\ul 0.68}            & {\ul 0.80}            & {\ul 0.75}            & \multicolumn{1}{c|}{\textbf{0.76}} & {\ul 0.75}               \\ 
        Chronos & 0.53                  & 0.47                  & 0.47                  & \multicolumn{1}{c|}{0.57}          & 0.51                      & {0.49}            & {0.54}            & {0.51}            & \multicolumn{1}{c|}{0.48} & {0.51}               \\ \midrule
        Ours      & \textbf{0.85}         & \textbf{0.80}         & \textbf{0.79}         & \multicolumn{1}{c|}{\textbf{0.83}} & \textbf{0.82}             & \textbf{0.70}         & \textbf{0.82}         & \textbf{0.77}         & \multicolumn{1}{c|}{{\ul 0.75}}    & \textbf{0.76}            \\ \bottomrule
        \end{tabular}}}
        \label{tab:eeg_result}
    \end{minipage}
    \vspace{-15pt}
\end{table}

\subsection{Effectiveness of \method{} across Applications}
%% Source :: https://www.tablesgenerator.com/
% \vspace{-3pt}
\textbf{Human Activity Recognition.} We assess the generalization ability of \method{} framework in two settings: 1) \textit{cross-person generalization}, where the model is trained on $N_S$ ($N_S > 1$) source domains and evaluated on unseen target domains, and 2) \textit{one-person-to-another}, where the model is trained on one person ($N_S=1$) and evaluated on another person. In the cross-person setting, as shown in Table~\ref{tab:har_result}, we find that existing state-of-the-art domain generalization methods, popular in vision-based domains, do not perform as well in time-series classification (such observation is consistent with previous works~\citep{gagnon2022woods, lu2022out}). \textbf{\method{} achieves superior out-of-domain generalization performance across all cases, notably outperforming the best baseline on WISDM, HHAR, and UCIHAR by 3\%, 9\%, and 6\%, respectively}. In the more challenging one-person-to-another setting, as shown in Table~\ref{tab:one_to_x}, we select the HHAR dataset due to its high non-stationarity, and the results show that \textbf{\method{} excels in this setting as well, outperforming Diversify by almost 20\% \iclr{and InceptionTime by almost 8\%}}.

\textbf{Sleep-Stage Classification.}
Next, we evaluate \method{} for \textit{cross-person generalization} in five types of sleep-stage classification using EEG. Past methods~\citep{ragab2023adatime, he2023domain} generally report the lowest performance in their respective settings for SSC tasks indicating its inherent complexity. The results in Table~\ref{tab:eeg_result} (left) show that \textbf{\method{} provides the best performance in all cases, outperforming the best baseline (BCResNet) by 2\% and the time-series domain generalization baseline (Diversify) by almost 11\%}.

% and achieves an average classification accuracy of 0.82 without accessing the domain labels or any target samples. It outperforms the best baseline (BCResNet) by 2\% and the time-series domain generalization baseline (Diversify) by almost 11\%. The best performance of our method in this case highlights the value even in spectral representation of physiological time-series data. 

\textbf{Gesture Recognition.}
In GR, the used bio-electronic signals are heavily influenced by user behavior and sensor time-varying properties, which correspond to natural non-stationarity. We follow the approach in~\citep{lu2023outofdistribution} to use 6 common classes when conducting evaluations in a \textit{cross-person setting}. The results in Table~\ref{tab:eeg_result} (right) show that \textbf{\method{} again offers the best overall performance}.

\subsection{Further Analysis}
% We present an analysis of our method in two folds -- first systematically ablating each proposed component, and second, demonstrating the enhanced performance by adopting a \textit{phase-driven} approach to NSTrans using components from our proposed \method{} framework.
\noindent \textbf{Ablation Study.} We examine the impact of our proposed design components in two cases: WISDM and GR (Table~\ref{tab:ablation}). The first row represents the performance of the complete \method{} framework, with subsequent rows showing performance with specific components detached or modified (details in Section~\ref{app:implementation} of the Appendix). When phase augmentation is omitted (row 2), performance notably decreases (by 11.6\% on WISDM and 5.8\% on GR). This finding along with results from Tables~\ref{tab:gen_results} and~\ref{tab:domain_discrepancy}, is consistent with recent vision domain research showing that Hilbert Transform-based augmentations indeed preserve task-relevant features while generating diverse representations~\citep{wang2025split}. Comparing the results of row 6 with that of row 5 confirms the importance of separate phase-magnitude encoding, aligning with earlier findings in Table~\ref{tab:table_pilot}. Under identical conditions (comparing row 5 with row 1), phase-residual broadcasting boosts the performance of \method{} by 4\%, aligning well with our design motivation that phase can be considered a proxy for non-stationarity. Reintroducing this phase-dictionary deeper in the layers enables the model to learn task-specific representations that are more robust to non-stationarity, making it better equipped to handle unseen non-stationarity in the target domains. Removing the phase-based residual and separate encoding structure (rows 3-7 in Table~\ref{tab:ablation}) results in average performance drops of 10.6\% and 13.7\%, respectively. \textbf{This demonstrates the value of all the components in \method{}}.

% We use two cases, HAR using WISDM and GR, to study the impact of each of our proposed design components as shown in Table~\ref{tab:ablation}. The first row indicates the performance of the complete \method{} framework and the following rows present the performance of different versions of \method{} with certain components detached or modified (details of how we detach or modify \method{} are provided in Section~\ref{app:implementation} of Appendix). From this table, we can observe that when the phase augmentation is not carried out (row 2), the performance significantly drops (by 11.6\% on WISDM, and 5.8\% on GR). Recall from our motivation study (refer to Table~\ref{tab:table_pilot}) that providing only magnitude as the input was the second best option compared to separate encoding. The results in rows 5 and 6 further support this observation, highlighting the importance of incorporating phase information for learning generalizable representations. Overall, we can observe that when ablating the phase-based residual and separate encoding structure (rows 3-7) the average performance drops by 10.6\% and 13.7\%, respectively. \textbf{This study demonstrates the value of all the components in \method{}}.

\begin{figure*}[t]
  \centering
  \begin{minipage}{0.45\textwidth}
    \centering
    \resizebox{\textwidth}{!}{
    \setlength{\tabcolsep}{0.6mm}
    \begin{tabular}{cccc|cc}
      \toprule
      & \textbf{Phase}        & \textbf{Separate} & \textbf{$F_\mathrm{Pha}$}      & \multicolumn{2}{c}{\textbf{Accuracy}}       \\ \cmidrule{5-6} 
      & \textbf{Augmentation} & \textbf{Encoders} & \textbf{Residual}              & \multicolumn{1}{c|}{\textbf{WISDM}} & \textbf{GR} \\ \midrule
      1 & \cmark                   & \cmark               & \cmark                              & \multicolumn{1}{c|}{$0.86_{\pm0.02}$} & $0.70_{\pm0.01}$        \\
      2 & \xmark                   & \cmark               & \cmark                              & \multicolumn{1}{c|}{$0.81_{\pm0.01}$} & $0.61_{\pm0.01}$        \\
      3 & \cmark                   & \cmark               & \xmark ($F_\mathrm{Mag}$ Res.)       & \multicolumn{1}{c|}{$0.82_{\pm0.01}$} & $0.55_{\pm0.01}$        \\
      4 & \cmark                   & \cmark               & \xmark ($F_\mathrm{Fus}$ Res.)       & \multicolumn{1}{c|}{$0.84_{\pm0.01}$} & $0.60_{\pm0.01}$        \\
      5 & \cmark                   & \cmark               & \xmark                               & \multicolumn{1}{c|}{$0.82_{\pm0.01}$} & $0.65_{\pm0.01}$        \\
      6 & \cmark                   & \xmark (Mag Only)     & \xmark                              & \multicolumn{1}{c|}{$0.73_{\pm0.01}$} & $0.59_{\pm0.03}$        \\
      7 & \cmark                   & \xmark (Mag Only)     & \xmark ($F_\mathrm{Mag}$ Res.)      & \multicolumn{1}{c|}{$0.83_{\pm0.01}$} & $0.66_{\pm0.02}$        \\ 
      8 & \cmark                   & \xmark (Mag-Pha Concat.)     & \xmark    & \multicolumn{1}{c|}{$0.73_{\pm0.03}$} & $0.61_{\pm0.02}$        \\ 
      \bottomrule
    \end{tabular}}
    \vspace{10pt}
    \captionof{table}{\small Ablation of \method{} on WISDM and GR. The inclusion of a component is denoted as \cmark{} and exclusion as \xmark{} (modification).}
    \label{tab:ablation}
  \end{minipage}
  \hspace{0.01\textwidth}
  \begin{minipage}{0.5\textwidth}
    \centering
    \includegraphics[width=\textwidth]{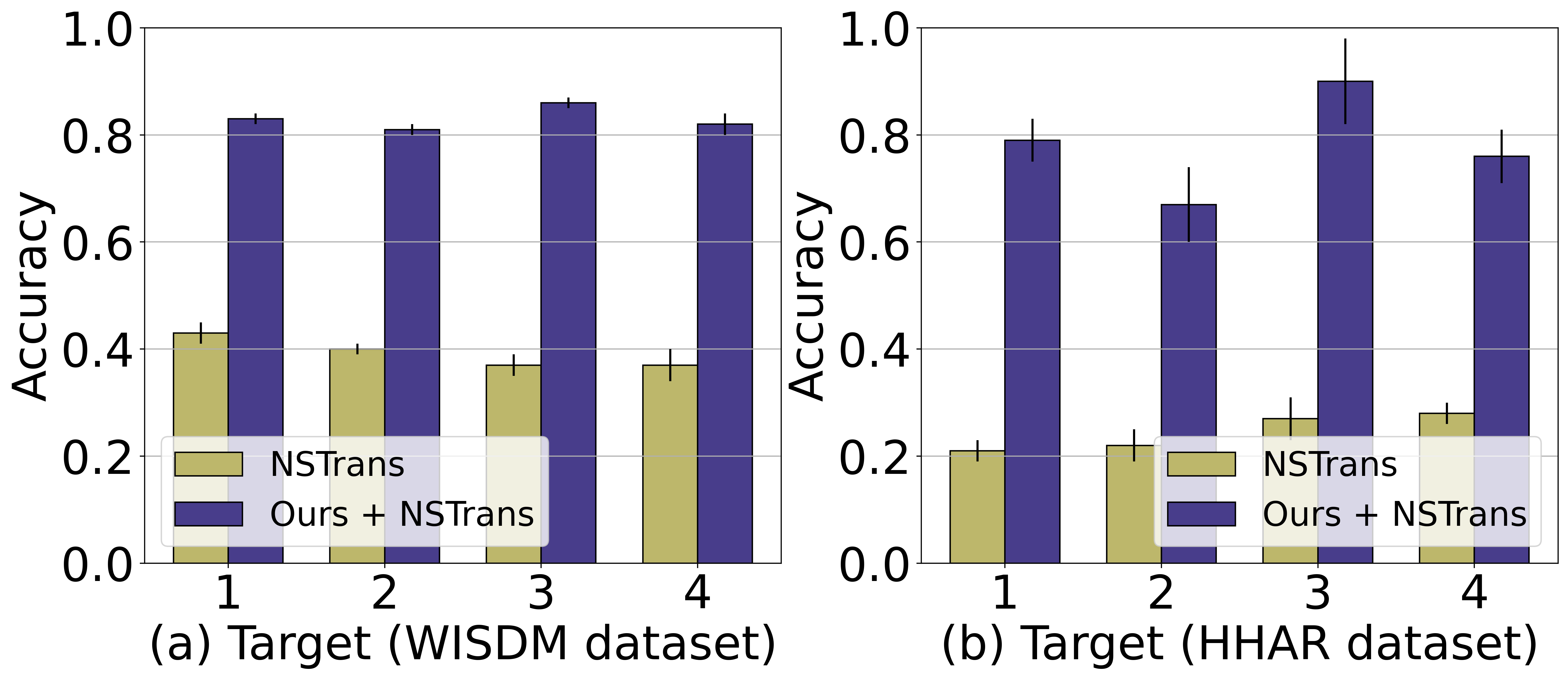} 
    \vspace{-20pt}
    \caption{\small Improvement in average cross-person generalization performance of NSTrans in (a) WISDM from 0.40 to 0.83 and (b) HHAR from 0.25 to 0.78, with our phase-driven approach.}   
    \label{fig:ablation_2}
  \end{minipage}
  \vspace{-15pt}
\end{figure*}

\noindent\textbf{General Applicability of \method{}.} We demonstrate the general applicability and flexibility of \method{} by incorporating three proposed design elements into the NSTrans model for classification: phase-based augmentation for nonstationarity diversification, separate magnitude-phase feature encoding, and phase incorporation with a residual connection. Significant performance improvements on WISDM and HHAR (Figure~\ref{fig:ablation_2}) highlight the effectiveness of these designs and the flexibility of \method{} with different backbone models. Further details are provided in Section~\ref{app:NST_ablation} of the Appendix. \tmlr{Additionally, our experiments integrating statistical modules for nonstationary time series, such as RevIN, into \method{} yield performance that is on par with or slightly below that of \method{} alone. This corroborates recent findings~\citep{chen2025fictsc} that incorporating RevIN offers no clear advantage for classification tasks. Overall, the nonstationary baselines—such as NSTrans, Koopa, and RevIN—originally designed for forecasting tasks—do not perform competitively on classification tasks (Tables~\ref{tab:har_result},~\ref{tab:one_to_x},~\ref{tab:eeg_result}), which is unsurprising~\citep{chen2025fictsc, wang2024tssurvey}; however, we include them for completeness of nonstationary time-series baselines.}

\noindent\textbf{Visualization.} We provide t-SNE visualizations of our method (\method{}), Diversify, and BCResNet on the HHAR dataset for left-out domains in scenario 1 (Figure~\ref{fig:tsne}). The plots depict out-of-domain data, with colors representing the six activity classes, showcasing \method{}'s superior separability without domain labels or target domain data. Further details are in Section~\ref{subsec:app_vis} of the Appendix.
% We present some visualizations using the t-distributed stochastic neighbor embedding (t-sne)~\citep{van2008visualizing} analyses on our \method{}, Diversify, and BCResNet for the HHAR dataset for the left-out domains in scenario 1 in Figure~\ref{fig:tsne}. We illustrate the t-sne plots of out-of-domain data and the different colors indicate the six activity classes of this dataset to highlight superior the separability achieved by \method{} on out-of-domain data without accessing the domain labels or target domain data. More details are in the Section~\ref{subsec:app_vis} of Appendix.

\begin{figure}[htbp]
% \vspace{-10pt}
\centering
\includegraphics[width=0.9\linewidth]{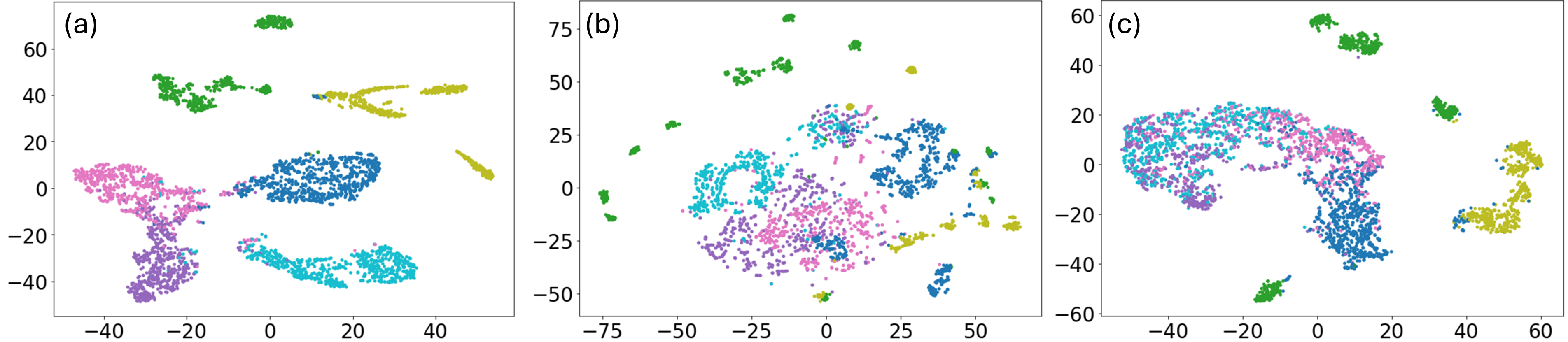}
% \vspace{-5pt}
\caption{\small t-sne visualization for (a) \method{},  (b) Diversify, and (c) BCResNet for HHAR scenario 1.}
\label{fig:tsne}
% \vspace{-10pt}
\end{figure}

\begin{wrapfigure}{r}{0.4\textwidth}
    \begin{center}
        \vspace{-15pt}
        \includegraphics[width=0.4\textwidth]{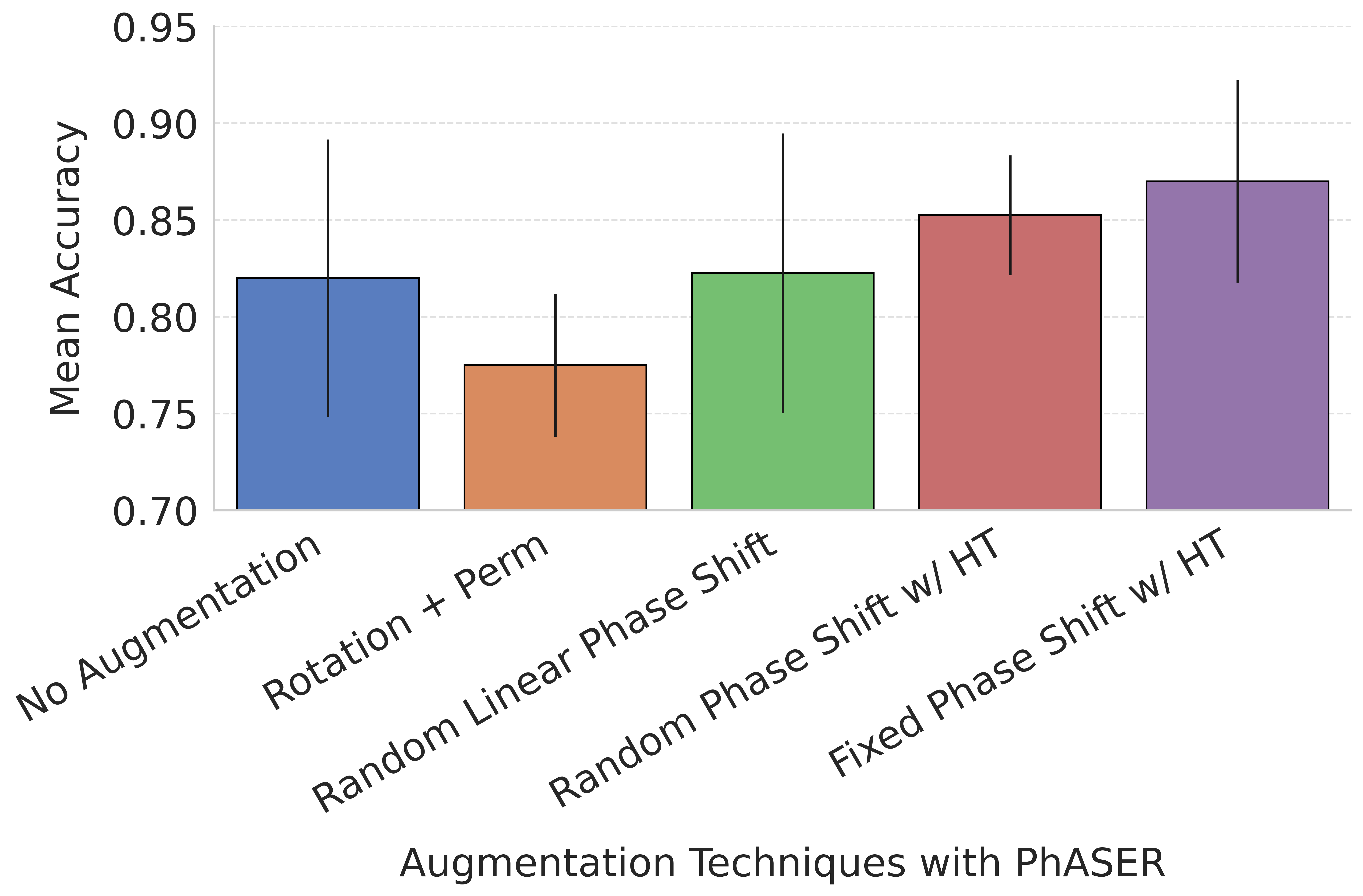}
        \vspace{-15pt}
        \caption{\small Brief comparison between different augmentation strategies with \method{}.}
        \vspace{-20pt}
        \label{fig:aug_pha}      
    \end{center}
\end{wrapfigure}
\noindent\textbf{\method{} with Other Augmentation Strategies.}
Here, we explore a random phase augmentation-variant using Hilbert Transform under certain signal periodicity assumptions (more details in Section~\ref{subsub:randHT} in the Appendix). Additionally, we adopt traditional augmentations like rotation, permutation, and circular time-shift as proposed by past works~\citep{qin2023generalizable, um2017data}; on the HHAR dataset with the \method{} framework. The results are illustrated in Figure~\ref{fig:aug_pha} and implementation details are provided in Section~\ref{subsec:app_analyses} of the Appendix. The rotation and permutation augmentations perform ~5\% worse than the no augmentation scenario in this case possibly due to semantic corruption~\citep{mintun2021interaction}. Time-shift may be viewed as a linear phase shift for a pure sinusoid (for example, for an input $\mathbf{x}(t) = sin(\omega t)$, a time-shifted version by $T$ time units is given by  $\mathbf{x}(t-T) = sin(\omega (t-T))$ which incurs a phase shift $\phi = \omega T$), however, most real-world signals are not stationary or pure tone. In such a case, a time shift introduces varied phase shifts for each frequency, and past works like~\citet{umapathy2010phase} expose the difficulty in the correct choice of a time-shift amount for retaining the signal's spectral properties of interest. This highlights the overall motivation of Hilbert Transform to provide an accurate phase shift of all frequency components by -$\pi/2$ without any explicit signal characterization. Our further exploration to induce random phase shift using HT does not show any particular advantage, hence we stick to the choice of using the fixed phase-shift augmentation followed by other phase-anchored components for domain generalization in nonstationary time-series classification tasks in the proposed \method{} framework.

\vspace{-5pt}
% We use WISDM and HHAR datasets and demonstrate significant performance enhancements obtained by a \textit{phase-driven} NST over the original NST in Fig.~\ref{fig:ablation_2} for generalizable time-series classification. Implementation specific details are provided in the Section~\ref{app:NST_ablation} in Appendix.

%% Describe other application areas where excerpts of similar ideas have been used in a different context
\section{Related Works}
\label{sec:related_works}
%\vspace{-15pt}
% Adatime
% Phase based features in image classification
% Hilbert instantaneous features in EEG
\noindent \textbf{Nonstationary Time-Series Analysis.} In real-world scenarios, nonstationary time-series data pose challenges for forecasting and classification~\citep{time_servey_1, time_classification_survey, timeseriesdecomposition}. While various solutions exist, including Bayesian models, normalization techniques, recurrent neural networks, and transformers, systematic works addressing non-stationarity's impact on time-series classification are limited~\citep{liang2005bayesian, chen2021bayesian, normalization_1, chang2021subspectral, normoalization_2, RNN_1, du2021adarnn, non_transformer_1, non_transformer_2}. \tmlr{Our study builds upon prior empirical findings and, to the best of our knowledge, is the first to investigate the impact of non‑stationarity on generalizable time‑series classification.~\citep{non_classification_1, non_classification_2, non_classification_3}.}
\vspace{-4pt}

\noindent \textbf{Domain Generalizable Learning.} 
 While domain generalizable learning is well-established in visual data~\citep{DG_survey}, applying it to time-series data poses unique challenges. Traditional approaches like data augmentation~\citep{image_augmentation_1} and domain discrepancy minimization~\citep{image_distance_1, image_distance_2} face limitations in time series due to less flexible augmentation and broader domain concepts~\citep{time_augmentation, concept_domain}. Some studies explore domain-invariant representation learning~\citep{lu2023outofdistribution, wang2023domino} and learnable data transformation~\citep{qin2023generalizable}. We highlight the non-stationarity of time series and its role in domain discrepancy, drawing on evidence from the visual domain regarding the importance of phase~\citep{kim2023domain, xu2021fourier}. A handful of works hint at phase's role in domain-invariant learning in time-series applications~\citep{lu2022domain}, and there is evidence in traditional signal processing that phase-only information is sufficient to reconstruct a signal~\citep{masuyama2023signal, jacques2020keep, jacques2021importance}. Inspired by these insights, we propose a novel phase-driven framework with an augmentation module and a phase-anchored representation learning to address non-stationarity and minimize domain discrepancy.

\tmlr{\noindent \textbf{Spectral Features for Time-Series Analysis.} Spectral representation of time series data is generally used for feature extraction~\citep{zhang2022tfad, woo2022cost, ma2024survey, yi2025survey,mohapatra2023effect} and compression~\citep{zhou2022film, rippel2015spectral}, which effectively captures the periodicity and global dependencies in the data. Commonly, Discrete Fourier Transform is used to obtain these spectral features~\citep{yang2022unsupervised, zhang2022self, wu2023timesnet}. However, for nonstationary signals where the spectral content is also time-dependent, a time-frequency representation is more suitable, and Short-term Fourier Transform~\citep{li2021units, yao2019stfnets}, Discrete Wavelet Transform~\citep{wang2018multilevel, khan2018learning}, and Empirical Mode Decomposition~\citep{cai2025ma, van2023tutorial} are used in such cases. More recently, time-series foundation models have also explored spectral representation as a tokenization scheme~\citep{Masserano2025}. While most prior work leverages the magnitude response, some works especially in audio denoising~\citep{paliwal2011importance} and beamforming applications have demonstrated the benefit of incorporating phase information. Other works on time series have shown promising phase-based augmented views for contrastive learning settings~\citep{qian2022makes, liu2023temporal, demirel2023finding}. Our work contributes to this line of investigation by demonstrating a design paradigm anchored in phase to learn generalizable time series representations.
}

\section{Limitations and Future Work}\label{sec:lim}
\method{} achieves domain generalization without explicit domain characterization or accessing target domain samples, by diversifying non-stationarity and anchoring design to signal's phase information. Our evaluation is currently limited to categorical tasks due to a scarcity of publicly available datasets with distinct domain definitions for continuous tasks like regression. Our future work aims to develop a universal representation for generalization across various tasks in dynamic conditions~\citep{mohapatra2025maestro, mohapatra24_interspeech}.
% \vspace{-6pt}

\section{Conclusion}
\label{sec:conclusion}

We address the generalization problem for nonstationary time-series classification using a phase-driven approach without accessing domain labels of source domains or samples from unseen distributions. Our approach conducts phase-based augmentation, treats time-varying magnitude and phase as separate modalities, and incorporates a phase-derived residual connection in the network. We support our design choices with rigorous theoretical and empirical evidence. Our method demonstrates significant improvement over baselines across \iclr{13} benchmarks on 5 real-world datasets.%, and our principles can be generally adapted to other works.

\section*{Acknowledgements}
We gratefully acknowledge support in part from the National Science Foundation under Grants 2038853, 2324936, and 2328973.
% \section*{Reproducibility Statement}
% All source code required to reproduce the experimental results, along with instructions for running the code, as well as the derivation of the theoretical insights, are provided in the Supplementary Materials and the Appendix respectively. We use public datasets and include implementation details in the Appendix.

\bibliography{main}
\bibliographystyle{tmlr}

\appendix
\onecolumn
\section*{\Large \centering{Appendix}}
This Appendix includes additional details for the paper\emph{``Phase-driven Domain Generalizable Learning for Nonstationary Time Series''}, including the reproducibility statement, theoretical proofs (Section~\ref{sec_proof}), additional details of \method{} (Section~\ref{app:nw_design}), detailed dataset introduction (Section~\ref{app:dataset}), implementation details (Section~\ref{app:implementation}), and detailed results (Section~\ref{app:add_results}) of main experiments. 
\appendix

\section{Theoretical Proofs}\label{app_theorem}
\label{sec_proof}
\begin{proof}[\textbf{Lemma}~\ref{lemma_bounding_beta}]
\emph{Let a set $S$ of source domains $S = \{\mathcal{S}_i\}_{i=1}^{N_S}$. A convex hull $\Lambda_S$ is considered here that consists of mixture distributions $\Lambda_S = \{\Bar{\mathcal{S}}: \Bar{\mathcal{S}}(\cdot) = \sum_{i=1}^{N_S} \pi_i \mathcal{S}_i (\cdot), \pi_i \in \Delta_{N_S - 1}\}$, where $\Delta_{N_S - 1}$ is the ($N_S \!-\! 1$)-th dimensional simplex. Let $\beta_q(\mathcal{S}_i \| \mathcal{S}_j) \leq \epsilon$ for $\forall i, j \in [N_S]$, we have the following relation for the $\beta$-Divergence between any pair of two domains $\mathcal{D}^\prime, \, \mathcal{D}^{\prime \prime} \in \Lambda_S$ in the convex hull,}
\begin{align}
\beta_q(\mathcal{D}^\prime \| \mathcal{D}^{\prime \prime}) \leq \epsilon.
\end{align}

\textbf{Proof.} Suppose two unseen domains $\mathcal{D}^\prime$ and $\mathcal{D}_{\prime \prime}$ on the convex hull $\Lambda_S$ of $N_S$ source domains with support $\Omega$. More specifically, let these two domains be $\mathcal{D}^\prime = \sum_{k=1}^{N_S}\pi_k\mathcal{S}_k(\cdot)$ and $\mathcal{D}^{\prime \prime} = \sum_{l=1}^{N_S}\pi_l\mathcal{S}_l(\cdot)$, then the $\beta$-Divergence between $\mathcal{D}^\prime$ and $\mathcal{D}^{\prime \prime}$ is
\begin{align}
\beta_q(\mathcal{D}^\prime \| \mathcal{D}^{\prime \prime}) = 2^{\frac{q-1}{q}\mathrm{RD}_q(\mathcal{D}^\prime \| \mathcal{D}^{\prime \prime})}.
\end{align}
Let us consider the part of R\'{e}nyi Divergence as follows,
\begin{equation}
\begin{aligned}
\mathrm{RD}_q(\mathcal{D}^\prime \| \mathcal{D}^{\prime \prime}) &= \frac{1}{q-1}\ln{\int_\Omega \left[\mathcal{D}^\prime(x)\right]^q \left[\mathcal{D}^{\prime \prime}(x)\right]^{1-q}dx} \\
&= \frac{1}{q-1}\ln{\int_\Omega \left[\sum_{k=1}^{N_S}\pi_k\mathcal{S}_k(x)\right]^q \left[\sum_{l=1}^{N_S}\pi_l\mathcal{S}_l(x)\right]^{1-q}dx} \\
&= \frac{1}{q-1}\ln{\int_\Omega \left[\sum_{k=1}^{N_S}\sum_{l=1}^{N_S}\pi_k\pi_l\mathcal{S}_k(x)\right]^q \left[\sum_{k=1}^{N_S}\sum_{l=1}^{N_S}\pi_k\pi_l\mathcal{S}_l(x)\right]^{1-q}dx} \\
&= \frac{1}{q-1}\ln{\sum_{k=1}^{N_S}\sum_{l=1}^{N_S}\pi_k\pi_l \int_\Omega \left[\mathcal{S}_k(x)\right]^q \left[\mathcal{S}_l(x)\right]^{1-q}dx} \\
&\leq \frac{1}{q-1}\ln{\sum_{k=1}^{N_S}\sum_{l=1}^{N_S}\pi_k\pi_l \max\limits_{k, l \in [N_S]}\int_\Omega \left[\mathcal{S}_k(x)\right]^q \left[\mathcal{S}_l(x)\right]^{1-q}dx} \\
&= \frac{1}{q-1}\ln{\max\limits_{k, l \in [N_S]}\int_\Omega \left[\mathcal{S}_k(x)\right]^q \left[\mathcal{S}_l(x)\right]^{1-q}dx}.
\end{aligned}
\end{equation}
According to the given assumption that $\beta_q(\mathcal{S}_i \| \mathcal{S}_j) \leq \epsilon$ for $\forall i, j \in [N_S]$, we have,
\begin{align}
\mathrm{RD}_q(\mathcal{D}^\prime \| \mathcal{D}^{\prime \prime}) \leq \frac{1}{q-1}\ln{\max\limits_{k, l \in [N_S]}\int_\Omega \left[\mathcal{S}_k(x)\right]^q \left[\mathcal{S}_l(x)\right]^{1-q}dx} = \max\limits_{k, l \in [N_S]}\mathrm{RD}_q(\mathcal{S}_k \| \mathcal{S}_l) \leq \frac{q}{q-1}\log_2\epsilon.
\end{align}
Thus $\beta_q(\mathcal{D}^\prime \| \mathcal{D}^{\prime \prime}) \leq \epsilon$.
\end{proof}

\begin{proof}[\textbf{Theorem}~\ref{theorem_risk}]
\emph{Let $\mathcal{H}$ be a hypothesis space built from a set of source time-series domains $S = \{\mathcal{S}_i\}_{i=1}^{N_S}$ with the same value range (i.e., the supports of these source domains are the same). Suppose $q>0$ is a constant, for any unseen time-series domain $\mathcal{D}_\mathrm{U}$ from the convex hull $\Lambda_S$, we have its closest element $\mathcal{D}_{\Bar{\mathrm{U}}}$ in $\Lambda_S$, i.e., $\mathcal{D}_{\Bar{\mathrm{U}}} = \arg \min\limits_{\pi_1, ..., \pi_{N_S}} \beta_q(\mathcal{D}_{\Bar{\mathrm{U}}} \| \sum_{i=1}^{N_S}\pi_i\mathcal{S}_i)$. Then the risk of $\mathcal{D}_\mathrm{U}$ on any $\rho$ in $\mathcal{H}$ is,
\begin{align}
R_{\mathcal{D}_\mathrm{U}}[\rho] \leq \frac{1}{2}\mathrm{d}_{\mathcal{D}_\mathrm{U}}(\rho) + \epsilon \cdot \left[\mathrm{e}_{\mathcal{D}_{\Bar{\mathrm{U}}}}(\rho) \right]^{1-\frac{1}{q}},
\end{align}
where $\mathrm{d}_{\mathcal{D}}(\rho)$ and $\mathrm{e}_{\mathcal{D}}(\rho)$ are an expected disagreement and an expected joint error of a domain $\mathcal{D}$, respectively, and they are defined as follows,
\begin{align}
\mathrm{d}_\mathcal{D}(\rho) &= \mathbb{E}_{\mathbf{x}\sim\mathcal{D}_\mathbf{x}} \mathbb{E}_{h \sim \rho} \mathbb{E}_{h^\prime \sim \rho} \mathrm{I}[h(\mathbf{x}) \neq h^\prime(\mathbf{x})], \\
\mathrm{e}_\mathcal{D}(\rho) &= \mathbb{E}_{(\mathbf{x}, y)\sim\mathcal{D}} \mathbb{E}_{h \sim \rho} \mathbb{E}_{h^\prime \sim \rho} \mathrm{I}[h(\mathbf{x}) \neq y]\mathrm{I}[h^\prime(\mathbf{x}) \neq y],
\end{align}
where $\mathrm{I}[\cdot]$ is an indicator function with $\mathrm{I}[\mathrm{True}]=1$ and $\mathrm{I}[\mathrm{False}]=0$. The $\epsilon$ in Eq.~(\ref{eq_upperbound}) is a value larger than the maximum $\beta$-Divergence in $\Lambda_S$,
\begin{align}
\epsilon &\geq \max\limits_{i, j\in[N_S], i \neq j, t\in [0, +\infty)} 2^{\frac{q-1}{q}\mathrm{RD}_q(\mathcal{S}_i(t) \| \mathcal{S}_j(t))},
\end{align}
where
\begin{equation}
\begin{aligned}
\mathrm{RD}_q(\mathcal{S}_i(t) \| \mathcal{S}_j(t)) =\frac{q(\mu_{j, t}-\mu_{i, t})^2}{2(1-q)\sigma_{i, t}^2 + 2\sigma_{j, t}^2} + \frac{\ln{\frac{\sqrt{(1-q)\sigma_{i, t}^2 + \sigma_{j, t}^2}}{\sigma_{i, t}^{1-q}\sigma_{j, t}^q}}}{1-q} 
\end{aligned}
\end{equation}}

\textbf{Proof.} According to Theorem 3 of~\citet{divergence}, if $\mathcal{H}$ is a hypothesis space, and $\mathcal{S}, \mathcal{T}$ respectively are the source and target domains. For all $\rho$ in $\mathcal{H}$,
\begin{align}
R_\mathcal{T}[\rho] \leq \frac{1}{2}\mathrm{d}_\mathcal{T}(\rho) + \beta_q(\mathcal{T} \| \mathcal{S}) \cdot \left[\mathrm{e}_\mathcal{S}(\rho) \right]^{1-\frac{1}{q}} + \eta_{\mathcal{T} \setminus \mathcal{S}},
\label{eq_standard_bound}
\end{align}
where $\eta_{\mathcal{T} \setminus \mathcal{S}}$ denotes the distribution of $(\mathbf{x}, y) \sim \mathcal{T}$ conditional to $(\mathbf{x}, y) \in \mathrm{SUPP}(\mathcal{S})$. But because it is hardly conceivable to estimate the joint error $\mathrm{e}_{\mathcal{T} \setminus \mathcal{S}}(\rho)$ without making extra assumptions,~\citet{divergence} defines the worst risk for this unknown area,
\begin{align}
\eta_{\mathcal{T} \setminus \mathcal{S}} = \mathrm{Pr}_{(\mathbf{x}, y) \sim \mathcal{T}}\left[(\mathbf{x}, y) \notin \mathrm{SUPP}(\mathcal{S}) \right] \sup_{h \in \mathcal{H}} R_{\mathcal{T} \setminus \mathcal{S}} [h].
\end{align}
In Theorem~\ref{theorem_risk}, all domains from the convex hull $\Lambda_S$ have the same value range, in other words, their supports are continuous and fully overlapped. In this case, $\mathrm{Pr}_{(\mathbf{x}, y) \sim \mathcal{T}}\left[(\mathbf{x}, y) \notin \mathrm{SUPP}(\mathcal{S}) \right] = 0$, i.e., $\eta_{\mathcal{T} \setminus \mathcal{S}} = 0$.

With Eq.~(\ref{eq_standard_bound}), if the target domain $\mathcal{T}$ is assumed as an unseen domain $\mathcal{D}_\mathrm{U}$ from the convex hull $\Lambda_S$, and we select its closest element $\mathcal{D}_{\Bar{\mathrm{U}}} = \arg \min\limits_{\pi_1, ..., \pi_{N_S}} \beta_q(\mathcal{D}_{\Bar{\mathrm{U}}} \| \sum_{i=1}^{N_S}\pi_i\mathcal{S}_i)$ and regard it as the source domain, we can derive Eq.~(\ref{eq_standard_bound}) into
\begin{align}
R_{\mathcal{D}_\mathrm{U}}[\rho] \leq \frac{1}{2}\mathrm{d}_{\mathcal{D}_\mathrm{U}}(\rho) + \beta_q(\mathcal{D}_\mathrm{U} \| \mathcal{D}_{\Bar{\mathrm{U}}}) \cdot \left[\mathrm{e}_{\mathcal{D}_{\Bar{\mathrm{U}}}}(\rho) \right]^{1-\frac{1}{q}} + 0.
\end{align}
Then according to Lemma~\ref{lemma_bounding_beta}, as both $\mathcal{D}_\mathrm{U}$ and $\mathcal{D}_{\Bar{\mathrm{U}}}$ are from the convex hull $\Lambda_S$, $\beta_q(\mathcal{D}_\mathrm{U} \| \mathcal{D}_{\Bar{\mathrm{U}}}) \leq \epsilon$. As for acquiring Eq.~(\ref{eq_rd_details}), we only need to substitute the time series domains in the form of random variable distributions into the R\'{e}nyi Divergence.

\end{proof}

\begin{theorem}[\textbf{Non-stationarity Change of Hilbert Transform}]
% \textcolor{red}{I understand you want a different series $x_i$ but the notation $x_{i,t}$ looks confusing especially because we use such two variable notations for STFT. Since Hilbert is yet another signal processing element it may look like we are denoting time-frequency. Can we instead use some other variable? I think Equation 22 should have an $x_i$ term. Can you review?} \lixu{where do we use i, t for STFT?}
\label{theorem_hilbert}
Suppose there are $M_\mathcal{D}$ samples (observations) available for a nonstationary time-series domain $\mathcal{D}_\mathbf{x}$, and each sample $\mathbf{x}_i=\{x_{i, 0}, ..., x_{i, t}, ...\}$ is characterized by its deterministic function, i.e., $\mathbf{x}_i(t) = x_{i, t} = \mathrm{x}_i(t)$, $i\in[1, M_\mathcal{D}]$. If we apply Hilbert Transformation $\mathrm{HT}(\mathbf{x}(t)) = \widehat{\mathbf{x}}(t) = \int_{-\infty}^{\infty}\mathrm{x}(\tau)\frac{1}{\pi(t-\tau)}d\tau$ to augment these time-series samples, the nonstationary statistics of augmented samples are different from the original ones,
$
\mathrm{Pr}_{\mathbf{x} \sim \widehat{\mathcal{D}}_\mathbf{x}}(\mathbf{x})(t) \neq \mathrm{Pr}_{\mathbf{x} \sim \mathcal{D}_\mathbf{x}}(\mathbf{x})(t).
$

\end{theorem}

%\vspace{-10pt}

\begin{proof}

% [\textbf{Theorem}~\ref{theorem_hilbert}]
% \emph{Suppose there are $M_\mathcal{D}$ samples (observations) available for a non-stationary time-series domain $\mathcal{D}_\mathbf{x}$, and each sample $\mathbf{x}_i=\{x_{i, 0}, ..., x_{i, t}, ...\}$ is characterized by its deterministic function, i.e., $\mathbf{x}_i(t) = x_{i, t} = \mathrm{x}_i(t), i\in[1, M_\mathcal{D}]$. If we apply Hilbert Transformation $\mathrm{HT}(\mathbf{x}(t)) = \widehat{\mathbf{x}}(t) = \int_{-\infty}^{\infty}\mathrm{x}(\tau)\frac{1}{\pi(t-\tau)}d\tau$ to augment these time-series samples, the non-stationary statistics of augmented samples are different from the original ones,
% \begin{align}
% \mathrm{Pr}_{\mathbf{x} \sim \widehat{\mathcal{D}}_\mathbf{x}}(\mathbf{x})(t) \neq \mathrm{Pr}_{\mathbf{x} \sim \mathcal{D}_\mathbf{x}}(\mathbf{x})(t).
% \end{align}}

\textbf{Proof.} According to Definition~\ref{definition_nonstationary}, the statistics of the non-stationary time-series domain consist of non-stationary mean and variance. To prove Theorem~\ref{theorem_hilbert}, we only need to prove that the mean of the time-series domain changes after applying Hilbert Transformation (HT). HT can only be conducted on deterministic signals, thus we use the empirical statistics of $M_\mathcal{D}$ samples to approximate the real statistics, 
\begin{align}
\mathbb{E}_{\mathbf{x} \sim \widehat{\mathcal{D}}_\mathbf{x}}(\mathbf{x})(t) = \sum_{i=1}^{M_\mathcal{D}} \widehat{\mathbf{x}_i}(t) = \widehat{\mu}_t, \quad \mathbb{E}_{\mathbf{x} \sim \mathcal{D}_\mathbf{x}}(\mathbf{x})(t) = \sum_{i=1}^{M_\mathcal{D}} \mathbf{x}_i(t) = \mu_t.
\end{align}
According to the standard definition of HT~\citep{king2009hilbert} and the linear property of integral operation, we have
\begin{equation}
\begin{aligned}
\mathbb{E}_{\mathbf{x} \sim \widehat{\mathcal{D}}_\mathbf{x}}(\mathbf{x})(t) = \sum_{i=1}^{M_\mathcal{D}} \widehat{\mathbf{x}_i}(t) = \sum_{i=1}^{M_\mathcal{D}} \int_{-\infty}^{\infty} \mathrm{x}_i(\tau)\frac{1}{\pi(t-\tau)}d\tau &= \int_{-\infty}^{\infty}\sum_{i=1}^{M_\mathcal{D}}\left[\mathrm{x}_i(\tau)\frac{1}{\pi(t-\tau)}d\tau\right] \\
&= \frac{1}{\pi}\int_{-\infty}^{\infty}\frac{\mu_\tau}{t-\tau}d\tau.
\end{aligned}
\label{eq_ht_mean}
\end{equation}

To interpret Eq.~(\ref{eq_ht_mean}), we can assume there is a new signal $\mathbf{s}=\{\mu_0, ..., \mu_t, ...\}$ with the deterministic function $\mu_t = \mathrm{u}(t)$, and we next apply proof by contradiction for the following proof. Suppose the non-stationary statistics of the original and HT-transformed samples are identical, i.e., $\mathbb{E}_{\mathbf{x} \sim \widehat{\mathcal{D}}_\mathbf{x}}(\mathbf{x})(t) = \mathbb{E}_{\mathbf{x} \sim \mathcal{D}_\mathbf{x}}(\mathbf{x})(t)$, we can derive the following formula,
\begin{align}
\frac{1}{\pi}\int_{-\infty}^{\infty}\frac{\mathrm{u}(\tau)}{t-\tau}d\tau = \mathrm{u}(t),
\label{eq_mean_no_change}
\end{align}
which indicates that the HT-transformed $\widehat{\mathbf{s}}$ is identical to the original $\mathbf{s}$. HT has a property called Orthogonality~\citep{king2009hilbert}: if $\mathbf{x}(t)$ is a real-valued energy signal, then $\mathbf{x}(t)$ and its HT-transformed signal $\widehat{\mathbf{x}}(t)$ are orthogonal, i.e., 
\begin{align}
\int_{-\infty}^{\infty} \mathbf{x}(t)\widehat{\mathbf{x}}(t)dt = 0.
\end{align}
To prove the property of Orthogonality, we need to use Plancherel's Formula,
\begin{theorem}[Plancherel's Formula~\citep{lang1985plancherel}]
Suppose that $u, v \in L^1(\mathbb{R}) \cap L^2(\mathbb{R})$, then
\begin{align}
\int_{-\infty}^{\infty}u(t)\overline{v(t)}dt = \frac{1}{2\pi}\int_{-\infty}^{\infty} \mathcal{F}u(\omega)\overline{\mathcal{F}v(\omega)}d\omega,
\end{align}
where $L^1(\cdot), L^2(\cdot)$ denote the $L^p$ spaces with $p=1, p=2$ respectively, $\mathbb{R}$ represents the real-valued space, and $\mathcal{F}$ denotes the Plancherel transformation. 
\end{theorem}
With Plancherel's Formula, we can prove the property of Orthogonality as follows,
\begin{equation}
\begin{aligned}
\int_{-\infty}^{\infty} \mathbf{x}(t)\widehat{\mathbf{x}}(t)dt &= \frac{1}{2\pi}\int_{-\infty}^{\infty} \mathcal{F}(\omega)(-i\,\mathrm{sgn}(\omega)\mathcal{F}(\omega))^*d\omega\\
&= \frac{i}{2\pi}\int_{-\infty}^{\infty} \mathrm{sgn}(\omega)\mathcal{F}(\omega)\mathcal{F}^*(\omega)d\omega\\
&= \frac{i}{2\pi}\int_{-\infty}^{\infty} \mathrm{sgn}(\omega)|\mathcal{F}(\omega)|^2d\omega\\
&=0,
\end{aligned}
\end{equation}
where $\mathrm{sgn}(\cdot)$ is a sign function. After proving the Orthogonality, we can use it with the condition of Eq.~(\ref{eq_mean_no_change}), i.e.,
\begin{align}
\int_{-\infty}^{\infty} \mathrm{u}(t)\widehat{\mathrm{u}}(t) dt = \int_{-\infty}^{\infty} \mathrm{u}^2(t) dt = 0.
\label{eq_contradiction}
\end{align}
Eq.~(\ref{eq_contradiction}) holds true only if $\forall t \in [0, +\infty), \mathrm{u}(t)=0$, which is contradict to our initial assumption that $\mu_t = \mathrm{u}(t)$ is not always zero in Definition~\ref{definition_nonstationary}. As a result, the assumption of $\widehat{\mu}_t = \mu_t$ is false.
\end{proof}

\textbf{Insights.} This theorem illustrates that HT does change the nonstationary statistics of time series, proving that our phase augmentation can diversify the non-stationarity of time series.

\section{Additional Details on \method{}} \label{app:nw_design}

\textbf{Augmented Dickey Fuller (ADF) Test.} This is a statistical tool to assess the non-stationarity of a given time-series signal. This test operates under a null hypothesis $\mathbb{H}_0$ where the signal has a \textit{unit-root}. The existence of \textit{unit-root} is a guarantee that the signal is non-stationary~\citep{said1984testing}. To reject $\mathbb{H}_0$, the statistic value of the ADF test should be less than the critical values associated with a significance level of $0.05$ (denoted by $p$, the probability of observing such a test statistic under the null hypothesis). Throughout the paper, for multivariate time series, the average ADF statistics across all variates are reported. Besides, since this is a statistical tool to evaluate non-stationarity for each instance of time-series data, we provide an average of this number across a dataset to give the reader a view of the degree of non-stationarity. 

\noindent \textbf{Phase Augmentation.} \label{app:phase}
In this work, we are particularly interested in learning representations robust to temporal distribution shifts. Incorporating a phase shift in a signal is a less-studied augmentation technique. One of the main challenges is that real-world signals are not composed of a single frequency component and accurately estimating and controlling the shifting of the phase while retaining the magnitude spectrum of a signal is difficult. To solve this, we leverage the analytic transformation of a signal using the Hilbert Transform. The key advantages of this technique are maintaining global temporal dependencies and magnitude spectrum, no exploration of design parameters and being extendible to non-stationary and periodic time series. 

% Look at this definition for generalisation --> https://wiki.seg.org/wiki/Phase_and_the_Hilbert_transform (Equation 3)
Lets walk through a simple example for a signal, $\mathbf{x}(t) = 2cos(w_0 t)$ which can be written in the polar coordinates as $\mathbf{x}(t) = e^{iw_0 t} + e^{-iw_0 t}$. Applying the HT conditions from Equation~\ref{eq:ht}, $\mathrm{HT}(\mathbf{x}(t)) = 2sin(w_0 t)$. Essentially, HT shifts the signal by $\pi / 2$ radians. We conduct this instance-level augmentation for each variate of the time series input. The aim is to diversify the phase representation. We use the \textit{scipy}~\citep{2020SciPy-NMeth} library to implement this augmentation.

\noindent \textbf{STFT Specifications.} Non-stationary signals contain time-varying spectral properties. We use STFT to capture these magnitude and phase responses in both time and frequency domains. There are three main arguments to compute STFT - length of each segment (characterized by the window size and the ratio for overlap), the number of frequency bins, and the sampling rate. We use the scipy library to implement this operation and use a $k < 1$ as a multiplier to the length of the window $W$ to give the segment length as $k \times W$ with no overlap between segments. The complete list of STFT specifications is given in Table~\ref{tab:stft_specs}. We also demonstrate a sensitivity analysis concerning the number of frequency bins and the segment length in Figure~\ref{fig:app_sns_stft}. 

\begin{figure}[!htb]
\centering
\includegraphics[width=0.7\linewidth]{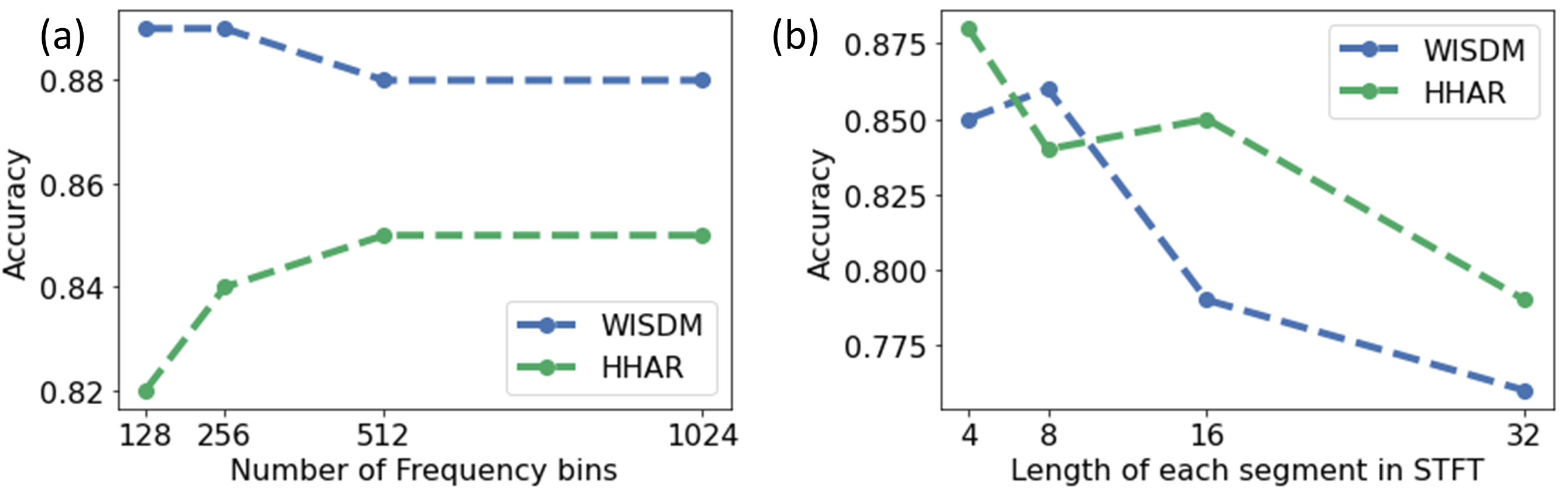}
\caption{Illustration of the sensitivity of performance to the design choices of STFT by varying a) the number of frequency bins with a fixed segment length of 4 and b) by varying the segment lengths with a 1024 frequency bins.}
\label{fig:app_sns_stft}
\end{figure}

\begin{table}[htbp]
\caption{Arguments for STFT computation}
\centering
\resizebox{\textwidth}{!}{
\begin{tabular}{ccccc}
\hline
Dataset & Sampling Rate & Sequence Length & STFT   segment length & Number of frequency bins \\ \hline
WISDM   & 20 Hz         & 128             & 4                     & 1024                     \\
HHAR    & 100 Hz        & 128             & 4                     & 1024                     \\
UCIHAR  & 50 Hz         & 128             & 4                     & 1024                     \\
SSC     & 100 Hz        & 3000            & 16                    & 1024                     \\
GR      & 200 Hz        & 200             & 4                     & 1024                     \\ \hline
\end{tabular}}
\label{tab:stft_specs}
\end{table}

\noindent \textit{Note:} It is tempting to use an empirical mode transformation and then apply a Hilbert-Huang transformation to obtain an instantaneous phase and amplitude response in the case of non-stationary signals. It absolves us from a finite time-frequency resolution for the STFT spectra. However, our initial results indicate a high dependence on the choice of the number of intrinsic mode functions~\citep{huang2014hilbert} for signal decomposition. Hence, for a generalizable approach, we choose STFT as the tool for the time-frequency spectrum. 

\tmlr{
\noindent \textbf{Additional analyses with wavelet transform.} We present a comparison of Discrete Wavelet Transform (DWT) based analyses of the WISDM dataset with STFT within the \method{} architecture in Table~\ref{tab:dwt}.

\begin{table}[h]
\centering
\tmlr{
\caption{Accuracy and standard deviation for different methods.}
\begin{tabular}{l c}
\toprule
\textbf{Method} & \textbf{Accuracy ± Std Dev} \\
\midrule
STFT & 0.87 ± 0.01 \\
DWT  & 0.82 ± 0.01 \\
\bottomrule
\end{tabular}
}

\label{tab:dwt}
\end{table}

We also present the spectrograms obtained using STFT, DWT, and EMD for a sample from the WISDM dataset in Figure~\ref{fig:spec}.

\begin{figure}[h]
    \centering
    \includegraphics[width=0.8\linewidth]{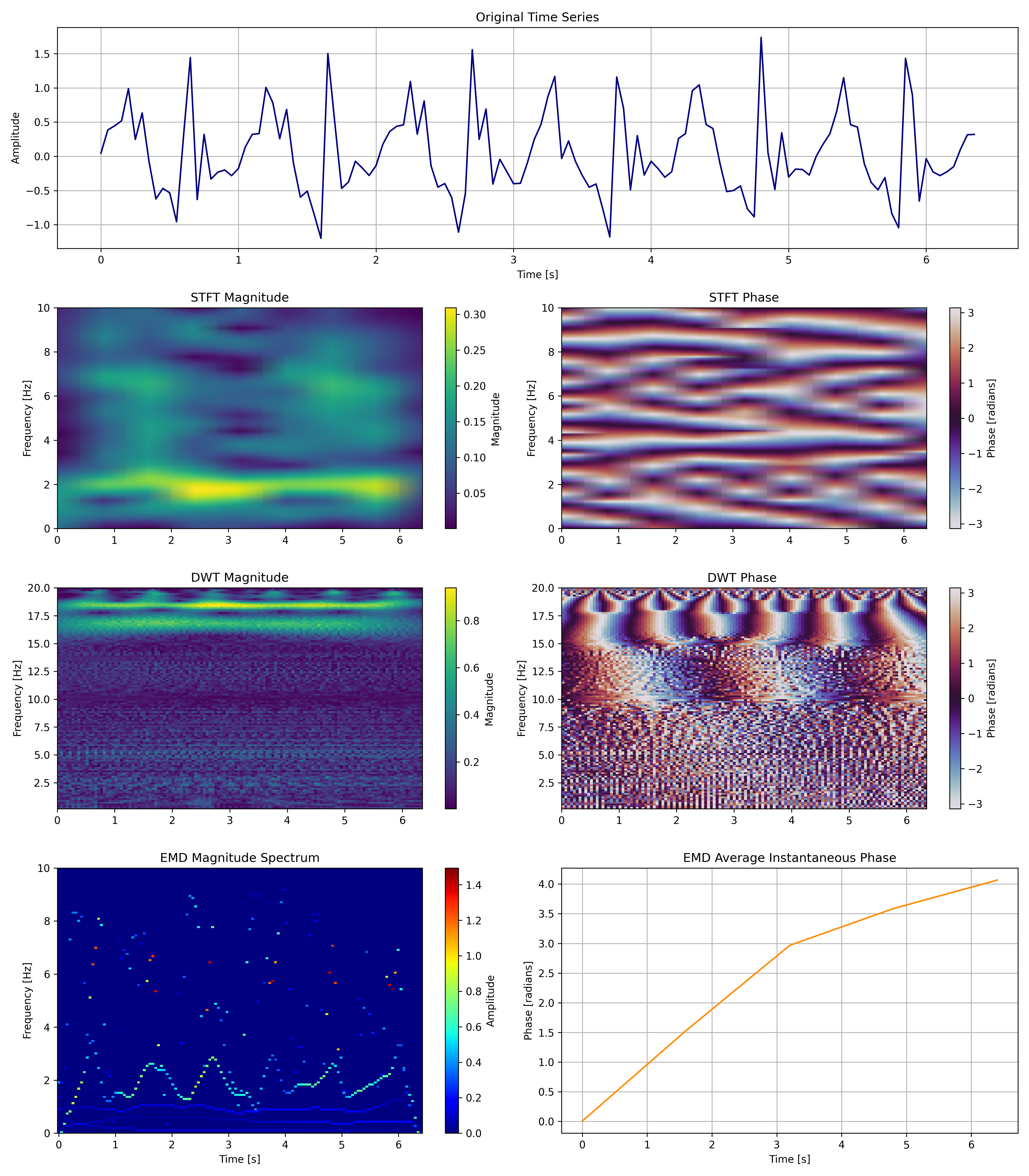}
    \caption{\tmlr{Comparison of STFT, DWT, and EMD-based frequency domain transformations of a time series sample from WISDM.}}
    \label{fig:spec}
\end{figure}
}

\noindent \textbf{Backbones for Temporal Encoder.}\label{app:F_temp}
The choice of temporal encoder, $F_\mathrm{Tem}$, is not central to our design. Table~\ref{tab:backbone} demonstrates the performance of \method{} under the identical settings for four cross-person settings using WISDM datasets using different backbones for $F_\mathrm{Tem}$. For the convolution-based self-attention (second row in Table~\ref{tab:backbone}) we use three encoders to compute query ($W_q$), key ($W_k$), and value($V)$ matrices for $\mathbf{r}_\mathrm{Dep}$ following the guidelines from~\citet{vaswani2017attention}. Then we compute self-attention as, $A = softmax\Biggl(\frac{QK^T}{\sqrt{d_k}}\Biggl)V$, where $d_k$ is the temporal dimension of $\mathbf{r}_\mathrm{Dep}$. Subsequently, we use $\hat{\mathbf{r}}_\mathrm{Dep} = \mathbf{r}_\mathrm{Dep} + A$, as the input to $F_\mathrm{Tem}$. For more details on the convolution and transformer backbones refer to Section~\ref{app:imp}.

\begin{table}[htbp]
\centering
\caption{Results for 4 different cross-person settings for WISDM dataset.}
\begin{tabular}{c|cccc}
\toprule
Backbones for $F_\mathrm{Tem}$ & 1    & 2    & 3    & 4    \\ \hline
2D Convolution based                   & 0.86 & 0.85 & 0.86 & 0.84 \\
2D Convolution based with self-attention  &  0.88	& 0.83	& 0.84	& 0.81 \\
Transformer           & 0.87 & 0.84 & 0.87 & 0.84 \\
\bottomrule
\end{tabular}
\label{tab:backbone}
\end{table}

\section{Dataset Details} \label{app:dataset}
Past works~\citep{gagnon2022woods, ragab2023adatime} have shown that the datasets used in our work suffer from a distribution shift across users and also within the same user temporally. This makes them suitable for evaluating the efficacy of our framework. In this section, we provide more details on the datasets. Table~\ref{tab:app_dataset} summarizes the average ADF statistics of the datasets along with their variates and their number of classes and domains. 
\begin{table}[htbp]
\centering
\caption{Summary of the dataset attributes. Higher value of ADF stat indicates greater non-stationarity within a signal.}
\resizebox{\textwidth}{!}{
\begin{tabular}{|c|c|c|c|c|c|}
\hline
\textbf{Category}                                  & \textbf{Dataset}               & \textbf{\begin{tabular}[c]{@{}c@{}}Representative ADF-Statistic \\ (mean across all variates)\end{tabular}} & \textbf{Variates} & \textbf{Domains} & \textbf{Classes} \\ \hline
\cellcolor[HTML]{FFFFFF}Human Activity recognition & \cellcolor[HTML]{FFFFFF}UCIHAR & -2.58  \iclr{(0.044)}                                                                                                     & 9                 & 31               & 6                \\ \hline
\cellcolor[HTML]{FFFFFF}Human Activity recognition & \cellcolor[HTML]{FFFFFF}HHAR   & -1.74  \iclr{(0.062)}                                                                                                    & 3                 & 9                & 6                \\ \hline
\cellcolor[HTML]{FFFFFF}Human Activity recognition & \cellcolor[HTML]{FFFFFF}WISDM  & -0.78 \iclr{(0.051)}                                                                                                      & 3                 & 36               & 6                \\ \hline
\cellcolor[HTML]{FFFFFF}Gesture Recognition        & \cellcolor[HTML]{FFFFFF}EMG    & -33.14   \iclr{(0.011)}                                                                                                   & 8                 & 36               & 6                \\ \hline
\cellcolor[HTML]{FFFFFF}Sleep Stage Classification & \cellcolor[HTML]{FFFFFF}EEG    & -3.7 \iclr{(0.047)}                                                                                                        & 1                 & 20               & 5                \\ \hline
\end{tabular}}
\label{tab:app_dataset} 
\end{table}

% Preprocessing || Cross-person splits || Class labels || Data specifics is any
\noindent \textbf{WISDM}~\citep{kwapisz2011activity}: It originally consists of 51 subjects performing 18 activities but we follow the ADATime~\citep{ragab2023adatime} suite to utilize 36 subjects comprising of 6 activity classes given as walking, climbing upstairs, climbing downstairs, sitting, standing, and lying down. The dataset consists of 3-axis accelerometer measurements sampled at 20 Hz to predict the activity of each participant for a segment of 128-time steps. According to~\citet{ragab2023adatime}, this is the most challenging dataset suffering from the highest degree of class imbalance.

\noindent \textbf{HHAR}~\citep{stisen2015smart}: To remain consistent with the existing AdaTime benchmark we leverage the Samsung Galaxy recordings of this dataset from 9 participants from a 3-axis accelerometer sampled at 100 Hz. The 6 activity classes, in this case, are - biking, sitting, standing, walking, climbing up the stairs, and climbing down the stairs.

\noindent \textbf{UCIHAR}~\citep{bulbul2018human}: This dataset is collected from 30 participants using 9-axis inertial motion unit using a waist-mounted cellular device sampled at 50 Hz. The six activity classes are the same as WISDM dataset.

\noindent \textbf{SSC}~\citep{goldberger2000physiobank}: This is a single channel EEG dataset collected from 20 subjects to classify five sleep stages - wake, non-rapid eye movement stages - N1, N2, N3, and rapid-eye-movement.

\noindent \textbf{GR}~\citep{lobov2018latent}: For surface-EMG based gesture recognition we follow~\citet{lu2023outofdistribution}'s preprocessing and use an 8-channel data recorded from 36 participants for six types of gestures sampled at 200 Hz. Note, that this is the least stationary dataset (see Table~\ref{tab:app_dataset}, yet \method{} performs as well as or better than the stat-of-the-art techniques as shown in Table~\ref{tab:eeg_result} in the main paper.

\section{Implementation Details} \label{app:implementation}
All experiments are performed on an Ubuntu OS server equipped with NVIDIA TITAN RTX GPU cards using PyTorch framework. Every experiment is carried out with 3 different seeds (2711, 2712, 2713). During model training, we use Adam optimizer~\citep{kingma2020method} with a learning rate from 1e-5 to 1e-3 and maximum number of epochs is set to 150 based on the suitability of each setting. We tune these optimization-related hyperparameters for each setting and save the best model checkpoint based on early exit based on the minimum value of the loss function achieved on the validation set.

\subsection{Dataset Configuration}\label{app:subsubdataset}
There is no standard benchmarking for domain generalization for time-series where the domain labels and target samples are inaccessible. We leverage past works of~\citet{ragab2023adatime, lu2023outofdistribution} for preprocessing steps. For each dataset, we use a cross-person setting in four scenarios. The details of the target domains chosen in each scenario are given in Table~\ref{tab:app_scenario}, the rest are used as source domains. Note for GR we use the same splits as~\citet{lu2023outofdistribution}. Our method is not influenced by domain labels as we do not require them for our optimization.

\begin{table}[htbp]
\centering
\caption{Target domain splits for 4 scenarios of each dataset.}
\begin{tabular}{ccccc}
\hline
\begin{tabular}[c]{@{}c@{}}Target \\ Domains\end{tabular} & Scenario 1 & Scenario 2 & Scenario 3 & Scenario 4 \\ \hline
WISDM                                                     & 0-9        & 10-17      & 18-27      & 28-35      \\
HHAR                                                      & 0,1        & 2,3        & 4,5        & 6-8        \\
UCIHAR                                                    & 0-7        & 8-15       & 16-23      & 24-29      \\
GR                                                        & 0-8        & 9-17       & 18-26      & 27-35      \\
SSC                                                       & 0-5        & 5-9        & 10-14      & 15-20      \\ \hline
\end{tabular}
\label{tab:app_scenario}
\end{table}

Figure~\ref{app:data_dist} illustrates the class distribution for each dataset. Only the WISDM and Sleep Stage Classification (SSC) datasets exhibit notable imbalances among certain classes. To validate the consistency of our conclusions, we compare the Area Under the Curve (AUC) with the adopted accuracy metric in Figure~\ref{app:auc_roc}. Generally, past works~\citep{lu2023outofdistribution, gagnon2022woods}, utilizing these datasets have adopted accuracy as the primary performance metric, and we follow the same approach.

\begin{figure}
    \centering
    \includegraphics[width=\textwidth]{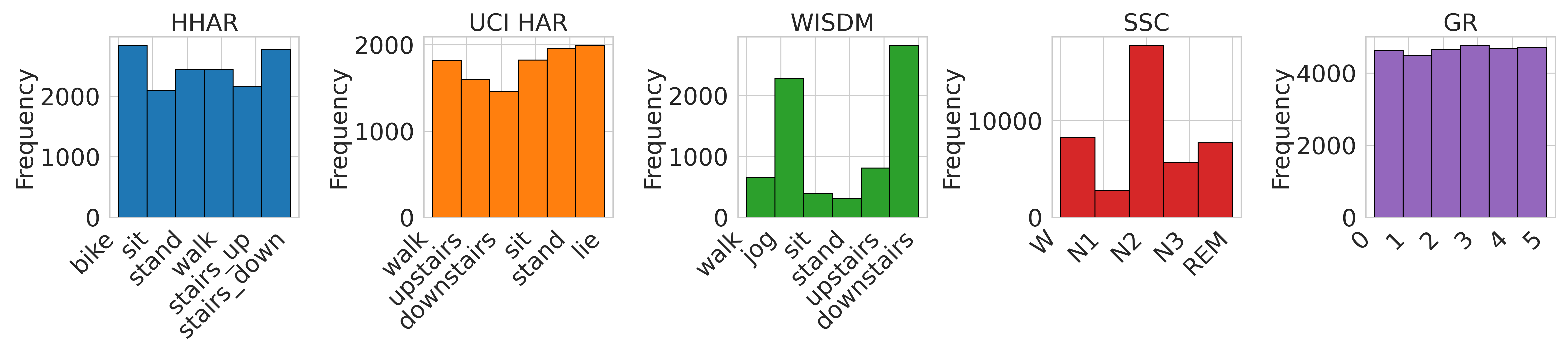}
    \vspace{-25pt}
    \caption{Class Distributions of the datasets used for evaluation.}
    \vspace{-10pt}
\end{figure}\label{app:data_dist}

\begin{figure}
    \centering
    \includegraphics[scale=0.4]{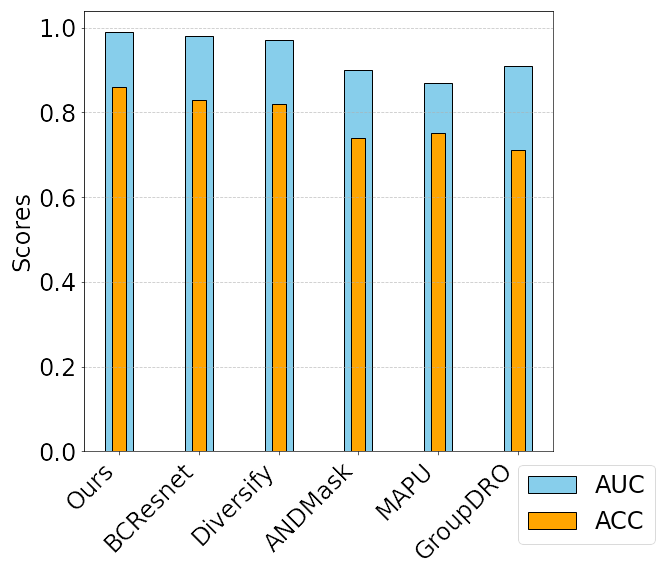}
    \caption{Illustration of additional performance metric, Area Under the ROC Curve (AUC), along with Accuracy—for Scenario 1 of the WISDM dataset, for the top-performing baselines. These metrics demonstrate consistency and justify our choice of accuracy as the primary evaluation metric.}
\end{figure}\label{app:auc_roc}

\subsection{Baseline Methods}\label{app:baseline}

\noindent \textbf{General Domain Generalization Methods.}
For all the standard domain generalization baselines we use conv2D layers for feature transformation of multivariate time series. It is worth mentioning that DANN is actually a domain adaptation study, which requires access to certain unlabeled target domain data. For cross-person generalization, the source domain consists of data from multiple people, in which we divide the source domain data into two parts with equal size and view one of them as the target domain to leverage DANN for domain-invariant training. As for one-person-to-another cases, we randomly sample a small number of unlabeled instances from each target person and merge them into the target set that is needed for running DANN.

\noindent \textbf{BCResNet.} This is a competitive benchmark for several audio-scene recognition challenges and demonstrates many useful techniques for domain generalization. BCResNet originally required mel-frequency-cepstral-coefficients but it is not suitable for time-series, hence, we use standard STFT of the multivariate-time series as input in this case. 

\noindent \textbf{Non-Stationary Transformer and Koopa.}
These are forecasting baselines that particularly address non-stationarity in short-term time sequences, Non-stationary transformer (NSTrans)~\citep{non_transformer_1} and Koopa~\citep{liu2024koopa}. To adapt it to our setting we use the encoder part of NSTrans followed by a classification head composed of fully connected layers. We simply average the encoder's output from all time steps and feed it to this classifier head.

\noindent \textbf{Ours+RevIN.} Further, we demonstrate that statistical techniques like Reversible Instance Normalization (RevIN)~\citep{kim2021reversible} may be used as a plug-and-play module with our framework. One limitation of using RevIN is that the input and output dimensions of this module must have the same dimensions to de-normalize the instance in the feature space. This may limit the usability of the module, however, we find that applying this module around the fusion encoder specifying the same number of input and output channels in the 2D convolution layer is suitable. We do not observe any significant benefit of incorporating this module from the experiments, however, if an application can specifically benefit from such RevIN, \method{} framework can support it.
% Using RevIN after average pooling gives very very low accuracy – 0.46 on WISDM. So we are using it around the fusion encoder and it give good performance.

\noindent \textbf{Diversify.}
The goal of this design is to characterize the latent domains and use a proxy-training schema to assign pseudo-domain labels to the samples to learn generalizable representations. It is an end-to-end version of the adaptive RNN~\citep{du2021adarnn} method which also proposes to identify sub-domains within a domain for generalization. It is interesting to note that for time-series generalizable representation viewing the non-stationarity or intra-domain shifts is crucial. Both diversify and \method{} address this problem from completely different approaches and demonstrate improvement over other standard methods or even domain adaptation methods that have the advantage of accessing samples from unseen distributions. While diversify aims to characterize latent distributions and uses a parametric setting, \method{} forces the model to learn domain-invariant features by anchoring the design to the phase which is intricately tied to non-stationarity. It also highlights that time-series domain generalization is a unique problem (compared to the more popular visual domain) and dedicated frameworks need to be designed in this case.

\noindent \textbf{MAPU.} MAPU is the state-of-the-art source-free domain adaptation study for time series, thus, in fact, it does not apply to the time-series domain generalizable learning problem. However, we still view it as an effective approach that can address distribution shifts and achieve domain-invariant learning. In our implementation, in addition to the source domain data, we still provide MAPU with the unlabeled target domain data for both cross-person generalization and one-person-to-another cases. The training procedure is identical to the default MAPU design, which is to pre-train the model on labeled source domain data and then conduct the training on unlabeled target domain data.

\noindent \textbf{Chronos.} Large foundation models are a sought-after approach in many domains and Chronos is one such most recent candidate for time-series. It is trained on 42 datasets and presents impressive zero-shot and few-shot abilities. Although it is largely targeted as a forecasting tool, the authors indicate its universal representation ability for a variety of tasks. Four variants of Chronos model checkpoints are available ranging from 20M to 70M parameters and embedding sizes from 256 to 1024. Based on pilot testing with scenario 1 on WISDM dataset (accuracies with a 1M parameter downstream model for the three variants: tiny-0.65, base-0.41, large-0.36), we find that the smallest version of the model, Chronos-tiny best suits our conservative dataset sizes for downstream fine-tuning. We use a few layers of 2D convolution layers with max-pooling to reduce the feature size which is dependent of the length of the sequence and then flatten and input to fully-connected layers as our downstream model. 

\noindent\textit{Note:} A few works~\citep{jin2024position, liu2023biosignal} use large language models directly to analyze raw time-series despite the obvious modality gap and can report comparable performance. However, our preliminary testing with ChatGPT~\citep{radford2019language} with in-context-learning by prompting similar to Jin et. al~\citep{jin2024position} using the HHAR dataset does not provide satisfactory results and we do not pursue that direction. Instead, we use a domain-specific large foundation model like Chronos as a fair baseline. 

\begin{table}[h]
\centering
\caption{Complete set of results from three trials on each baseline for WISDM cross-person generalization setting.}
\begin{tabular}{|c|cc|cc|cc|cc|}
\hline
Baselines       & \multicolumn{2}{c|}{Scenario 1}                     & \multicolumn{2}{c|}{Scenario 2}                     & \multicolumn{2}{c|}{Scenario 3}                     & \multicolumn{2}{c|}{Scenario   4}                   \\ \cline{2-9} 
                & \multicolumn{1}{c}{Mean} & \multicolumn{1}{c|}{Std} & \multicolumn{1}{c}{Mean} & \multicolumn{1}{c|}{Std} & \multicolumn{1}{c}{Mean} & \multicolumn{1}{c|}{Std} & \multicolumn{1}{c}{Mean} & \multicolumn{1}{c|}{Std} \\ \hline
ERM             & \multicolumn{1}{c}{0.57}     & \multicolumn{1}{c|}{0.02}    & \multicolumn{1}{c}{0.50}     & \multicolumn{1}{c|}{0.02}    & \multicolumn{1}{c}{0.51}     & \multicolumn{1}{c|}{0.02}    & \multicolumn{1}{c}{0.55}     & \multicolumn{1}{c|}{0.02}    \\
GroupDRO        & 0.71                     & 0.06                     & 0.67                     & 0.06                     & 0.60                     & 0.07                     & 0.67                     & 0.04                     \\
DANN            & 0.71                     & 0.02                     & 0.65                     & 0.01                     & 0.65                     & 0.06                     & 0.70                     & 0.03                     \\
RSC             & 0.69                     & 0.05                     & 0.71                     & 0.07                     & 0.64                     & 0.10                     & 0.61                     & 0.11                     \\
ANDMask         & 0.74                     & 0.01                     & 0.73                     & 0.03                     & 0.69                     & 0.06                     & 0.69                     & 0.03                     \\
InceptionTime   & 0.83                     & 0.01                     & 0.82                     & 0.02                     & 0.80                     & 0.04                     & 0.77                     & 0.01                     \\
BCResNet        & 0.83                     & 0.00                     & 0.79                     & 0.04                     & 0.75                     & 0.04                     & 0.78                     & 0.04                     \\
NSTrans             & 0.43                     & 0.02                     & 0.40                     & 0.01                     & 0.37                     & 0.02                     & 0.37                     & 0.03                     \\
Koopa &0.63&0.02&0.61&0.04&0.72&0.03&0.57&0.01 \\
MAPU            & 0.75                     & 0.02                     & 0.69                     & 0.04                     & 0.79                     & 0.06                     & 0.79                     & 0.03                     \\
Diversify       & 0.82                     & 0.01                     & 0.82                     & 0.01                     & 0.84                     & 0.01                     & 0.81                     & 0.01                     \\ 
Chronos         & 0.71                     & 0.01                     & 0.67                     & 0.01                     & 0.65                     & 0.01                     & 0.62                     & 0.01                     \\ \hline
Ours +   RevIN* & 0.86                     & 0.01                     & 0.85                     & 0.01                     & 0.84                     & 0                     & 0.84                     & 0.03                     \\ 
Ours            & 0.86                     & 0.01                     & 0.85                     & 0.01                     & 0.85                     & 0.01                     & 0.82                     & 0.02                     \\\hline
\end{tabular}
\label{tab:app_wisdm}
\vspace{-16pt}
\end{table}

\begin{table}[h]
\centering
\caption{Complete set of results from three trials on each baseline for HHAR cross-person generalization setting.}
\begin{tabular}{|c|cc|cc|cc|cc|}
\hline
Baselines       & \multicolumn{2}{c|}{Scenario 1}                     & \multicolumn{2}{c|}{Scenario 2}                     & \multicolumn{2}{c|}{Scenario 3}                     & \multicolumn{2}{c|}{Scenario   4}                   \\ \cline{2-9} 
                & \multicolumn{1}{c}{Mean} & \multicolumn{1}{c|}{Std} & \multicolumn{1}{c}{Mean} & \multicolumn{1}{c|}{Std} & \multicolumn{1}{c}{Mean} & \multicolumn{1}{c|}{Std} & \multicolumn{1}{c}{Mean} & \multicolumn{1}{c|}{Std} \\ \hline
ERM             & \multicolumn{1}{c}{0.49}     & \multicolumn{1}{l|}{0.05}    & \multicolumn{1}{c}{0.46}     & \multicolumn{1}{c|}{0.01}    & \multicolumn{1}{c}{0.45}     & \multicolumn{1}{c|}{0.02}    & \multicolumn{1}{c}{0.47}     & \multicolumn{1}{c|}{0.03}    \\
GroupDRO        & 0.60                     & 0.01                     & 0.53                     & 0.02                     & 0.59                     & 0.02                     & 0.64                     & 0.03                     \\
DANN            & \multicolumn{1}{c}{0.66}     & \multicolumn{1}{l|}{0.01}    & \multicolumn{1}{c}{0.71}     & \multicolumn{1}{l|}{0.01}    & \multicolumn{1}{c}{0.67}     & \multicolumn{1}{l|}{0.09}    & \multicolumn{1}{c}{0.69}     & \multicolumn{1}{l|}{0.03}    \\
RSC             & 0.52                     & 0.05                     & 0.49                     & 0.04                     & 0.44                     & 0.03                     & 0.47                     & 0.03                     \\
ANDMask         & 0.63                     & 0.02                     & 0.64                     & 0.06                     & 0.66                     & 0.11                     & 0.69                     & 0.05                     \\
InceptionTime   & 0.77                     & 0.04                     & 0.80                     & 0.01                     & 0.82                     & 0.03                     & 0.83                     & 0.01                     \\
BCResNet        & 0.66                     & 0.05                     & 0.70                     & 0.06                     & 0.75                     & 0.04                     & 0.68                     & 0.04                     \\
NSTrans             & 0.21                     & 0.02                     & 0.22                     & 0.03                     & 0.27                     & 0.04                     & 0.28                     & 0.02                     \\
Koopa &0.72&0.04&0.63&0.03&0.72&0.05&0.69&0.02 \\
MAPU            & 0.73                     & 0.02                     & 0.72                     & 0.03                     & 0.81                     & 0.01                     & 0.78                     & 0.03                     \\
Diversify       & 0.82                     & 0.01                     & 0.76                     & 0.01                     & 0.82                     & 0.01                     & 0.68                     & 0.01                     \\ 
Chronos       & 0.73                     & 0.04                     & 0.75                     & 0.03                     & 0.73                     & 0.01                     & 0.66                     & 0.12                     \\
\hline
Ours +   RevIN* & 0.82                     & 0.05                     & 0.82                     & 0.02                     & 0.92                     & 0.04                     & 0.85                     & 0.03                     \\
Ours            & 0.83                     & 0.02                     & 0.83                     & 0.02                     & 0.94                     & 0.03                     & 0.88                     & 0.02                     \\ \hline
\end{tabular}
\label{tab:app_hhar}
\end{table}

\begin{table}[h]
\centering
\caption{Complete set of results from three trials on each baseline for UCIHAR cross-person generalization setting.}
\begin{tabular}{|c|cc|cc|cc|cc|}
\hline
Baselines       & \multicolumn{2}{c|}{Scenario 1}                     & \multicolumn{2}{c|}{Scenario 2}                     & \multicolumn{2}{c|}{Scenario 3}                     & \multicolumn{2}{c|}{Scenario   4}                   \\ \cline{2-9} 
                & \multicolumn{1}{c}{Mean} & \multicolumn{1}{c|}{Std} & \multicolumn{1}{c}{Mean} & \multicolumn{1}{c|}{Std} & \multicolumn{1}{c}{Mean} & \multicolumn{1}{c|}{Std} & \multicolumn{1}{c}{Mean} & \multicolumn{1}{c|}{Std} \\ \hline
ERM    &0.72&0.09&0.64&0.05&0.70&0.01&0.72&0.03    \\
GroupDRO       &0.91&0.02&0.84&0.01&0.89&0.04&0.85&0.07                     \\
DANN            &0.84&0.02&0.79&0.01&0.81&0.02&0.86&0.03    \\
RSC             &0.82&0.13&0.73&0.07&0.74&0.03&0.81&0.06                     \\
ANDMask         &0.86&0.08&0.80&0.06&0.76&0.13&0.78&0.09                     \\
InceptionTime   &0.91&0.03&0.82&0.07&0.88&0.02&0.91&0.04                     \\
BCResNet        &0.81&0.02&0.77&0.02&0.78&0.02&0.83&0.02                     \\
NSTrans             &0.35&0.02&0.35&0.01&0.51&0.02&0.47&0.01                    \\
Koopa           &0.81&0.02&0.72&0.05&0.81&0.06&0.77&0.03 \\
MAPU            &0.85&0.03&0.80&0.01&0.85&0.02&0.82&0.03                     \\
Diversify       &0.89&0.03&0.84&0.04&0.93&0.02&0.90&0.02                     \\ 
Chronos       &0.56&0.05&0.57&0.01&0.50&0.02&0.82&0.13                     \\ \hline
Ours + RevIN*  &0.96&0.01&0.90&0.01&0.93&0.03&0.97&0.01                     \\
Ours            &0.96&0.01&0.91&0.01&0.95&0&0.97&0.01                     \\ \hline
\end{tabular}
\label{tab:app_ucihar}
\vspace{-16pt}
\end{table}

\subsection{Implementation Details of \method{}}\label{app:imp}
The magnitude and phase encoders, $F_\mathrm{Mag}$ and $F_\mathrm{Pha}$ are implemented using 2D convolution layers with the number of input channels equal to the variates, $V$, and the out channels as $2c$ with $(5 \times 5)$ kernels. $c$ is a hyperparameter used to conveniently control the size of the overall network. For all HAR and GR models we adopt $c$ as 1 and for SSC $c$ is 4. For more specific details please refer to our code. The sub-spectral feature normalization uses a group number of 3 and follows Equation~\ref{eq:ssn} for operation. This is inspired by Chang et. al~\citep{chang2021subspectral} subspectral normalization for audio applications with a frequency spectrum input. The key idea is to conduct sub-band normalization (across a fixed set of frequency bins along time and examples for each channel). We find merit in using this technique for domain generalizable applications, as it can help overcome the low-frequency drifts arising due to device differences (for eg. DC drifts in various sensors). One implementation-specific modification we carried out to ensure a generalizable framework is that if the number of sub-bands is not divisible by the total number of features then we choose to apply the remainder bands with batch-normalization. The output from the respective encoders is then fused along the channel/variate axis by multiplying with 2D convolution kernels to provide a new feature map which is the input to our phase-driven residual network. The $F_\mathrm{Fus}$ similarly is implemented using 2D convolution layers with the number of input channels as $4c$ and output channels to be $2c$.

Subsequently for the depth-wise encoder, $F_\mathrm{Dep}$, we use 2D convolution layers with batch normalization and SiLU~\citep{elfwing2018sigmoid} activation function. This style of architecture is closely adapted from the basic building blocks in BCResNet~\citep{kim2021broadcasted}. After average pooling the $F_\mathrm{Tem}$ can assume any backbone as per the requirements of the application. As demonstrated previously in Section~\ref{app:F_temp}, the choice of backbone is not central to our design here. 
We find that some applications(like WISDM and GR) benefit from attention-based temporal encoding more than others. For the attention-based version of $F_\mathrm{Tem}$ we used a multi-headed attention based on a transformer encoder~\citep{vaswani2017attention}. Regarding positional encoding, we used a simple sinusoid-based encoding and added it to the sequence representation $\mathbf{r}_\mathrm{Dep}$. However, arriving at the best positional encoding for numerical time-series data is an active area of research~\citep{kazemi2019time2vec, tang2023bio, mohapatra2023effect} given its uniqueness compared to typical natural language inputs and further optimizations can be carried out. For the the convolution-based $F_\mathrm{Tem}$ we simply use a kernel of size $(1 \times 3)$ in a 2D convolution layer to conduct temporal convolutions. 

For the classification head, $g_\mathrm{Cls}$, we apply 2D convolution layers to have the number of output channels equal to the number of classes in an application, followed by softmax operation. Interestingly, if the choice of $F_\mathrm{Tem}$ remains convolutional the entire network can be implemented in a purely convolutional form allowing applicability to real-time problems. The model sizes across the different datasets range from 40k-100k trainable parameters (based on the number of variates, temporal encoding etc.) which is modest and can be further tuned for resource-constrained applications by adjusting the $c$ parameter. 

% We demonstrate the various model sizes and \method{}'s performance in \textcolor{red}{(Figure)}
% \begin{table}[]
% \centering
% \caption{\method{}'s performance with varying model sizes on HAR.}
% \begin{tabular}{cccc}
% \hline
% c & Model Size & WISDM & HHAR \\ \hline
% 1 & 41006      & 0.88  & 0.85 \\
% 2 & 157206     & 0.86  &      \\
% 3 & 348606     & 0.87  &      \\
% 4 & 615206     & 0.87  &      \\ \hline
% \end{tabular}
% \end{table}

\begin{table}[htbp]
\centering
\caption{Complete set of results from three trials on each baseline for SSC cross-person generalization setting.}
\begin{tabular}{|c|cc|cc|cc|cc|}
\hline
Baselines       & \multicolumn{2}{c|}{Scenario 1}                     & \multicolumn{2}{c|}{Scenario 2}                     & \multicolumn{2}{c|}{Scenario 3}                     & \multicolumn{2}{c|}{Scenario   4}                   \\ \cline{2-9} 
                & \multicolumn{1}{c}{Mean} & \multicolumn{1}{c|}{Std} & \multicolumn{1}{c}{Mean} & \multicolumn{1}{c|}{Std} & \multicolumn{1}{c}{Mean} & \multicolumn{1}{c|}{Std} & \multicolumn{1}{c}{Mean} & \multicolumn{1}{c|}{Std} \\ \hline
ERM    &0.50&0.05&0.46&0.04&0.49&0.02&0.45&0.03    \\
GroupDRO       &0.57&0.07&0.56&0.03&0.55&0.05&0.59&0.06                     \\
DANN            &0.64&0.02&0.63&0.02&0.69&0.03&0.63&0.04    \\
RSC             &0.50&0.09&0.48&0.02&0.52&0.07&0.46&0.01                     \\
ANDMask         &0.55&0.10&0.50&0.09&0.54&0.07&0.57&0.08                     \\
InceptionTime   &0.74&0.04&0.78&0.03&0.72&0.05&0.80&0.02                     \\
BCResNet        &0.79&0&0.82&0.01&0.79&0.01&0.81&0                     \\
NSTrans             &0.43&0.02&0.37&0.04&0.42&0.06&0.35&0.03                    \\
Koopa           &0.58&0.02&0.62&0.01&0.53&0.04&0.49&0.06 \\
MAPU            &0.69&0.01&0.68&0.01&0.65&0.03&0.69&0.02                     \\
Diversify       &0.73&0.03&0.76&0.02&0.68&0.05&0.77&0.02                     \\ 
Chronos       &0.53&0.04&0.47&0.04&0.47&0.01&0.57&0.03                     \\ 
\hline
Ours + RevIN*   &0.82&0.01&0.79&0.02&0.78&0.01&0.81&0.01                    \\
Ours            &0.85&0.01&0.80&0.01&0.79&0.01&0.83&0.01                     \\ \hline
\end{tabular}
\label{tab:app_ssc}
\end{table}

\begin{table}[htbp]
\centering
\caption{Complete set of results from three trials on each baseline for GR cross-person generalization setting.}
\begin{tabular}{|c|cc|cc|cc|cc|}
\hline
Baselines       & \multicolumn{2}{c|}{Scenario 1}                     & \multicolumn{2}{c|}{Scenario 2}                     & \multicolumn{2}{c|}{Scenario 3}                     & \multicolumn{2}{c|}{Scenario   4}                   \\ \cline{2-9} 
                & \multicolumn{1}{c}{Mean} & \multicolumn{1}{c|}{Std} & \multicolumn{1}{c}{Mean} & \multicolumn{1}{c|}{Std} & \multicolumn{1}{c}{Mean} & \multicolumn{1}{c|}{Std} & \multicolumn{1}{c}{Mean} & \multicolumn{1}{c|}{Std} \\ \hline
ERM    &0.45&0.02&0.58&0.03&0.57&0.03&0.54&0.04    \\
GroupDRO       &0.53&0.08&0.36&0.11&0.59&0.05&0.45&0.13                     \\
DANN            &0.60&0.01&0.66&0.04&0.65&0.02&0.64&0.03    \\
RSC             &0.50&0.10&0.66&0.05&0.64&0.03&0.56&0.03                     \\
ANDMask         &0.41&0.13&0.54&0.20&0.45&0.15&0.39&0.12                     \\
InceptionTime   &0.68&0.07&0.70&0.09&0.72&0.03&0.69&0.02                     \\
BCResNet        &0.62&0.06&0.67&0.09&0.65&0.05&0.61&0.07                     \\
NSTrans             &0.31&0.01&0.34&0.01&0.34&0.01&0.32&0.02                    \\
Koopa           &0.47&0.03&0.54&0.02&0.60&0.05&0.70&0.06 \\
MAPU            &0.64&0.02&0.69&0.03&0.71&0.01&0.68&0.04                     \\
Diversify       &0.69&0.01&0.80&0.01&0.76&0.02&0.76&0.01                     \\ 
Chronos       &0.49&0.01&0.54&0.03&0.51&0.05&0.48&0.02                     \\ 
\hline
Ours + RevIN*  &0.68&0.03&0.81&0.04&0.77&0.03&0.76&0.02                     \\
Ours            &0.70&0.02&0.82&0.02&0.77&0.04&0.75&0.01                     \\ \hline
\end{tabular}
\label{tab:app_emg}
\end{table}

\begin{table}[htbp]
\small
\centering
\caption{Complete set of results from three trials on each baseline for HHAR one-person-to-another setting.}
\resizebox{1\textwidth}{!}{
\setlength{\tabcolsep}{.6mm}{
\begin{tabular}{|c|cc|cc|cc|cc|cc|cc|cc|cc|cc|}
\hline
Baselines       & \multicolumn{2}{c|}{0}                     & \multicolumn{2}{c|}{1}                     & \multicolumn{2}{c|}{2}                     & \multicolumn{2}{c|}{3}      & \multicolumn{2}{c|}{4}                     & \multicolumn{2}{c|}{5}                     & \multicolumn{2}{c|}{6}                     & \multicolumn{2}{c|}{7}   & \multicolumn{2}{c|}{8}          \\ \cline{2-19} 
                & \multicolumn{1}{c}{Mean} & \multicolumn{1}{c|}{Std} & \multicolumn{1}{c}{Mean} & \multicolumn{1}{c|}{Std} & \multicolumn{1}{c}{Mean} & \multicolumn{1}{c|}{Std} & \multicolumn{1}{c}{Mean} & \multicolumn{1}{c|}{Std} & \multicolumn{1}{c}{Mean} & \multicolumn{1}{c|}{Std} & \multicolumn{1}{c}{Mean} & \multicolumn{1}{c|}{Std} & \multicolumn{1}{c}{Mean} & \multicolumn{1}{c|}{Std} & \multicolumn{1}{c}{Mean} & \multicolumn{1}{c|}{Std}  & \multicolumn{1}{c}{Mean} & \multicolumn{1}{c|}{Std}\\ \hline
ERM    &0.27&0.01&0.40&0.05&0.41&0.05&0.44&0.05&0.42&0.08&0.44&0.01&0.45&0.04&0.44&0.04&0.48&0.02    \\
GroupDRO       &0.33&0.02&0.53&0.02&0.38&0.05&0.48&0.04&0.47&0.04&0.51&0.08&0.47&0.03&0.48&0.02&0.49&0.05                    \\
DANN            &0.32&0.03&0.44&0.05&0.42&0.03&0.45&0.06&0.42&0.03&0.48&0.04&0.49&0.02&0.45&0.05&0.51&0.01    \\
RSC             &0.27&0.03&0.45&0.06&0.38&0.05&0.45&0.09&0.40&0.08&0.47&0.02&0.50&0.06&0.44&0.08&0.53&0.01                     \\
ANDMask         &0.34&0.06&0.50&0.03&0.37&0.04&0.43&0.05&0.46&0.04&0.51&0.07&0.46&0.03&0.47&0.02&0.52&0.03                     \\
InceptionTime   &0.52&0.05&0.62&0.02&0.44&0.03&0.69&0.04&0.60&0.09&0.57&0.05&0.66&0.03&0.64&0.01&0.61&0.01                     \\
BCResNet        &0.28&0.03&0.48&0.08&0.32&0.04&0.47&0.03&0.42&0.06&0.52&0.05&0.44&0.02&0.45&0.02&0.49&0.06                     \\
NSTrans         &0.20&0.01&0.22&0.02&0.17&0.02&0.20&0.01&0.21&0.01&0.22&0.01&0.26&0.07&0.17&0.05&0.20&0.01                    \\
Koopa 
&0.32 &0.02 &0.42 &0.04 &0.37 &0.01 &0.40 &0.01 &0.42 &0.02 &0.45 &0.05 &0.35 &0.02 &0.43 &0.03 &0.48 &0.02 \\

MAPU            &0.39&0.05&0.57&0.05&0.35&0.06&0.52&0.03&0.49&0.04&0.54&0.02&0.49&0.01&0.50&0.06&0.52&0.04                     \\
Diversify       &0.42&0.04&0.62&0.04&0.32&0.09&0.62&0.01&0.56&0.03&0.61&0.01&0.53&0.04&0.52&0.10&0.61&0.05                     \\ 
Chronos         &0.32&0.03&0.23&0.05&0.26&0.04&0.25&0.03&0.27&0.09&0.23&0.08&0.24&0.06&0.21&0.08&0.24&0.05                     \\ 
\hline
Ours + RevIN*   &0.48&0.02&0.66&0.08&0.57&0.05&0.65&0.03&0.61&0.04&0.64&0.05&0.65&0.06&0.64&0.01&0.63&0.03                     \\
Ours            &0.53&0.04&0.70&0.03&0.63&0.01&0.66&0.03&0.64&0.06&0.67&0.01&0.65&0.03&0.67&0.04&0.62&0.02                     \\ \hline
\end{tabular}}}
% \vspace{-15pt}
\label{tab:app_hhar_one_to_x}
\end{table}

\
\subsection{Ablation Details of \method{}}\label{app:NST_ablation}
For row 1 in Table~\ref{tab:ablation}, the modification to \method{} is straightforward by simply omitted the Hilbert transformation during data preprocessing. When the separate encoders are not used (rows 6 and 7 in Table~\ref{tab:ablation}), we only use $F_\mathrm{Mag}$ and connect the output of the sub-feature normalization block directly to the $F_\mathrm{Dep}$. When the residual is removed entirely (rows 5 and 6 in Table~\ref{tab:ablation}), we cannot broadcast the 1D input to 2D anymore so we take the mean across all the temporal indices of $F_\mathrm{Tem}(\mathbf{r}_\mathrm{Dep})$ and flatten it to input to fully connected layers. Based on the dataset we choose a few fully connected layers truncating to the number of classes finally.

\subsection{Phase-driven NSTrans}\label{app:NSTrans}
Non-stationary transformer, NSTrans~\citep{non_transformer_1}, applies a destationarizing attention around the transformer block. Since it is typically used for forecasting tasks, it comprises of encoder and a decoder module. For adapting this model to classification we update the design to conduct normalization and denormalization around the encoder block. We use this modified version of NSTrans as the $F_\mathrm{Tem}$ module in \method{} and observe significant improvement in performance as shown in Figure~\ref{fig:ablation_2}.

\noindent \textit{Note:} The poor performance of the Nonstationary transformer can be attributed to two main reasons:

(1) Originally, the Nonstationary transformer was designed for forecasting time-series tasks and employs an encoder-decoder style architecture. To successfully apply the core module of the Nonstationary transformer~\citep{non_transformer_1}, stationarization-destationarization, the input-output space needs to remain consistent. This consistency is naturally ensured in an encoder-decoder design. However, in our classification applications, we only utilize the encoder module. Although we maintain the input-output dimensions, the semantics of the latent space and input space are not the same. Hence, destationarization is not very successful.

(2) Nonstationary transformer inputs consist of raw time-series data with positional encoding. Given the fine-grained nature of current tasks, such an approach can be more data-hungry as they try to establish a relation (attention) among every time step. Therefore, it may not perform well on short-range classification tasks that focus on domain generalization. This indicates a limitation in its direct usage for optimizing a categorical objective function using only the encoder part with a classification head.

% \section{Case Study Implementation Specifics} \label{sec:case_study}

\subsection{Computational Analyses}\label{app:comp_analyses}
To assess the resource utilization of \method{} against other baselines, we offer two metrics - 1) Number of Multiply and Accumulate operations per sample (MACs) for approximate computational complexity at run-time and 2) Number of trainable parameters to determine the memory footprint. We compute these for the HHAR dataset in Table~\ref{tab:mac} (these metrics are dependent on input dimensions, hence different choices of dataset, sequence length, and modalities can yield different numbers).

\begin{table}[htbp]
    \centering
    \caption{Model comparison based on MACs and number of trainable parameters.}
    \begin{tabular}{lcc}
        \toprule
        Model & MACs ($\times 10^6$) & Trainable Parameters ($\times 10^3$) \\
        \midrule
        ERM & 19.5 & 98.1 \\
        GroupDRO & 19.5 & 98.1 \\
        DANN & 21.7 & 102.9 \\
        RSC & 19.5 & 98.1 \\
        ANDMask & 19.5 & 98.1 \\
        BCResNet & 55.3 & 154.7 \\
        NSTrans & 35.3 & 75.6 \\
        Koopa & 32.7 & 118.7 \\
        MAPU & 46.9 & 128.3 \\
        Diversify & 35.7 & 922.9 \\
        Chronos & 345.5 & 1049.8\\
        Ours & 48.6 & 81.4 \\
        \bottomrule
    \end{tabular}
    \label{tab:mac}
\end{table}

Our computation cost is comparable to the other methods, achieving much better performance. We also determine the asymptotic time complexity of the \method{} modules in Table~\ref{tab:module_complexity}. For multi-layer neural network modules, the representative time complexity for one layer is provided (rows 3-7).

\begin{table}[htbp]
    \centering
    \small
    \caption{Complexity per module and input notation for each module.}
    \resizebox{\textwidth}{!}{
    \begin{tabular}{lp{0.6\textwidth}c}
        \toprule
         & \centering{Module} & Complexity \\
        \midrule
        1 & Hilbert augmentation (using Fast-Fourier transform) & $\mathcal{O}(V \cdot N \log N)$ \\
        2 & Short-Term Fourier Transform & $\mathcal{O}(V \cdot N \cdot W \log W)$ \\
        3 & Magnitude Encoder ($F_\mathrm{Mag}$), Phase Encoder ($F_\mathrm{Pha}$), Phase Projection Head ($g_\mathrm{Res}$) - 2D Convolution Layers & $\mathcal{O}(k^2 \cdot N \cdot d \cdot c_{in} \cdot c_{out} )$ \\
        4 & Depthwise Feature Encoder ($F_\mathrm{Dep}$) - 2D Convolution Layers with average pooling along feature axis & $\mathcal{O}(k^2 \cdot N \cdot d \cdot c_{in} \cdot c_{out} ) + \mathcal{O}(d)$ \\
        5 & Temporal Encoder ($F_\mathrm{Tem}$) - (worst case backbone) Transformer Encoder & $\mathcal{O}(N \cdot d)$ \\
        6 & Classification Encoder ($g_\mathrm{Cls}$) - fully connected layers & $\mathcal{O}(d \cdot h)$ \\
        \bottomrule
    \end{tabular}}
    \label{tab:module_complexity}
\end{table}

\subsection{Additional Analyses}\label{subsec:app_analyses}

\subsubsection{Traditional Augmentation}
For time series, brute augmentations like scaling, reverting, cropping, and jittering may not be always suitable as they may alter the morphological properties that are important for the task. Even more advanced techniques like frequency-time warping and additive noise, need deliberate characterization of the signal's frequency response to meaningfully provide an augmented view while retaining the task-relevant semantics. This is one of the key motivating factors for us to explore a general-purpose augmentation strategy that diversifies the non-stationarity in a signal without altering its task-specific semantics (magnitude and frequency responses).

To demonstrate the use of traditional augmentations with \method{} for human-activity recognition, we incorporate the following augmentations proposed by past works~\citep{qin2023generalizable, um2017data} on the HHAR dataset. 

\begin{itemize}
    \item Rotation - incorporating arbitrary rotation matrices to simulate different sensor locations.
    \item Permutation - random temporal perturbation for fixed window within each sample~\citep{um2017data}.
    \item Circular Time-shift - shifting the signal by a random time interval, constrained by a predefined maximum time-shift parameter (20\% of the sample length in this case) for each sample. The shifted time points from the trailing edge are wrapped around and padded to the leading edge of the signal
\end{itemize}

We incorporate these augmentations in place of the Hilbert augmentation and apply the \method{}. We also run an experiment with identical settings with no augmentations and illustrate in Figure~\ref{fig:aug_pha}. These results are indicative that arbitrary augmentations in the time domain do not necessarily diversify the non-stationarity of a signal. Hence, PhASER principles like residual connections to re-introduce nonstationary dictionary as phase-projection and broadcasting (using $g_\mathrm{Res}$) do not bode well here, and even the performance of a no-augmentation scenario is sometimes better than the traditional temporal augmentations for domain-generalization tasks in this case. However, in the future, we may encounter applications where established augmentation strategies, in combination with Hilbert augmentation, might be the best choice. In this work, we aim to propose a more generic framework that can benefit most time-series classification tasks to achieve better generalizability.

% We incorporate these augmentations in place of the Hilbert augmentation, apply the \method{}, and present the results in Table~\ref{tab:aug_analys} (row 2). We also run an experiment with identical settings with no augmentations shown in row 3. These results are indicative that arbitrary augmentations in the time domain do not necessarily diversify the non-stationarity of a signal. Hence, PhASER principles like residual connections to re-introduce nonstationary dictionary as phase-projection and broadcasting (using $g_\mathrm{Res}$) do not bode well here and even the performance of a no-augmentation scenario (row 3) is better than the traditional temporal augmentations for domain-generalization tasks in this case. However, in the future, we may encounter applications where established augmentation strategies, in combination with Hilbert augmentation, might be the best choice. In this work, we aim to propose a more generic framework that can benefit most time-series classification tasks to achieve better generalizability.

% \begin{table}[htbp]
% \centering
% \caption{Comparison of traditional augmentation against phase augmentation using Hilbert transform on HHAR Dataset}
% \label{tab:hhartable}
% \begin{tabular}{ccccccc}
% \hline
% & Target & 1 & 2 & 3 & 4 & Average Accuracy \\
% \hline
% 1 & Ours (with Phase Augmentation) & 0.83 & 0.83 & 0.94 & 0.88 & 0.87 \\

% 2 & Ours + Traditional Augmentation & 0.76 & 0.76 & 0.83 & 0.75 & 0.78 \\

% 3 & Ours + No Augmentation & 0.83 & 0.72 & 0.89 & 0.84 & 0.82 \\
% \hline
% \end{tabular}
% \label{tab:aug_analys}
% \end{table}

\subsubsection{Random Phase Augmentation using Hilbert Transform} \label{subsub:randHT}

We aimed to explore a random phase augmentation while ensuring minimal distortion to the signal's magnitude response to preserve important task-relevant properties. To achieve this, we leverage an adaptation of the Hilbert Transform. We illustrate our approach using a simple example: let the input signal be $\mathbf{x}(t) = \sin(\omega t)$, and its Hilbert Transform be $\mathrm{HT}(\mathbf{x}(t)) = \widehat{\mathbf{x}}(t) = -\cos(\omega t)$. For an arbitrary phase shift $\phi$, the following trigonometric identity holds:

\begin{equation} \sin(\omega t + \phi) = \sin(\omega t) \cos(\phi) + \cos(\omega t) \sin(\phi). \end{equation}

This gives us the desired randomly phase-shifted version of $\mathbf{x}(t)$, expressed as $\mathbf{y}(t) = a\mathbf{x}(t) - b\widehat{\mathbf{x}}(t)$, where $a = \cos(\phi)$ and $b = \sin(\phi)$. The following constraint is imposed on the scalars $a$ and $b$:

\begin{equation} a^2 + b^2 = 1, \end{equation}

which defines a valid phase shift $\phi$ as:

\begin{equation} \phi = \arctan\left(\frac{b}{a}\right). \end{equation}

We solve for $a$ and $b$, and apply them as shown in Figure~\ref{fig:randomHT} to obtain an approximately identical random phase shift across all frequency components of a nonstationary signal. The desired $\phi$ is randomly sampled from the range $[-\pi/2, \pi/2]$.

\begin{figure}[!htbp] 
\centering 
\includegraphics[width=0.7\linewidth]{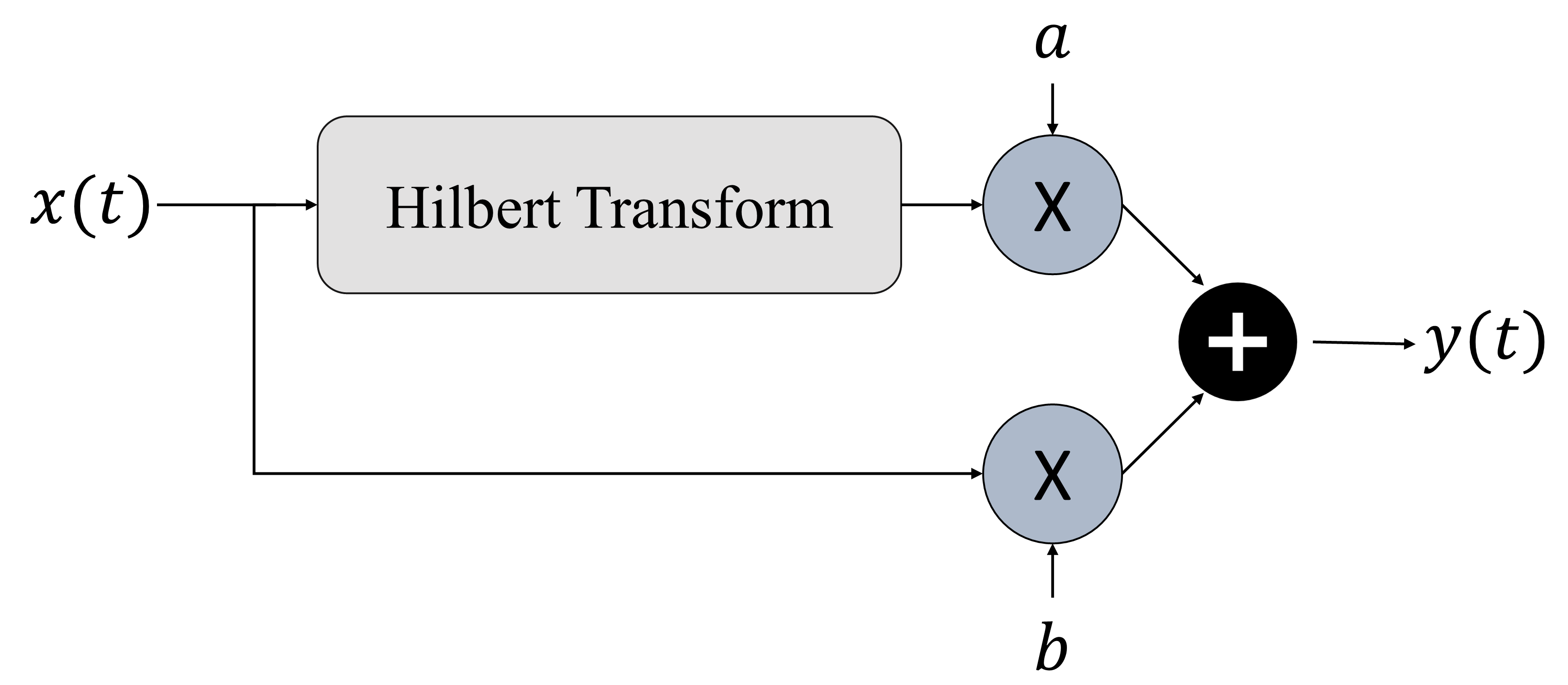} \caption{Schema illustrating the process for obtaining random phase augmentation by leveraging the Hilbert Transform of the original input $\mathbf{x}(t)$.} 
\label{fig:randomHT} 
\end{figure}

As shown in Figure~\ref{fig:aug_pha}, we observe no significant benefit from this randomization on the generalization performance of the current classification tasks. However, we are interested in exploring this direction in future by imposing additional constraints inspired by underlying processes for other time-series tasks.

% This is achieved by applying an abstract phase shift to the non-stationary signal, $\mathbf{x}(t)$, to $\mathbf{x}(t + \phi)$, under the assumption that $\mathbf{x}(t)$ is a pure sinusoid. Through trigonometric manipulations, we leverage the imaginary part of the analytic signal from the Hilbert transform to solve for $\phi$. More details are provided in the Appendix. 

\subsection{Visualization}\label{subsec:app_vis}
We present some visualizations using the t-distributed stochastic neighbor embedding (t-sne) analyses on our \method{}, Diversify, and BCResNet for the HHAR dataset for the left-out domains in scenario 1 in Figure~\ref{fig:tsne}. We illustrate the t-sne plots for in-domain and out-of-domain data and the different colors indicate the six activity classes of this dataset. In all the cases, we only make necessary modifications to extract the embeddings from the last layer of the network before categorical score assignment and tune the perplexity parameters during the t-sne plotting for optimal 2-dimensional projection. Figure~\ref{fig:app_tsne}. (a,d) shows that the clustering for each class is distinct and clearly separable for both in-domain and out-of-domain data using PhASER. The accuracy disparity for unseen domains is also very low, 0.97 for in-domain PhASER accuracy and 0.94 for out-of-domain, which justifies the overall strong generalization ability of PhASER without access to any target domain samples. We would also like to point out that t-sne plots are susceptible to hyperparameters, hence, even though the accuracy of Diversify is better than BCResnet for out-of-domain data, visually Figure~\ref{fig:app_tsne}. (f) may convey better separation between classes than Figure~\ref{fig:app_tsne}. (e).

\begin{figure}[htbp]
\centering
\includegraphics[width=\linewidth]{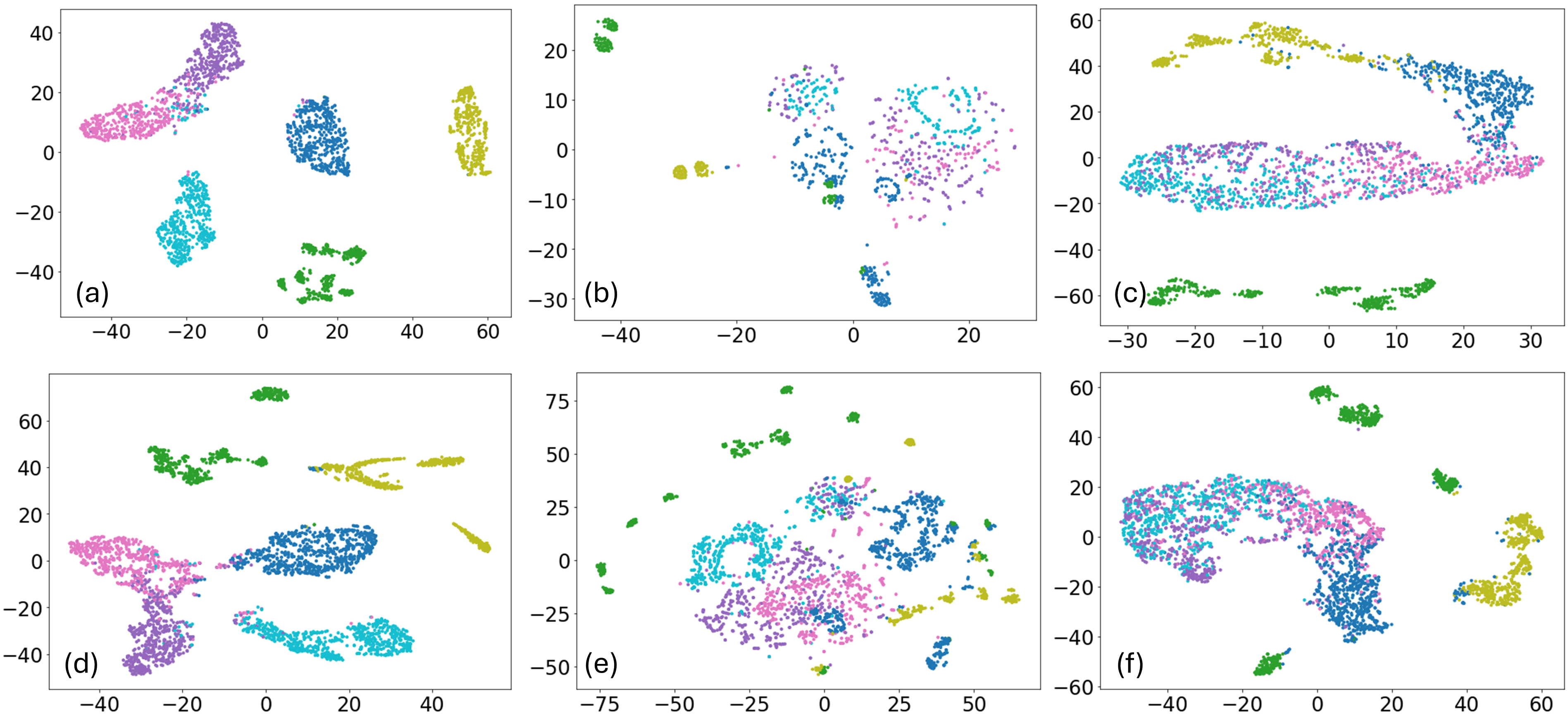}
% \vspace{-2 pt}
\caption{t-sne plots for visualizations using embeddings from HHAR scenario 1 for in-domain samples in (a) \method{} with an in-domain-accuracy of 0.97, (b) Diversify with in-domain accuracy of 0.82 and (c) BCResNet with in-domain accuracy of 0.78; and out-of-domain samples in (c) \method{} with accuracy of 0.94, (d) Diversify with accuracy of 0.77 and (e) BCResNet with accuracy of 0.74.}
\label{fig:app_tsne}
\end{figure}

\section{Supplementary of Main Results} \label{app:add_results}
We conduct all experiments with three random seeds (2711, 2712, 2713), and present the error range
in this section. Tables~\ref{tab:app_wisdm},~\ref{tab:app_hhar} and ~\ref{tab:app_ucihar} represent the mean and standard deviation corresponding to the main paper's Table~\ref{tab:har_result} for the WISDM, HHAR and UCIHAR datasets respectively. Tables~\ref{tab:app_ssc} and ~\ref{tab:app_emg} are the complete representations of all the runs corresponding to Table~\ref{tab:eeg_result} in the main paper for sleep stage classification and gesture recognition respectively. Table~\ref{tab:app_hhar_one_to_x} corresponds to the Table~\ref{tab:one_to_x} in the main paper for the complete performance statistics for one person to another generalization using HHAR dataset.

% Limitation -- Unwanted phase shift in audio signals can cause issues --> Latency between devices and so on -- Phase Vocoder

\tmlr{\section{Experimental Setup for Figure~\ref{fig:motivation_ns_phase}}\label{app:ns_exp}

We analyze three synthetic signals sampled at 100 Hz with 500 points: a stationary sinusoid \(x_1(t) = \sin(2 \pi 3 t)\), and two nonstationary signals \(x_2(t) = 5000 \sin(2 \pi 3 t) - 10 t^5\) and \(x_3(t) = 5 t \sin(2 \pi 3 t) + 10 t\). Each is normalized by its maximum amplitude. The Hilbert transform is applied to extract imaginary components representing instantaneous phase information. Using FFT, we compute normalized magnitude and unwrapped phase spectra for original and Hilbert-imaginary signals. }

\section{Broader Impacts} \label{app:broad}
% This paper presents work whose goal is to advance the field of Machine Learning. There are many potential societal consequences of our work, none which we feel must be specifically highlighted here.
\method{}, with its advanced approach to time-series domain-generalizable learning, offers significant societal benefits to various fields and domains, such as healthcare, environment monitoring, and manufacturing domains, by enabling more precise and dependable data analysis. While \method{} itself does not directly cause negative social impacts, its application within these critical areas necessitates a thoughtful examination of ethical concerns. In healthcare, the application of \method{} could usher in a new era of patient monitoring and treatment, leading to improved experiences and outcomes for individuals across diverse demographics. Its robust generalization capabilities, even with limited access to source domains (see Table~\ref{tab:one_to_x}), offer the potential to bridge gaps and foster inclusivity, particularly in minority communities, while enabling insights from rare occurrences. Moreover, for applications in environmental monitoring—ranging from continuous sensing of ambient living conditions to remote and sporadic sensing of inaccessible geological sites—\method{}'s principles hold promise for sample-efficient, generalizable analysis. Similarly, in manufacturing applications, \method{} can be deployed for both qualitative and quantitative analyses of physical components, as well as for enhancing workers' safety through continuous sensing instrumentation. However, the implementation of \method{} in such vital areas brings to the forefront ethical considerations like data privacy, bias prevention, and the careful management of automation reliance. Addressing these issues is important to leverage \method{}'s benefits across these domains while ensuring ethical integrity and maintaining public trust in these areas.

\pagebreak

\end{document}